\documentclass{jmlr}

\usepackage{microtype}
\usepackage{graphicx}
\usepackage{booktabs} % for professional tables

\usepackage{amsmath,amssymb,amsthm,mathrsfs,amsfonts,dsfont,tikz}

\usepackage{algorithm,algorithmic}
\usepackage{algorithm2e}
\usepackage{hyperref}
\hypersetup{final}
\newtheorem{thm}{Theorem}

\newtheorem{lem}{Lemma}

\newtheorem{defi}[thm]{Definition}
\newtheorem{ass}{Assumption}
\newtheorem{propo}{Proposition}

\def\R{\mathbb{R}}
\def\calS{\mathcal{S}}
\def\calB{\mathcal{B}}

\def\xbar{{\bar{x}}}
\def\ybar{{\bar{y}}}

\def\yhat{{\hat{y}}}

\def\xstar{x^{*}}

\def\Pstar{P^{*}}

\def\xtilde{{\tilde{x}}}
\def\ytilde{{\tilde{y}}}

\def\yhat{\hat{y}}
\def\E{\mathds{E}}

\def\LHS{{\text{LHS}}}
\def\RHS{{\text{RHS}}}

\def\dist{dist}
\def\calL{\mathcal{L}}

\def\tildedelta{\tilde{\delta}}

\def\hatepsilon{{\hat{\epsilon}}}

\def\Holder{H\"older~}

\def\w{{\bf{w}}}
\def\x{{\bf{x}}}

\def \S {\mathbf{S}}

\def \R {\mathbb{R}}

\def \w {\mathbf{w}}
\def \v {\mathbf{v}}

\def \x {\mathbf{x}}

\def \E {\mathrm{E}}

\def \x {\mathbf{x}}

\def \E {\mathrm{E}}
\def \x {\mathbf{x}}

\def \w {\mathbf{w}}

\def \R {\mathbb{R}}
\def \S {\mathcal{S}}

\def \v {\mathbf{v}}

\newcommand\yancomment[1]{}

\begin{document}
\title[Faster Stochastic Primal-Dual Algorithm]{Stochastic Primal-Dual Algorithms with Faster Convergence than $O(1/\sqrt{T})$ for Problems without Bilinear Structure}

\author{\Name{Yan Yan}$^1$\Email{yan-yan-2@uiowa.edu}\\
\Name{Yi Xu}$^{1}$\Email{yi-xu@uiowa.edu}\\
\Name{Qihang Lin}$^2$\Email{qihang-lin@uiowa.edu}\\
\Name{Lijun Zhang}$^3$\Email{zhanglj@lamda.nju.edu.cn}\\
\Name{Tianbao Yang}$^1$\Email{tianbao-yang@uiowa.edu}\\
\addr $^1$Department of Computer Science, The University of Iowa, Iowa City, IA 52242 \\
\addr $^2$Department of Management Sciences, University of Iowa, Iowa City, IA 52242  \\
\addr $^3$National Key Laboratory for Novel Software Technology Nanjing University, Nanjing 210023, China \\
}

\maketitle

% this must go after the closing bracket ] following \twocolumn[ ...

% This command actually creates the footnote in the first column
% listing the affiliations and the copyright notice.
% The command takes one argument, which is text to display at the start of the footnote.
% The \icmlEqualContribution command is standard text for equal contribution.
% Remove it (just {}) if you do not need this facility.

%\printAffiliationsAndNotice{}  % leave blank if no need to mention equal contribution
%\printAffiliationsAndNotice{\icmlEqualContribution} % otherwise use the standard text.

\begin{abstract}
%Stochastic primal-dual algorithms are attractive  for solving min-max saddle-point problems. %solving minimization problems that are non-decomposable into individual components, which maintain and update simultaneously/alternatively  both the primal variable and the dual variable. %However, stochastic primal-dual algorithms with faster convergence than $O(1/\sqrt{T})$ with $T$ being the number of iterations remains under-explored even with strong convexity (resp. concavity) on the primal (resp. the dual variable). 
Previous studies on  stochastic primal-dual algorithms for solving  min-max  problems with faster convergence heavily rely  on the bilinear structure of the problem, which restricts their applicability to a narrowed range of problems. The main contribution of this paper is the design and analysis of new stochastic primal-dual algorithms that use a mixture of stochastic gradient updates and a logarithmic number of deterministic dual updates for solving a family of convex-concave problems with no bilinear structure assumed. Faster convergence rates than $O(1/\sqrt{T})$ with $T$ being the number of stochastic gradient updates are established under some mild conditions of involved functions on the primal and the dual variable. For example, for a family of problems that enjoy a weak strong convexity in terms of the primal variable and has a strongly concave function of the dual variable, the convergence rate of the proposed algorithm is $O(1/T)$. We also investigate the effectiveness of the proposed algorithms for learning robust models and empirical AUC maximization. 
%Recent stohcastic primal-dual algorithms achieves a linear convergence rate by exploiting the strongly convex condition.However, this condition may not be common in practice. Comapred to the strong convexity condition, local error bound is instead a milder condition. This paper presents a primal-dual algorithm which employs the local error bound condition to accelerate the convergence of the primal-dual algorithm.
\end{abstract}

\setlength{\abovedisplayskip}{3pt}
\setlength{\belowdisplayskip}{3pt}
\section{Introduction}
This paper is motivated by solving the following convex-concave problem: 
\begin{align}\label{eqn:main}
\min_{x\in X}\max_{y\in \text{dom}(\phi^*)}y^{\top}\ell(x) - \phi^*(y) + g(x)
\end{align}
where $X\subseteq\R^d$ is a closed convex set, $\ell(x) = (\ell_1(x), \ldots, \ell_n(x))^{\top}: X\rightarrow\R^n$ is a lower-semicontinuous mapping whose component function $\ell_i(\x)$ is lower-semicontinuous  and convex, $\phi^*(y): \text{dom}(\phi^*)\rightarrow\R$ is a convex function whose convex conjugate is denoted by $\phi$, and $g(\x):X\rightarrow\R$ is a lower-semicontinuous convex function.  To ensure the convexity of the problem, it is assumed that $\text{dom}(\phi^*)\subseteq\R_+^n$ if $\ell(x)$ is not an affine function. %If $\ell(x) = Ax + b\in\R^n$ is an affine function, then the requirement  $\text{dom}(\phi^*)\subseteq\R_+^n$ is not necessary and  the above problem is said to have a bilinear structure. 
%However, such a bilinear structure is not necessarily imposed in this paper. 
By using the convex conjugate $\phi^*$, the problem~(\ref{eqn:main}) is equivalent  to the following convex minimization problem: 
\begin{align}\label{eqn:primal}
\min_{x\in X}P(x): = \phi(\ell(x)) + g(x).
\end{align}
%which is a convex problem due to the assumption  $\text{dom}(\phi^*)\subseteq\R_+^n$~\footnote{One can show that $\phi$ is convex and non-decreasing. }. 
A particular family of min-max problem~(\ref{eqn:main}) and its minimization form~(\ref{eqn:primal}) that has been considered extensively in the literature~\citep{DBLP:conf/icml/ZhangL15,DBLP:journals/corr/YuLY15,DBLP:conf/nips/TanZML18,citeulike:11703902,DBLP:conf/nips/LinLX14} is that $\ell(x) = Ax + b$ is an affine function and $\phi(s)= \sum_{i=1}^n \phi_i(s_i)$ for $s\in\R^n$ is decomposable. In this case, the problem~(\ref{eqn:primal}) is known as (regularized) {empirical risk minimization} problem in machine learning:
\begin{align}\label{eqn:p}
\min_{x \in X} \frac{1}{n} \sum_{i=1}^{n} \phi_{i}( a_i^{\top}x + b_i ) + g(x) ,
\end{align}
where $a_i$ is the $i$-th row of $A$ and $b_i$ is the i-th element of $b$. 

However, stochastic optimization algorithms with fast convergence rates are still under-explored for a more challenging family of problems of~(\ref{eqn:main}) and~(\ref{eqn:primal}) where $\ell(x)$ is not necessarily an affine or smooth function and $\phi$ is not necessarily decomposable. It is our goal to design new stochastic primal-dual algorithms for solving these problems with a fast convergence rate. 
%In this paper, we {\bf consider solving a more challenging family of problems} of~(\ref{eqn:main}) and~(\ref{eqn:primal}) by stochastic algorithms, where $\ell(x)$ is not necessarily an affine function and $\phi$ is not necessarily decomposable.   
A key motivating example of the considered problem is to solve a distributionally robust optimization problem: 
 \begin{align}\label{eqn:dro}
 \min_{x\in X}\max_{y\in \Delta_n}\sum_{i=1}^ny_i \ell_i(x) - V(y, y_0) + g(x),
 \end{align} 
 where $\Delta_n = \{y\in\R^n; y_i\geq 0, \sum_iy_i = 1\}$ is a simplex, and $V(y, y_0)$ denotes a divergence measure (e.g., $\phi$-divergence)  between two sets of probabilities $y$ and $y_0$.  In machine learning with $\ell_i(x)$ denoting the loss of a model $x$ on the $i$-th example, the above problem corresponds to {\bf robust risk minimization paradigm}, which can achieve variance-based regularization for learning a predictive model from $n$ examples~\citep{DBLP:conf/nips/NamkoongD17}. Other examples of the considered challenging problems can be found in robust learning from multiple perturbed distributions~\citep{NIPS2017_7056} in which $\ell_i(x)$ corresponds to the loss from the $i$-th perturbed distribution, and minimizing non-decomposable loss functions~\citep{DBLP:conf/nips/FanLYH17,DBLP:conf/nips/DekelS06}.

With stochastic (sub)-gradients computed for $x$ and $y$, one can employ the conventional  primal-dual stochastic gradient method or its variant~\citep{Nemirovski:2009:RSA:1654243.1654247,juditsky2011} for solving the problem~(\ref{eqn:main}). Under appropriate basic assumptions, one can derive the standard $O(1/\sqrt{T})$ convergence rate with $T$ being the number of stochastic updates. However, the convergence rate $O(1/\sqrt{T})$ is known as a slow convergence rate. It is always desirable to design optimization algorithms with a faster convergence. Nonetheless, to the best of our knowledge stochastic primal-dual algorithms with a fast convergence rate of $O(1/T)$ in terms of minimizing $P(x)$ remain  unknown  in general, even under  the strong convexity of $\phi^*$ and $g$.  In contrast, if $\phi$ is decomposable and $P$ is strongly convex, the standard stochastic gradient method for solving~(\ref{eqn:primal}) with an appropriate scheme of step size has a convergence rate of $O(1/T)$~\citep{DBLP:journals/ml/HazanAK07,DBLP:journals/jmlr/HazanK11a}. A direct extension of algorithms and analysis  for stochastic strongly convex minimization to the stochastic concave-concave optimization does not give a satisfactory $O(1/T)$ convergence rate~\footnote{One may obtain a dimensionality dependent convergence rate of $O(n/T)$ by following conventional analysis, but it is not the standard dimensionality independent rate that we aim to achieve.  }. It is still {\bf an open problem} that whether there exists a stochastic primal-dual algorithm by solving the convex-concave problem~(\ref{eqn:main}) that enjoys a fast rate of $O(1/T)$ in terms of minimizing $P(x)$. 

{\bf The major contribution} of this paper is to fill this gap by developing stochastic primal-dual algorithms for solving~(\ref{eqn:main}) such that they enjoy a faster convergence than $O(1/\sqrt{T})$ in terms of the primal objective gap. In particular,  under the assumptions that $\nabla\phi$ is Lipschitz continuous, $\ell_i(x)$ are Lipschitz continuous and the minimization problem~(\ref{eqn:primal}) satisfies the strong convexity condition, the proposed algorithms enjoy an iteration complexity  of $O(1/\epsilon)$ for finding a solution $x$ such that $\E[P(x) - \min_{x\in X}P(x)]\leq \epsilon$, which corresponds to a faster convergence rate of $O(1/T)$. The key difference of the proposed algorithms from the traditional stochastic primal-dual algorithm is that it is required to compute a logarithmic number of deterministic updates for $y$ in the following form: 
\begin{align}
\mathcal A(x) = \arg\max_{y\in\text{dom}(\phi^*)}y^\top\ell(x) - \phi^*(y),
\end{align}
which can be usually solved in $O(n)$ time complexity. 
It would be worth noting that $\mathcal A(x) = \nabla \phi(\ell(x))$ (See Appendix \ref{app:sec:regarding_A}).
When $n$ is a moderate number, the proposed algorithms could converge faster than the traditional primal-dual stochastic gradient method. It is also important to note that we do not assume the proximal mapping of $\phi^*$ and $g$ can be easily computed. Instead, our algorithms only require (stochastic) sub-gradients of $\phi^*$ and $g$, which make them  applicable and efficient for solving  more challenging problems where $g$ is an empirical sum of individual  functions.  

In addition, the proposed algorithms and theories can be easily extended to the case that $\nabla\phi$ is \Holder continuous and the minimization problem~(\ref{eqn:primal})  satisfies a more general local error bound condition as defined later, with intermediate faster rates established.

\section{Related Work}\label{sec:rw}
%In this section, we review some closely related work for solving the convex-concave problem~(\ref{eqn:main}) or its more general form. 

Stochastic primal-dual  gradient method and its variant were first analyzed by~\citep{Nemirovski:2009:RSA:1654243.1654247} for solving a more general problem $\min_{x\in X}\max_{y\in Y}\E_{\xi}[f(x, y; \xi)]$. Under the standard bounded stochastic (sub)-gradient assumption, a convergence rate of $O(1/\sqrt{T})$ was established for a primal-dual gap, which implies a convergence rate of $O(1/\sqrt{T})$ for minimizing the primal objective $P(x) = \max_{y\in Y}\E_{\xi}[f(x, y; \xi)]$. Later, there are couple of studies that aim to strengthen this convergence rate by leveraging the smoothness of $f(x, y; \xi)$ or the involved function when there is a special structure of the objective function~\citep{juditsky2011,doi:10.1137/130919362,Chen2017}. However, the worst-case convergence rate of these later algorithms is still dominated by $O(1/\sqrt{T})$. Without smoothness assumption on $\ell(\x)$ or a bilinear structure, these later algorithms are not directly applicable to solving~(\ref{eqn:main}).  
In addition, Frank Wolfe algorithms are analyzed for saddle point problems in~\citep{gidel2016frank}, which could also achieve a convergence rate of $O(1/\sqrt{T})$ in terms of primal-dual gap under the smoothness condition.

Recently, there emerge several algorithms with faster convergence for solving~(\ref{eqn:main}) by leveraging the bilinear structure and strong convexity of  $\phi^*$ and $g$. 
For example, \citet{DBLP:conf/icml/ZhangL15} proposed a stochastic primal-dual coordinate (SPDC) method for solving~(\ref{eqn:p}) under the condition that $\ell(x) = Ax$ is of bilinear structure and $\phi^*$ is strongly convex. When $g$ is also a strongly convex function, SPDC enjoys a linear convergence for the primal-dual gap. Other variants of SPDC have been considered in~\citep{DBLP:journals/corr/YuLY15,DBLP:conf/nips/TanZML18} for solving~(\ref{eqn:main}) with bilinear structure. \citet{DBLP:conf/nips/PalaniappanB16} proposed stochastic variance reduction methods for solving a family of  saddle-point problems. When applied to~(\ref{eqn:main}), they require $\ell(\x)$ is either an affine function or a smooth mapping. If additionally $g$ and $\phi^*$ are strongly convex, their algorithms also enjoy a linear convergence for finding a solution that is $\epsilon$-close to the optimal solution in squared Euclidean distance. \citet{DBLP:journals/corr/abs-1802-01504} established a similar linear convergence of a primal-dual SVRG algorithm for solving~(\ref{eqn:main}) when $\ell = Ax$ is an affine function with a full column rank for $A$, $g$ is smooth, and $\phi^*$ is smooth and strongly convex, which are stronger assumptions than ours.  All of these algorithms except \citep{DBLP:journals/corr/abs-1802-01504} also need to compute  the proximal mapping of $\phi_i^*$ and $g$ at each iteration. In contrast, the present work is complementary to these studies aiming to solve a more challenging family of problems. In particular, the proposed algorithms do not require the bilinear structure or the smoothness of $\ell$, and the smoothness and strong convexity of $\phi^*$ and $g$ are also not necessary. In addition, we do not assume that $g$ and $\phi^*$ have an efficient proximal mapping. %Instead, only stochastic (sub)-gradients of them are required,  which can capture a broader family of problems. 

Several recent studies have been devoted to stochastic AUC optimization based on a min-max formulation that has a bilinear structure~\citep{fastAUC18,DBLP:conf/icml/NatoleYL18}, aiming to derive a faster convergence rate of $O(1/T)$. The differences from the present work is that (i) \citep{fastAUC18}'s analysis is restricted to the online setting for AUC optimization; (ii) \citep{DBLP:conf/icml/NatoleYL18} only proves a convergence rate of $O(1/T)$ in term of squared distance of found primal solution to the optimal solution under the strong convexity of the regularizer on the primal variable, which is weaker than our results on the convergence of the primal objective gap.  To the best of our knowledge, the present work is the first one that establishes a  convergence rate of $O(1/T)$ in terms of minimizing $P(x)$ for the proposed stochastic primal-dual methods by solving a general convex-concave problem~(\ref{eqn:main}) without bilinear structure or smoothness assumption on $\ell(\x)$ under (weakly local) strong convexity. 

Restart schemes are recently considered to get improved convergence rate under some conditions.
In~\citep{roulet2017sharpness}, restart scheme is analyzed for smooth convex problems under the sharpness and \Holder continuity condition.
In~\citep{dvurechensky2018generalized}, a universal algorithm is proposed for variational inequalities under \Holder conituity condition where the \Holder parameters are unknown. 
Stochastic algorithms are proposed for strongly convex stochastic composite problems in \citep{ghadimi2012optimal,ghadimi2013optimal}.

Finally, we would like to mention that our algorithms and techniques  share many similarities to that proposed in~\citep{ICMLASSG} for solving stochastic convex minimization problems under the local error bound condition. However, their algorithms are not directly applicable to the convex-concave problem~(\ref{eqn:main}) or the problem~(\ref{eqn:primal}) with non-decomposable function $\phi$.  The novelty of this work is the design and analysis of new algorithms that can leverage the weak local strong convexity or more general local error bound condition of the primal minimzation problem~(\ref{eqn:primal}) through solving the convex-concave problem~(\ref{eqn:main}) for enjoying a faster convergence.

\section{Preliminaries}\label{sec:prel}
%We present some preliminaries in this section.
Recall that the problem of interest: 
\begin{align}\label{eq:primal_dual_problem}
\min_{x\in X}\bigg\{P(x) =& \phi( \ell( x ) ) + g(x) \nonumber\\
                            = &\max_{y\in Y} \underbrace{y^{\top} \ell ( x ) - \phi^*(y) + g(x)}\limits_{f(x,y)}\bigg\},
                       %& = \max_{y} \frac{1}{n} \sum_{i=1}^{n} ( y_{i} \ell_{i} ( x ) - \phi_{i}^{*}(y_{i}) ) + g(x)  ,
\end{align}
where $Y= \text{dom}(\phi^*)$. Let $X^*$ denote the optimal set of the primal variable for the above problem, $\Pstar$ denote the optimal primal objective value and $x^* = \arg\min_{z \in X^*} ||x - z||$ is the optimal solution closest to $x$, where $\|\cdot\|$ denotes the Euclidean norm. 
%We can construct the convex conjugate of $\phi$ as follows
%\begin{align}\label{eq:primal_dual_problem}
%P(x) = \phi( \ell( x ) ) + g(x) & = \max_{y} y^{T} \ell ( x ) - \phi^*(y) + g(x)  \nonumber\\
%                       & = \max_{y} \frac{1}{n} \sum_{i=1}^{n} ( y_{i} \ell_{i} ( x ) - \phi_{i}^{*}(y_{i}) ) + g(x)  ,
%\end{align}
%\noindent

Let  $\Pi_{\Omega}[\cdot]$ denote the projection onto the set $\Omega$. Denote by  $\calS_{\epsilon} := \{ x \in X: P(x) - P^* \leq \epsilon \}$ and  $\calL_{\epsilon} := \{ x \in X: P(x) - P^* =\epsilon \}$ denote the $\epsilon$-level set and  $\epsilon$-sublevel set of the primal problem, respectively.    A function $f(x): X\rightarrow\R$ is $L$-smooth if it is differentiable and its gradient is $L$-Lipchitz continuous, i.e., $\|\nabla f(x_1) - \nabla f(x_2)\|\leq L\|x_1 - x_2\|, \forall x_1, x_2\in X$. A differentiable function $f$ is said to have an $(L,v)$-H\"{o}lder continuous gradient with $v\in(0,1]$ iff $\|\nabla f(x_1) - \nabla f(x_2)\|\leq L\|x_1  - x_2\|^v$. When $v=1$, H\"{o}lder continuous gradient reduces to Lipchitz continuous gradient. A function $f$ is called $\lambda$-strongly convex if for any $x_1, x_2\in X$ there exists $\lambda>0$ such that 
\begin{align*}
f(x_1) \geq f(x_2) + \partial f(x_2)^{\top}(x_1 - x_2) + \frac{\lambda}{2}\|x_1 - x_2\|^2,
\end{align*}
where $\partial f(\x)$ denotes any subgradient of $f$ at $x$. 
A more general definition is the uniform convexity. $f$ is uniformly convex with degree $p\geq 2$ if for any $x_1,x_2\in X$ there exists $\lambda>0$ such that 
\begin{align*}
f(x_1) \geq f(x_2) + \partial f(x_2)^{\top}(x_1 - x_2) + \frac{\lambda}{2}\|x_1 - x_2\|^p. 
\end{align*}
For analysis of the proposed algorithms, we need a few basic notions about convex conjugate.
For an extended real-valued convex function $h: \R^d\rightarrow \R\cup\{\infty, -\infty\}$, the convex conjugate of $h$ is defined as 
\begin{align*}
h^*(y) = \max_{x}y^{\top}x  - h(x).
\end{align*}
The convex conjugate of $h^*$ is $h$. Due to the convex duality, if $h^*$ is $\lambda$-strongly convex then $h$ is differentiable and is $(1/\lambda)$-smooth. More generally, if $h^*$ is $p$-uniformly convex then $h$ is differentiable and its gradient is $(L, v)$-H\"{o}lder continuous where $v = \frac{1}{p-1}$, $L = (\frac{1}{\lambda})^{v}$~\citep{Nesterov2015}. 

One of the conditions that allows us to derive a fast rate of $O(1/T)$ for a stochastic algorithm is  that both $g$ and $\phi^*$ are strongly convex, which implies  that $f(x, y)$ is strongly convex in terms of $x$ and strongly concave in terms of $y$. One might regard this as a trivial task  given the $O(1/T)$ result for stochastic strongly convex minimization where a stochastic gradient is available for the objective function to be minimized~\citep{DBLP:journals/ml/HazanAK07,DBLP:journals/jmlr/HazanK11a}. However, the analysis for stochastic strongly convex minimization is not directly applicable to stochastic primal-dual algorithms, as briefly explained later as we present our results.  %That is the reason we are not aware of any studies giving $O(1/T)$ convergence rate of a stochastic primal-dual algorithm for solving a general stochastic convex-concave problem. 

Moreover,  the strong convexity of $g$ can be relaxed to a weak strong convexity of $P$ to derive a similar order of convergence rate, i.e., for any $x\in X$, we have 
\begin{align*}
\dist(x, X^{*}) \leq  c(P(x) - \Pstar )^{1/2},
\end{align*}
where $\dist(x, X^{*}) =\min_{z\in X^*}\|z-x\|_2$ is the distance between $x$ and the optimal set $X^*$. More generally, we can consider a setting in which $P$ satisfies a local error bound (or local growth) condition as defined below. 
\begin{defi}
A function $P(x)$ is said to be satisfied local error bound (LEB) condition if for any $x\in \calS_{\epsilon}$,
\begin{align}\label{eq:leb_condition}
%|| x - \xstar || \leq c ( P(x) - P(\xstar) )^{\theta}  ,
\dist(x, X^{*}) \leq c ( P(x) - \Pstar )^{\theta},
\end{align}
\noindent
where $c>0$ is a constant,  and $\theta \in [0, 1]$ is a parameter. 
\end{defi}
This condition was recently studied in~\citep{yang2018rsg} for developing a faster subgradient method than the standard subgradient method, and was laster considered in~\citep{ICMLASSG} for stochastic convex optimization. A global version of the above condition (known as the global error bound condition) has a long history in mathematical programming~\citep{Pang:1997}. However, exploiting this condition for developing stochastic primal-dual algorithms seems to be new. When $\theta=1/2$, the above condition is also referred to as weakly local strong convexity. When $\theta=0$, it can capture general convex functions as long as $\text{dist}(x, X^*)$ is upper bounded for $x\in \mathcal S_\epsilon$, which is true if $X^*$ is compact or $X$ is compact.

In parallel with the relaxed condition on $P$, we can also relax the smoothness condition on $\phi$ or strong convexity condition on $\phi^*$ to H\"{o}lder continuous gradient condition on $\phi$ or a uniformly convexity condition on $\phi^*$. Under the local error bound condition of $P$ and the H\"{o}lder continuous gradient condition of $\phi$, we are able to develop stochastic primal-dual algorithms with intermediate complexity depending on $\theta$ and $\nu$, which varies from $O(1/\epsilon^2)$ to $O(\log(1/\epsilon))$. 

Formally, we will develop stochastic primal-dual algorithms for solving~(\ref{eq:primal_dual_problem}) under the following assumptions. 
\begin{ass}\label{assumption:general}
For Problem~(\ref{eq:primal_dual_problem}), we assume
\begin{enumerate}
  \item[(1)] There exist $x_{0} \in X$ and $\epsilon_{0} > 0$ such that $P(x_{0}) - \Pstar \leq \epsilon_{0}$;
  \item[(2)]  Let $\nabla_{x} f(x,y;\xi)$ and $\nabla_{y} f(x,y;\xi)$ denote the stochastic subgradient of $f(x,y)$ w.r.t. $x$ and $y$, respectively. 
        There exists constants $M\geq 0$ and $B\geq 0$ such that $|| \nabla_{x} f(x,y;\xi) || \leq M$ and $|| \nabla_{y} f(x,y;\xi) || \leq B$. 
   \item[(3)]  $\phi^*(\cdot)$ is $p$-uniformly convex with $\lambda_\phi>0$ such that $\phi$ has $(L,v)$-\Holder continuous gradient where $v = \frac{1}{p-1}$ and $L=(1/\lambda_\phi)^v$.
   \item[(4)] $\ell(x)$ is $G$-Lipchitz continuous for $x\in X$.
   \item[(5)] One of the following conditions hold: (i) $P(x)$ is $\mu$-strongly convex; (ii) $P(x)$ satisfies the LEB condition for $c>0$ and $\theta\in(0,1]$. 
\end{enumerate}
\end{ass}
{\bf Remark.} Assumption~\ref{assumption:general} (1) assumes that there is a lower bound of $\Pstar$, which is usually satisfied in machine learning problems. Assumption~\ref{assumption:general} (2) is a common assumption usually made in existing stochastic-based methods. Note that we do not assume $g$ and $\phi^*$ have efficient proximal mapping. Instead, we only require a stochastic subgradient of $g$ and $\phi^*$. Assumption~\ref{assumption:general} (3) is a general condition which unifies both smooth and non-smooth assumptions on $\phi$. When $v = 1$, $\phi(\cdot)$ satisfies the classical smooth condition with parameter $L$. When $v = 0$, it is the classical non-smooth assumption on the boundness of the subgradients. We will state our convergence results in terms of $v$ and $L$ instead of $p$ and $\lambda_\phi$. Assumption~\ref{assumption:general} (4)   on the Lipschitz continuity of $\ell(x)$ is more general than assuming a bilinear form $\ell(x) = Ax + b$.   Finally, we note that assuming the strong convexity of $P(x)$ allows us to develop a stochastic primal-dual algorithm with simpler updates.

\section{Main Results}\label{sec:main}
In this section, we will present our main results for solving~(\ref{eq:primal_dual_problem}). Our development is divided into three parts. First, we  present a stochastic primal-dual  algorithm and its convergence result when the primal objective function $P(x)$ is strongly convex and $\phi^*$ is also strongly convex.  Then we extend the result into a more general case, i.e., $P(x)$ satisfying LEB condition and $\phi^*$ is uniformly convex. Lastly, we propose an adaptive  variant with the same order of convergence result when the value of parameter $c$ in LEB condition is unknown, which is also useful for tackling problems without knowing the value of $\theta$. 
For both cases, we assume $P(x_0) - P^* \leq \epsilon_0$.

\subsection{Restarted Stochastic Primal-Dual Algorithm for Strongly Convex $P$}
\setlength{\textfloatsep}{5pt}% Remove \textfloatsep

\begin{algorithm}[t]
\caption{Restarted Stochastic Primal-Dual algorithm for strongly convex $P$: RSPD$^\text{sc}$($x_{0}, S, T, \epsilon_0$)}
\label{alg:RSPDsc}
\begin{algorithmic}[1]
\STATE Initialization: $x_{0}^{(1)} = x_0 \in X$, $y_{0}^{(1)}  = \mathcal A(x_{0}^{(1)})$, $ \eta_{x,1} = \frac{2\epsilon_{0}}{45  M^{2}  }$, $\eta_{y,1} = \frac{2\epsilon_{0}}{45  B^{2}  }$.
\FOR{$s = 1, 2, ..., S$}
\FOR{$t = 0, 1, 2, ..., T_s-1$}
\label{alg1:line:pd_sa_inner_loop_start}

  \STATE $x_{t+1}^{(s)} = \Pi_{X} (x_{t}^{(s)} - \eta_{x,s} \nabla_{x} f (x_{t}^{(s)}, y_{t}^{(s)}; \xi_{t}^{s})) $
  \label{alg1:line:pd_sa_update_x}
  
  \STATE $y_{t+1}^{(s)} = \Pi_{Y} (y_{t}^{(s)} + \eta_{y,s} \nabla_{y} f (x_{t}^{(s)}, y_{t}^{(s)}; \xi_{t}^{s})) $
  \label{alg1:line:pd_sa_update_y}
  
\ENDFOR
\label{alg1:line:pd_sa_inner_loop_end}

\STATE $x_{0}^{(s+1)} = \xbar_{s} = \frac{1}{T} \sum_{t=0}^{T-1} x_{t}^{(s)}$ 
\label{line:pd_sa_update_x_initial}
%\STATE $\ybar_{s} = \frac{1}{T} \sum_{n=1}^{T} y_{t}^{s}$
\STATE $y_{0}^{(s+1)} =\mathcal A(x_{0}^{(s+1)})$  
\label{line:pd_sa_update_y_initial}
\STATE $\eta_{x,s+1} = \frac{\eta_{x,s} }{2}$ and  $\eta_{y,s+1} = \frac{\eta_{y,s} }{2}$, $T_{s+1} = 2T_s$ 
\label{alg1:line:pd_update_R}
\ENDFOR
\STATE Return $\xbar_{S}.$
\end{algorithmic}
\end{algorithm}
%First, we present a stochastic primal-dual  algorithm and its convergence result for strongly convex $P$. 
The detailed updates  of the proposed stochastic algorithm  for strongly convex $P$ are presented in Algorithm~\ref{alg:RSPDsc}, to which we refer as restarted stochastic primal-dual algorithm or RSPD$^\text{sc}$ for short. The algorithm is based on a restarting idea that have been used widely in existing studies~\citep{hazan-20110-beyond,ghadimi2013optimal,ICMLASSG,yang2018rsg}. It runs in epoch-wise and it has two loops.  The steps 3-7 are the standard updates of stochastic primal-dual subgradient method~\citep{Nemirovski:2009:RSA:1654243.1654247}. However, the key difference from these previous studies is that the restarted solution for the dual variable $y$ for the next epoch $s+1$ is computed based on the averaged primal variable for the $s$-th epoch. It is this step that explores the strong convexity of $\phi^*$, which together with the restarting scheme allows us exploring the strong convexity of $P$ to derive a fast convergence rate of $O(1/T)$ with $T$ being the total number of iterations.

Below, we will briefly discuss the path for proving the fast convergence rate of RSPD. We first show that why the standard analysis for strongly convex minimization can not be generalized to the stochastic convex-concave problem to derive the fast convergence rate of $O(1/T)$. Let $\nabla_{x,t} = \nabla_{x} f (x_{t}, y_{t}; \xi_{t})$ and similarly for $\nabla_{y,t}$. A standard convergence analysis for the inner loop (steps 3-6) of Algorithm~\ref{alg:RSPDsc} usually starts from the following inequalities. 
\begin{lem}
For the updates in Step 4 and 5 omitting the subscript $s$, the following holds for any $x\in X, y\in Y$
\begin{align}
& \nabla_{x,t}^{\top}(x_t - x)\leq \frac{\|x_t - x\|^2 -\|x_{t+1} - x\|^2}{2\eta_{x}} + \frac{\eta_{x} M^2}{2}\label{eqn:bprimal}\\
&  \nabla_{y,t}^{\top}(y - y_t)\leq \frac{\|y_t - y\|^2 -\|y_{t+1} - y\|^2}{2\eta_{y}}  + \frac{\eta_{y} B^2}{2}\label{eqn:bdual}.
\end{align}
\end{lem}
For stochastic strongly convex minimization problems in which $y$ is absent in the above inequalities, one can take expectation over~(\ref{eqn:bprimal}) and then apply the $\lambda$-strong convexity of $f(x)$ to get the following  inequality 
\begin{align*}
\E[f(x_t) - f(x)]\leq& \frac{\|x_t - x\|^2 -\|x_{t+1} - x\|^2}{2\eta_{x}} 
 + \frac{\eta_{x} M^2}{2} - \frac{\lambda\|x_t - x\|^2}{2}.
\end{align*}
Based on the above inequalities for all $t=1,\ldots, T$, one can design a particular scheme of step size $\eta_{x,t} = 1/(\lambda t)$ that allows us to derive $\widetilde O(1/T)$ convergence rate. However, such analysis cannot be extended to the primal-dual case.

A naive approach would be taking expectation for both~(\ref{eqn:bprimal}) and~(\ref{eqn:bdual}) for a fixed $x, y$ and applying the $\lambda_x$-strong convexity (resp. $\lambda_y$-strong concavity) of $f(x, y)$ in terms of $x$ (resp. $y$), which yields the following inequalities 
\begin{align*}
\E[f(x_t, y_t) - f(x, y_t)]\leq& \frac{\|x_t - x\|^2 -\|x_{t+1} - x\|^2}{2\eta_{x}}
+ \frac{\eta_{x} M^2}{2} - \frac{\lambda_x\|x_t - x\|^2}{2}.
\end{align*}
\begin{align*}
 \E[f(x_t, y) - f(x_t, y_t)]\leq& \frac{\|y_t - y\|^2 -\|y_{t+1} - y\|^2}{2\eta_{y}}
 + \frac{\eta_{y} B^2}{2} - \frac{\lambda_y\|y_t - y\|^2}{2}.
\end{align*}
It is notable that in deriving the above inequalities, $x$ and $y$ have to be independent of $\xi_1, \ldots, \xi_T$.

By adding the above inequalities together and applying the same analysis for the R.H.S with $\eta_{x, t} = 1/(\lambda_x t)$ and $\eta_{y,t} = 1/(\lambda_{y}t)$, we can obtain the following inequalities for any fixed $y\in Y$ and $x\in X$ independent of $\xi_1, \ldots, \xi_T$: 
\begin{align}\label{eqn:naive}
\E\left[(f(\hat x_T, y) - f(x, \hat y_T))\right]\leq&  O\left(\frac{\log T}{T}\right),
\end{align}
where $\hat x_T = \sum_{t=0}^{T-1}x_t/T$ and $\hat y_T = \sum_{t=0}^{T-1} y_t/T$. 
However, the above inequality does not imply the convergence for the standard definition of primal-dual gap of $\max_{x\in X, y\in Y}(f(\hat x_T, y) - f(x, \hat y_T))$ or even the primal objective gap $P(\hat x_T) - \min_{x\in X}P(x)$. The main obstacle is that we cannot set $y = \arg\max_{y\in Y}f(\hat x_T, y)$ which will make $y$ depend on $\xi_1, \ldots, \xi_T$ and hence make the expectional analysis fail. % $\max_{x\in X, y\in Y}\E[(f(\hat x_T, y) - f(x, \hat y_T))]$ could be much smaller than  $\E[\max_{x\in X, y\in Y}(f(\hat x_T, y) - f(x, \hat y_T))]$. 
It would be worth noting that following \citep{gidel2016frank}, one could derive the upper bound of primal-dual gap of $(\hat{x}_{T}, \hat{y}_{T})$ by 
$\max_{y \in Y} f(\hat{x}_{T}, y) - \min_{x \in X} f(x, \hat{y}_{T}) \leq \sqrt{2} P_{\mathcal{L}} \sqrt{f(\hat{x}_{T}, y^*) - f(x^*, \hat{y}_{T})}$ (see Equation (5), (13) and (14) therein),
where $P_{\mathcal{L}}$ can be upper bounded by a constant and $y^* \in \arg\max_{y \in Y} f(x^*, y^*)$.
Even if one sets $x = x^*$ and $y = y^*$ in (\ref{eqn:naive}), the convergence rate of primal-dual gap is only of $O(\sqrt{\log(T) / T})$, which is not what we pursue.

Another approach that gets around of the issue introduced by taking the expectation is by using high probability analysis. To this end, one can use concentration inequalities to bound the martingale difference sequence $\sum_{t=1}^T(\nabla_{x} f (x_{t}, y_{t}; \xi_{t})  - \nabla_{x} f (x_{t}, y_{t}))^{\top}(x_t - x) $ and $\sum_{t=1}^T(\nabla_{y} f (x_{t}, y_{t}; \xi_{t})  - \nabla_{y} f (x_{t}, y_{t}))^{\top}(y - y_t) $ for a fixed $x$ and $y$~\citep{DBLP:conf/nips/KakadeT08}. However, in order to prove the primal objective gap $P(\hat x_T) - P^*$ one has to bound the later martingale difference sequence  for any possible $y\in Y$ so that one can get $P(\hat x_t)$ from $\max_{y\in Y}f(\hat x_T, y)$. A standard approach for achieving this high probability bound is by using a covering number argument for the set $Y$. However, this will inevitably introduce dependence on the dimensionality of $y$. For example, an $\epsilon$-cover of a bounded ball of radius $R$ in $\R^n$ has cardinality of  $O((R/\epsilon)^n)$, and of a simplex in $\R^n$ has cardinality of $O((1/\epsilon)^{n-1})$. 

To tackle the aforementioned   challenges for both exceptional analysis and high probability analysis, we develop a different analysis for the proposed RSPD algorithm in order to achieve a faster convergence rate of $O(1/T)$ without explicit dependence on the dimensionality  of $y$. In this subsection, we will focus on expectional  convergence result, which will be extended to high probability convergence in next subsection. Our expectional analysis is build on the following lemma that is used to derive $O(1/\sqrt{T})$ convergence rate in the literature~\citep{Nemirovski:2009:RSA:1654243.1654247}. 
\begin{lem}\label{lemma:convergence_RSPDsc_per_stage}
Let the 
Lines 4 and 5 of Algorithm~\ref{alg:RSPDsc} run for $T$ iterations   with a fixed step size $\eta_x$ and  $\eta_y$.
%Denote $D = || x_{0} - \xzeroproj ||$.
Then %with the probability at least $1 - \tildedelta$ where $\tildedelta \in (0, 1)$, we have
\begin{align}
\label{eq1:convergence_per_stage}
\E[\max_{y\in Y} f(\bar x_T, y) - f(x^*, \bar y_T)]
\leq &
\frac{\E[|| x^* - x_{0} ||^{2}] }{\eta_x   T  }
+  \frac{\E[ || \hat y_T - y_{0} ||^{2}] }{\eta_y T  }+ \frac{5\eta_x M^2}{2}    + \frac{5\eta_y B^2}{2}, %+ \frac{ 8 ( M + B ) R \sqrt{2 \log\frac{1}{\tildedelta}} }{ \sqrt{T} } ,
\end{align}
where $\bar x_T = \sum_{t=0}^{T-1} x_{t}/T$,  $\bar y_T = \sum_{t=0}^{T-1} y_{t}/T$,   $\hat y_T = \arg\max_{y\in Y}f(\bar x_T, y)$ and $x^*\in X^*$.
\end{lem}
{\bf Remark:} A nice property of the above result is that the max over $y$ in the L.H.S is  taken before expectation. 

Nevertheless, a simple approach for setting the step size as $O(1/\sqrt{T})$ still yields a convergence rate of $O(1/\sqrt{T})$ by assuming the size of $Y$ is bounded~\citep{Nemirovski:2009:RSA:1654243.1654247}. The proposed RSPD algorithm has the special design of computing the restarted solutions and setting the step sizes, which together allows us to achieve $O(1/T)$ convergence rate as stated in the following theorem. The key idea is that by using $y_0^{(s+1)}=\mathcal A(x_0^{(s+1)})$ as a restarted point for the dual variable, we are able to connect $\|\hat y_T - y_0\|$ to $P(x_0^{(s)})  - P^*$ by using the strong convexity of $P$ and of $\phi^*$. 
The convergence result of RSPD$^\text{sc}$ is presented below. 
\begin{thm}\label{thm:RSPDsc}
Suppose that Assumption~\ref{assumption:general} holds with $v=1$ and $P(x)$ being $\mu$-strongly convex.
By setting 
$S =  \lceil \log(\frac{\epsilon_{0}}{\epsilon})  \rceil$ and 
$T_1 = \frac{\max \{  
405M^2, 
 810 L^2G^2B^2
\}}{\mu\epsilon_{0}}$, then Algorithm~\ref{alg:RSPDsc} guarantees that $\E[P(\xbar_{S}) - \Pstar ]\leq  \epsilon$. The total number of iterations is 
$ O( \frac{1}{\mu\epsilon})$.
\end{thm}
{\bf Remark.} The equivalent convergence rate of the above result is $O(1/(\mu T))$ given a total number of iterations $T$. This matches the state-of-the-art convergence result for stochastic strongly convex minimization~\citep{hazan-20110-beyond}. Our algorithm can be applied to solving~(\ref{eqn:primal}) for non-decomposable $\phi$. In contrast to the standard stochastic primal-dual subgradient method, the additional computational overhead in RSPD$^\text{sc}$ is introduced by computing the  restarted points $y_0^{s+1} = \mathcal A(x_0^{(s+1)})$. However, such computation only happens for  a logarithmic number of times in the order of $O(\log(1/\epsilon))$. We defer the discussion on the total time complexity of RSPD to the next section for some particular applications.

\begin{proof}%[Proof of Theorem~\ref{thm:RSPDsc}]  % proof of thm 2

To prove Theorem~\ref{thm:RSPDsc}, we first need Lemma~\ref{lemma:convergence_RSPDsc_per_stage}.
Its proof will be given in Appendix \ref{app:sec:proof:lem:convergence_RSPDsc_per_stage}.

\yancomment{
To prove Theorem~\ref{thm:RSPDsc}, we first need the following lemma, whose proof will be given in Appendix.
\begin{lem}\label{lemma:3}
%Let us consider the update within the $s$-th stage of Algorithm~\ref{alg:RSPDsc}, where $s = 1, ..., S$ (so we omit the $s$-index).
%Run 
Let the Lines 4 %~\ref{alg1:line:pd_sa_update_x} 
% ( $x_{t+1} = \Pi_{X \cap \calB(x_{0}, R) } ( x_{t} - \eta \nabla_{x} f(x_{t}, y_{t}; \xi_{t}) ) $ ) 
and 5
%~\ref{alg:line:pd_sa_update_y} 
% ( $y_{t+1} = \Pi_{\calB(y_{0}, R) } ( y_{t} + \eta \nabla_{y} f(x_{t}, y_{t}; \hatxi_{t}) ) $ )
of Algorithm~\ref{alg:RSPDsc} run for $T$ iterations 
by %setting $x_{0} \in X$ with a 
fixed step size $\eta_x$ and $\eta_y$. 
%$y_{0} = \nabla \phi(\ell(x_{0}))$ with a fixed step size $\eta_y$.
%Denote $D = || x_{0} - \xzeroproj ||$.
Then %with the probability at least $1 - \tildedelta$ where $\tildedelta \in (0, 1)$, we have
\begin{align}
\label{eq1:convergence_per_stage}
     & \E[\max_{y\in Y} f(\bar x_T, y) - f(x^*, \bar y_T)]\leq   \frac{\E[|| x^* - x_{0} ||^{2}] }{\eta_x   T  }+  \frac{\E[ || \hat y_T - y_{0} ||^{2}] }{\eta_y T  }+ \frac{5\eta_x M^2}{2}    + \frac{5\eta_y B^2}{2}, %+ \frac{ 8 ( M + B ) R \sqrt{2 \log\frac{1}{\tildedelta}} }{ \sqrt{T} } ,
\end{align}
where $\bar x_T = \sum_{t=0}^{T-1} x_{t}/T$,  $\bar y_T = \sum_{t=0}^{T-1} y_{t}/T$,   $\hat y_T = \arg\max_{y\in Y}f(\bar x_T, y)$ and $x^*\in X^*$.
\end{lem}
}

Let $\epsilon_{s} = \frac{\epsilon_{s-1}}{2}$, by the setting of Algorithm~\ref{alg:RSPDsc}, we know $\eta_{x, s+1} = \frac{ 2 \epsilon_{s}}{45 M^2}$, $\eta_{y, s+1} = \frac{ 2 \epsilon_{s}}{45 B^2}$, and $x_{0}^{(s+1)} = \xbar_{s} = \frac{1}{T_{s}} \sum_{t=1}^{T_s} x_{t}^{(s)}$ for $s = 0, 1, \dots$.
We will show $\E[P(x_{0}^{(s+1)})] - P^* \leq \epsilon_{s}$ by induction for $s = 0, 1, \dots$. It is easy to verify
$
\E[P(x_{0}^{(1)})] - P^* \leq \epsilon_{0}
$
for a sufficiently large $\epsilon_{0}$ according to Assumption~\ref{assumption:general}.
Next, we need to show that conditional on $\E[P(x_{0}^{(s)}) ]- P^* \leq \epsilon_{s-1} $, then  we have 
$$
\E[P(x_{0}^{(s+1)})] - P^*\leq \epsilon_{s}.
$$
Consider the update of $s$-th stage.
By Lemma~\ref{lemma:convergence_RSPDsc_per_stage} for the update of $s$-the stage,
we have 
\begin{align*}%\label{eq1:convergence_induction}
 \E[f(\xbar_{s}, \yhat(\xbar_{s})) - f(x^*, \ybar_s)]
\leq \frac{\E[|| x^* - x_{0}^{(s)} ||^{2}] }{\eta_{x,s}  T_s } + \frac{\E[|| \yhat(\xbar_{s}) - y_{0}^{(s)} ||^{2}]}{\eta_{y,s}  T_s  }   + \frac{5  \eta_{x,s} M^{2}  }{ 2 }  + \frac{5  \eta_{y, s} B^{2}}{2}.
\end{align*}
Since $P(\xbar_{s}) =  f(\xbar_{s}, \yhat(\xbar_{s}))$ and $ P(x^*) =  \max_{y \in Y} f(x^*, y) \geq f(x^*, \ybar_s)$, we have
\begin{align}\label{eq1:convergence_induction}
       \E[P(\xbar_{s}) - P^*]
\leq & 
\frac{\E[|| x^* - x_{0}^{(s)} ||^{2}] }{\eta_{x,s}  T_s } + \frac{\E[|| \yhat(\xbar_{s}) - y_{0}^{(s)} ||^{2}]}{\eta_{y,s}  T_s }   + \frac{5  \eta_{x,s} M^{2}  }{ 2 }  + \frac{5  \eta_{y, s} B^{2}}{2}.
\end{align}

For the first term on the RHS of (\ref{eq1:convergence_induction}), by the strong convexity of $P(x)$ and the condition $\E[P(x_{0}^{(s)}) ]- P^* \leq \epsilon_{s-1} $ we have
\begin{align*}
       \E[|| x^* - x_{0}^{(s)} || ^2]
\leq  
     \E\bigg[  \frac{2 }{\mu} ( P(x_{0}^{(s)}) - P^*)\bigg]  \nonumber
\leq \frac{2\epsilon_{s-1} }{\mu}.
\end{align*}
For the second term on the RHS of (\ref{eq1:convergence_induction}), %it holds due to
%{\color{red}\begin{align*}
%       || \yhat(\xbar_{s}) - y_{0} ||^{2}   
%=    
%       || \nabla \phi(\ell(\xbar_{s})) - \nabla \phi(\ell(x_{0})) ||^{2}   
%\leq  
%       L^{2} || \ell(\xbar_{s}) - \ell(x_{0}) ||^{2v}   
%\leq 
%       L^{2} G^{2v} || \xbar_{s} - x_{0} ||^{2v}   
%\leq 
%       L^{2} G^{2v} R_s^{2v}   ,
%\end{align*}
%}
\begin{align*}
       || \yhat(\xbar_{s}) - y_{0}^{(s)} ||^{2}   
=    & 
       || \nabla \phi(\ell(\xbar_{s})) - \nabla \phi(\ell(x_{0}^{(s)})) ||^{2}   \\
\leq & 
       L^{2} || \ell(\xbar_{s}) - \ell(x_{0}^{(s)}) ||^{2v}   \\
= & 
       L^{2} || \ell(\xbar_{s}) - \ell(x_{0}^{(s)}) ||^{2}   \\
%\leq & 
%       L^{2} G^{2v} || \xbar_{s} - x_{0}^{(s)} ||^{2v},\\
  = &   L^{2} G^{2} || \xbar_{s} - x_{0}^{(s)} ||^{2},
%\leq 
       %L^{2} G^{2v} R_s^{2v}   ,
\end{align*}
%where the first equality is due to~(\ref{eq:dual_variable_init}), 
where the first equality is due to the set up of the algorithm and Lemma~\ref{lemma:supp:1}, 
the second equality is due to $\phi(\cdot)$ is smooth ($v=1$).
Since $P(x)$ is strongly convex with parameter $\mu>0$, its optimal solution $x_*$ is unique, then we have
\begin{align*}
     \E[  || \yhat(\xbar_{s}) - y_{0}^{(s)} ||^{2}]   \leq &   2L^{2} G^{2}( \E[|| \xbar_{s} - x_* ||^{2}] + \E[|| x_* - x_{0}^{(s)} ||^{2}])\\
  \leq &   \frac{4L^{2} G^{2}}{\mu}( \E[P(\xbar_{s}) - P^* ] + \E[ P(x_{0}^{(s)}) - P^*]) \\
    \leq &   \frac{4L^{2} G^{2}}{\mu}( \E[P(\xbar_{s}) - P^* ] + \epsilon_{s-1}  ]).   
\end{align*}

Then the inequality (\ref{eq1:convergence_induction}) becomes
\begin{align*}%\label{eq1:convergence_induction}
       \E[P(\xbar_{s}) - P^*]
\leq & 
\frac{ 2\epsilon_{s-1} }{\mu \eta_{x,s}  T_s } + \frac{\frac{4L^{2} G^{2}}{\mu}( \E[P(\xbar_{s}) - P^* ] + \epsilon_{s-1} ])}{\eta_{y,s}  T_s  }   + \frac{5  \eta_{x,s} M^{2}  }{ 2 }  + \frac{5  \eta_{y, s} B^{2}}{2}.
     %  \frac{R_{s}^{2} }{\eta_{x,s} T } + \frac{L^{2} G^{2v} R_{s}^{2v} }{\eta_{y,s} T } + \frac{5  \eta_{x,s} M^{2}  }{ 2 }  + \frac{5  \eta_{y, s} B^{2}}{2}.
\end{align*}

By the setting of $ \eta_{x,s} = \frac{2\epsilon_{s-1}}{45  M^{2}  }$, $\eta_{y,s} = \frac{2\epsilon_{s-1}}{45  B^{2}  }$ and 
$T_s = \frac{\max \{  
405M^2, 
810 L^2G^2B^2
\}}{\mu\epsilon_{s-1}}$,
we know $\frac{\frac{4L^{2} G^{2}}{\mu}}{\eta_{y,s}  T  }  \leq \frac{1}{9}$, then
%\begin{align*}%\label{eq1:convergence_induction}
%       \frac{8}{9}\E[P(\xbar_{s}) - P^*]
%\leq & 
%\frac{\frac{c^2 \epsilon^2_{s-1} }{\epsilon}}{\eta_{x,s}  T } + \frac{ \epsilon_{s-1} + \epsilon }{9}   + \frac{5  \eta_{x,s} M^{2}  }{2 }  + \frac{5  \eta_{y, s} B^{2}}{2} + \epsilon.
%\end{align*}
\begin{align*}
       \E[P(\xbar_{s}) - P^*]
\leq & 
 \frac{9 \epsilon_{s-1}}{4 \eta_{x,s} \mu T_s } + \frac{ \epsilon_{s-1}  }{8}   + \frac{45  \eta_{x,s} M^{2}  }{ 16 }  + \frac{45  \eta_{y, s} B^{2}}{16}  \leq \frac{\epsilon_{s-1}}{2} =  \epsilon_s.
\end{align*}
Therefore, by induction, after running $S = \lceil \log(\frac{\epsilon_{0}}{\epsilon}) \rceil$ stages, we have
$$
\E[P(\xbar_{S}) - P^*] \leq \epsilon_{S} = \epsilon.
$$
The total iteration complexity is $ \sum_{s=1}^{S} T_s = O (\frac{1}{\epsilon})$.
\end{proof}  % proof of thm 2

%It is worth  noting that the total complexity of RSPD, which consists of $O(1/\epsilon)$ stochastic updates of $x_t, y_t$ and $\log(1/\epsilon)$ computations of restarted point $y_0^{s+1}$
%When $P(x)$ is strongly convex and $\phi(\cdot)$ is smooth, RSPD-sc can enjoy an optimal convergence result with iteration complexity $O(1/\epsilon)$, which is faster than the rate $O(1/\epsilon^2)$ in terms of minimizing $P(x)$ for solving the convex-concave problem~(\ref{eqn:main}). It is worth mentioning that we do not need the smoothness assumption for primal objective function $P(x)$.

\subsection{RSPD Algorithm under the LEB condition}

In the previous subsection, we introduce the RSPD$^\text{sc}$ algorithm for solving problem~(\ref{eqn:main}) when the objective function $P(x)$ is strongly convex and $\phi(\cdot)$ is $L$-smooth. However, these conditions are sometimes too strong for many machine learning problems. In this subsection, we will relax these strong conditions by assuming that $P(x)$ satisfies the LEB condition (\ref{eq:leb_condition}) and $\phi(\cdot)$ has $(L,v)$-\Holder continuous gradient with $v\in[0,1]$. We will develop a different variant of RSPD that  also has high probability convergence guarantee.

Denote by $ \calB_x(x_{0}, R) = \{x\in X: \|x - x_0\|\leq R\}$ a ball centered at $x_0$ with a  radius $R$ intersected with $X$, and similarly by    $ \calB_y(y_{0}, R) = \{y\in Y: \|y - y_0\|\leq R\}$ a ball centered at $y_0$ with a  radius $R$ intersected with $Y$. The second variant of the RSPD algorithm for solving problem~(\ref{eqn:main}) is summarized in Algorithm~\ref{alg:restart_primal_dual_algorithm_sa}, which is similar to the RSPD$^\text{sc}$ algorithm except that the iterates are projected to bounded balls centered at the initial solutions of each epoch. This complication on the updates is introduced for the purpose of high-probability analysis, which also allows us to tackle problems that satisfies the LEB condition with $\theta>1/2$. 
%We would emphsize that $\Pi_{X \cap \calB(x_{0}, R_{s})} (x) = \arg \min_{z \in X \cap \calB(x_{0}, R_{s})} || x - z ||_{2}$ in Algorithm~\ref{alg:restart_primal_dual_algorithm_sa} is the projection onto $X \cap \calB(x_{0}, R_{s})$, where $\calB(x_{0}, R_{s})$ is a Euclidean ball whose center is $x_{0}$ with a radius $R_{s}$. 
After each epoch, the proposed RSPD algorithm reduces the radius of the Euclidean ball. It is notable that this ball shrinkage technique is not new and has already used in Epoch-SGD method~\citep{hazan-20110-beyond} for high probability bound analysis. We set the same value of initial radius for primal variable $x$ and dual variable $y$ in RSPD algorithm for the convenience of analysis. However, one can use different values but the same order of convergence result will be obtained by changing the analysis slightly. Another feature of RSPD that is different from  RSPD$^\text{sc}$ is that  RSPD uses a constant number of iterations in the inner loop in order to accommodate the local error bound condition. 

\begin{algorithm}[t]
\caption{RSPD($x_{0}, S, T, R_{x, 1}, \epsilon_0$)}
\label{alg:restart_primal_dual_algorithm_sa}
\begin{algorithmic}[1]
\STATE Initialization: 
       $x_{0}^{(1)} = x_0 \in X$, 
       $y_{0}^{(1)} \in \mathcal A(x_{0}^{(1)})$, 
       $R_{x, 1} \geq \frac{c \epsilon_{0}}{\epsilon^{1 - \theta}}$, 
       $R_{y, 1} = L G^v R_{x, 1}^v$,
       $\eta_{x,1} = \frac{\epsilon_{0}}{40  M^{2}  }$, 
       $\eta_{y,1} = \frac{\epsilon_{0}}{40  B^{2} }$.
\FOR{$s = 1, 2, ..., S$}

\FOR{$t = 0, 1, 2, ..., T-1$}
\label{alg:line:pd_sa_inner_loop_start}

  \STATE Compute $\mathcal G_{x, t} =  \nabla_{x} f (x_{t}^{(s)}, y_{t}^{(s)}; \xi_{t}^{s})$ and $\mathcal G_y = \nabla_{y, t} f (x_{t}^{(s)}, y_{t}^{(s)}; \xi_{t}^{s})$

  \STATE $x_{t+1}^{(s)} = \Pi_{\calB_x(x_{0}^{(s)}, R_{x, s})} (x_{t}^{(s)} - \eta_{x,s} \mathcal G_{x, t} )$
  
  \STATE $y_{t+1}^{(s)} = \Pi_{\calB_y(y_{0}^{(s)}, R_{y, s})} (y_{t}^{(s)} + \eta_{y,s} \mathcal G_{x, t} )$
  \label{alg2:line:pd_sa_update_y}
  
\ENDFOR
\label{alg:line:pd_sa_inner_loop_end}
\STATE $x_{0}^{(s+1)} = \xbar_{s} = \frac{1}{T} \sum_{t=0}^{T-1} x_{t}^{(s)}$
\STATE $y_{0}^{(s+1)} =\mathcal A(x_{0}^{(s+1)})$  
  \label{alg2:line:pd_sa_update_x}

\STATE $R_{x, s+1} = \frac{R_{x, s}}{2}$, 
       $R_{y, s+1} = \frac{R_{y, s}}{2^v}$,
       $\eta_{x,s+1} = \frac{\eta_{x,s} }{2}$ and  
       $\eta_{y,s+1} = \frac{\eta_{y,s} }{2}$ 
\label{alg:line:pd_update_R}
\ENDFOR
\STATE Return $\xbar_{S}.$
\end{algorithmic}
\end{algorithm}

We summarize the theoretical result of Algorithm~\ref{alg:restart_primal_dual_algorithm_sa} with a high probability bound in the following theorem.
\begin{thm}\label{theorem:convergence_rspd}
Suppose that Assumption~\ref{assumption:general} holds and $P(x)$ obeys the LEB condition (\ref{eq:leb_condition}).
Given $\delta \in (0, 1)$, 
let %$\tildedelta = \frac{\delta}{S}$, 
$S =  \lceil \log(\frac{\epsilon_{0}}{\epsilon})  \rceil$, 
$\tildedelta = \delta / S$,
$R_{1} = O ( \frac{c\epsilon_{0}}{\epsilon^{1 - \theta}} )$ 
% with $R_{1} \geq \frac{c\epsilon_{0}}{\epsilon^{1 - \theta}}$ 
and 
%$T \geq \frac{32(M+B)^{2} L^{2} G^{2v} c^{2v} }{\epsilon^{2(1 - v\theta)}}$
% $T = \max \bigg\{ \frac{320 B^{2} L^{2} G^{2v} R_{1}^{2v}}{\epsilon_{0}^{2v}\epsilon^{2-2v}}, 
%  \frac{2048 ( M + B )^{2}  R_{1}^{2} \log ({S}/{\delta})}{\epsilon_{0}^{2}} \bigg\}$.
\begin{align*}
T
\geq & 
\max \bigg\{ 
\frac{ 320 M^2 R_{x, s}^2 }{ \epsilon_{s-1}^2 } , 
\frac{ 320 B^2 L^2 G^{2v} R_{x, s}^{2v} }{ \epsilon_{s-1}^2 } , 
\frac{ 8192 \log(\frac{1}{\tildedelta}) M^2 R_{x, s}^2 }{ \epsilon_{s-1}^2 } ,
\frac{ 8192 \log(\frac{1}{\tildedelta}) B^2 L^2 G^{2v} R_{x, s}^{2v} }{ \epsilon_{s-1}^2 } \bigg\}   .
\end{align*}
Algorithm~\ref{alg:restart_primal_dual_algorithm_sa} guarantees that $P(\xbar_{S}) - \Pstar \leq 2 \epsilon$ with at least probability $1 - \delta$. 
The total number of iterations is 
$\widetilde O( \frac{1}{\epsilon^{2(1 - v\theta)}})$, 
%$O( \lceil \log( \frac{\epsilon_{0}}{\epsilon}) \rceil \frac{c^{2}}{\epsilon^{2(1-\theta)}} \max\{ 18 (B^{2} + M^{2}) (1 + L^{2} ||A||^2), 288 ( 2 M + B ( 1 + L ||A|| ) )^2 \log(\frac{1}{\tildedelta}) \} )$.
where $\widetilde O$ suppresses a logarithmic factor.  
\end{thm}
{\bf Remark.} %Different from the expectational bound in Theorem~\ref{thm:RSPDsc}, we provide high probability bound in Theorem~\ref{theorem:convergence_rspd}. 
When $v\theta>0$, RSPD enjoys the improved iteration complexity than $O(1/\sqrt{T})$. When $v=1$ (i.e., $\phi(\cdot)$ is smooth), if $\theta = \frac{1}{2}$ (e.g., $P(x)$ is (weakly) strongly convex), then RSPD enjoys the iteration complexity of $ O(\log(1/\epsilon)/\epsilon)$, which is only worse by a logarithmic factor than the expectional convergence result in Theorem~\ref{thm:RSPDsc} for strongly convex $P$. When $v=0$ or $\theta=0$ (i.e., $\phi$ is non-differentiable with no H\"{o}lder continuous gradient or $P$ does not obey the error bound condition), the convergence rate reduces to the standard $\widetilde O(1/\sqrt{T})$. 
%When $\theta = 1$ (e.g., $P(x)$ piecewise linear convex), then RSPD enjoys the linear convergence of $\widetilde O(\log(1/\epsilon))$. 

\begin{proof} %(of Theorem~\ref{theorem:convergence_rspd})   % proof of thm 3

To prove Theorem~\ref{theorem:convergence_rspd}, we first present the following two lemmes.
The first one presents Azuma's inequality which handles martingale difference sequence.
The second one analyzes the behaviour of the update within a stage of Algorithm~\ref{alg:restart_primal_dual_algorithm_sa}.
Proof of Lemma \ref{lemma:convergence_rspd_per_stage} is in Appendix \ref{app:sec:proof:lem:convergence_rspd_per_stage}.

\begin{lem}\label{lemma:azuma}
(Azuma's inequality)
Let $X_{1}, ..., X_{T}$ be the martingale difference sequence.
Suppose that $| X_{t} | \leq b$. 
Then for $\delta > 0$ we have
$$
\Pr \bigg( \sum_{t=1}^{T} X_{t} \geq b \sqrt{2 T \log \frac{1}{\delta}} \bigg) \leq \delta  .
$$
\end{lem}

\begin{lem}\label{lemma:convergence_rspd_per_stage}
%Let us consider the update within the $s$-th stage of Algorithm~\ref{alg:restart_primal_dual_algorithm_sa}, where $s = 1, ..., S$ (so we omit the $s$-index).
%Run Line~\ref{alg:line:pd_sa_update_x} 
% ( $x_{t+1} = \Pi_{X \cap \calB(x_{0}, R) } ( x_{t} - \eta \nabla_{x} f(x_{t}, y_{t}; \xi_{t}) ) $ ) 
%and Line~\ref{alg:line:pd_sa_update_y} 
% ( $y_{t+1} = \Pi_{\calB(y_{0}, R) } ( y_{t} + \eta \nabla_{y} f(x_{t}, y_{t}; \hatxi_{t}) ) $ )
%of Algorithm~\ref{alg:restart_primal_dual_algorithm_sa} for $T$ iterations 
%by setting $x_{0} \in X$ with a fixed step size $\eta_x$ and $y_{0} = \nabla \phi(\ell(x_{0}))$ with a fixed step size $\eta_y$.
%Denote $D = || x_{0} - \xzeroproj ||$.
%Then with the probability at least $1 - \tildedelta$ where $\tildedelta \in (0, 1)$, we have
%\begin{align}
%\label{eq:convergence_per_stage}
%       f(\xbar, y) - f(x, \ybar) 
%\leq   \frac{|| x - x_{0} ||^{2} }{\eta_x   T  }  +  \frac{ || y - y_{0} ||^{2} }{\eta_y T  }+ \frac{5\eta_x M^2}{2}    + \frac{5\eta_y B^2}{2} + \frac{ 8 ( M + B ) R \sqrt{2 \log\frac{1}{\tildedelta}} }{ \sqrt{T} } ,
%\end{align}
%where $\xbar = \sum_{t=0}^{T-1} x_{t}$ and $\ybar = \sum_{t=0}^{T-1} y_{t}$.% and $\varepsilon \in (0,1)$.

Let the Lines 4, 5, and 6 of Algorithm~\ref{alg:restart_primal_dual_algorithm_sa} run for $T$ iterations
by fixed step size $\eta_x$ and $\eta_y$ starting from $x_0$ and $y_0$. 
Then with the probability at least $1 - \tildedelta$ where $\tildedelta \in (0, 1)$, we have
\begin{align}
\label{eq1:convergence_per_stage}
     \max_{y \in Y \cap \calB(y_0, R_y)} f(\bar x_T, y) - f(x, \bar y_T)
     \leq &
     \frac{|| x - x_{0} ||^{2} }{\eta_x   T  }
     +  \frac{|| \hat y_T - y_{0} ||^{2} }{\eta_y T  }
     + \frac{5\eta_x M^2}{2}    
     + \frac{5\eta_y B^2}{2} 
     \nonumber\\ &
     + \frac{ 4 ( M R_x + B R_y ) \sqrt{2 \log\frac{1}{\tildedelta}} }{ \sqrt{T} } ,
\end{align}
where $\bar x_T = \sum_{t=0}^{T-1} x_{t}/T$,  
$\bar y_T = \sum_{t=0}^{T-1} y_{t}/T$,   
$\hat y_T = \arg\max_{y \in Y \cap \calB(y_0, R_y)} f(\bar x_T, y)$ and any fixed $x \in X \cap \calB(x_0, R_x)$.
\end{lem}

Now we proceed to proof of Theorem~\ref{theorem:convergence_rspd}.
Let $\epsilon_{s} = \frac{\epsilon_{s-1}}{2}$, by the setting of Algorithm~\ref{alg:restart_primal_dual_algorithm_sa}, we know $R_{s+1} = \frac{R_1}{2^{s}}\geq \frac{c\epsilon_{s}}{\epsilon^{1-\theta}}$, $\eta_{x, s+1} = \frac{\epsilon_{s}}{40M^2}$, $\eta_{y, s+1} = \frac{\epsilon_{s}}{40B^2}$, and $x_{0}^{(s+1)} = \xbar_{s} = \frac{1}{T} \sum_{t=1}^{T} x_{t}^{(s)}$ for $s = 0, 1, \dots$.
We will show $P(x_{0}^{(s+1)}) - P^* \leq \epsilon_{s} + \epsilon$ by induction for $s = 0, 1, \dots$ with a high probability. It is easy to verify
$
%P(x_{0}^{(1)}) - P(x_{0, \epsilon}^{(1), \dagger}) \leq \epsilon_{0}
P(x_{0}^{(1)}) - P^* \leq \epsilon_{0} + \epsilon
$
for a sufficiently large $\epsilon_{0}$ according to Assumption~\ref{assumption:general}.
Next, we need to show that conditional on $P(x_{0}^{(s)}) - P^* \leq \epsilon_{s-1} + \epsilon$, we have 
$$
P(x_{0}^{(s+1)}) - P_*\leq \epsilon_{s} + \epsilon
$$
with a high probability.

Consider the update of the $s$-th stage.
Define $\yhat(\xbar_{s}) = \arg\max_{y \in Y } f(\xbar_{s}, y)$ and
$x_{0, \epsilon}^{(s), \dagger} = \arg\min_{ x \in \calS_{\epsilon} } \| x - x_0^{(s)} \|$.
We would like to show that both $\| x_{0, \epsilon}^{(s), \dagger} - x_{0}^{(s)} \| \leq R_x$ and $\| \yhat(\xbar_{s}) - y_0^{(s)} \| \leq R_y$ always hold, 
so that we are able to plug $x = x_{0, \epsilon}^{(s), \dagger}$ and $y = \yhat(\xbar_{s})$ into (\ref{eq1:convergence_per_stage}) in Lemma~\ref{lemma:convergence_rspd_per_stage}.
To this end, we have for $x_{0, \epsilon}^{(s), \dagger}$,
\begin{align*}%\label{eq:radius_constraint}
       || x_{0, \epsilon}^{(s), \dagger} - x_{0}^{(s)} || 
\leq & 
       \frac{\dist( x_{0, \epsilon}^{(s), \dagger}, X^* ) }{\epsilon} ( P(x_{0}^{(s)}) - P(x_{0, \epsilon}^{(s), \dagger}) )  \nonumber\\
\leq &
       \frac{ c ( P( x_{0, \epsilon}^{(s), \dagger} ) - P^* )^{\theta} }{\epsilon}  \epsilon_{s-1}  \nonumber\\
\leq    &
       \frac{c \epsilon_{s-1} }{\epsilon^{1 - \theta}} 
\leq 
       R_{x, s},
\end{align*}
where the first inequality is due to Lemma 4 in~\citep{yang2018rsg},
the second inequality is due to~(\ref{eq:leb_condition}) and
the third inequality is due to $x_{0, \epsilon}^{(s), \dagger} \in \calS_\epsilon$.

For $\yhat(\xbar_{s})$, we have
\begin{align*}
       || \yhat(\xbar_{s}) - y_{0} ||  
=    & 
       || \nabla \phi(\ell(\xbar_{s})) - \nabla \phi(\ell(x_{0})) ||   \\
\leq & 
       L || \ell(\xbar_{s}) - \ell(x_{0}) ||^v   \\
\leq & 
       L G^v || \xbar_{s} - x_{0} ||^v   
\leq 
       L G^v R_{x, s}^v
= 
       R_{y, s},
\end{align*}
where the first equality is due to the set up of the algorithm and Lemma~\ref{lemma:supp:1},
the first inequality is due to the $(L, v)$-\Holder continuous gradients of $\phi$ (Assumption \ref{assumption:general} (3)), 
the second inequality is due to $G$-Lipschitz continuity of $\ell$ (Assumption \ref{assumption:general} (4)), and
the last equality is due to the setting of $R_{y, s} = L G^v R_{x, s}^v$.

By showing that $\| x_{0, \epsilon}^{(s), \dagger} - x_{0}^{(s)} \| \leq R_x$ and $\| \yhat(\xbar_{s}) - y_0^{(s)} \| \leq R_y$, we then plug in $x = x_{0, \epsilon}^{(s), \dagger}$ and $y = \yhat(\xbar_{s})$ into (\ref{eq1:convergence_per_stage}) in Lemma~\ref{lemma:convergence_rspd_per_stage} as follows
\begin{align}\label{eq:convergence_induction}
&
P(\bar x_s) - P(x_{0, \epsilon}^{(s), \dagger})
\leq 
f(\xbar_{s}, \yhat(\xbar_{s})) - f(x_{0, \epsilon}^{(s), \dagger}, \ybar_s) \nonumber\\
\leq & 
\frac{ R_{x, s}^2  }{ \eta_{x,s}  T } 
+ \frac{ R_{y, s}^2 }{ \eta_{y,s}  T  } 
+ \frac{5  \eta_{x,s} M^{2}  }{ 2 } 
+ \frac{5  \eta_{y, s} B^{2}}{ 2 }
% \\ &
+ \frac{ 4 ( M R_{x, s} + B R_{y, s} ) \sqrt{2 \log\frac{1}{\tildedelta}} }{\sqrt{T }}  
\nonumber\\
= &
\underbrace{ \frac{ R_{x, s}^2  }{ \eta_{x,s}  T } }_{(a)}
+ \underbrace{ \frac{ L^2 G^{2v} R_{x, s}^{2v} }{ \eta_{y,s}  T  } }_{(b)}
+ \underbrace{ \frac{ 5 \eta_{x, s} M^{2} }{ 2 } }_{(c)}
+ \underbrace{ \frac{ 5 \eta_{y, s} B^{2} }{ 2 } }_{(d)}
+ \underbrace{ \frac{ 4 M R_{x, s} \sqrt{2 \log\frac{1}{\tildedelta}} }{ \sqrt{ T } } }_{(e)}
+ \underbrace{ \frac{ 4 B L G^v R_{x, s}^v \sqrt{2 \log\frac{1}{\tildedelta}} }{ \sqrt{ T } } }_{(f)}   .
\end{align}
Finally, we would like to show $P(\bar x_s) - P(x_{0, \epsilon}^{(s), \dagger}) \leq \epsilon_{s} = \frac{\epsilon_{s-1}}{2}$ by properly setting the values of $T$, $\eta_{x, s}$, $\eta_{y, s}$, $R_{x, s}$ and $R_{y, s}$.

First, to make $\frac{5 \eta_{x, s} M^{2} }{ 2 } = \frac{\epsilon_{s-1}}{16}$ in term $(c)$ and $\frac{5 \eta_{y, s} B^{2} }{ 2 } = \frac{\epsilon_{s-1}}{16}$ in term $(d)$, we have $\eta_{x, s} = \frac{\epsilon_{s-1}}{40M^2}$ and $\eta_{y, s} = \frac{\epsilon_{s-1}}{40B^2}$, respectively.
Recalling that $\epsilon_{s} = \frac{\epsilon_{s-1}}{2}$, this requires $\eta_{x, s+1} = \frac{ \eta_{x, s} }{2}$ and $\eta_{y, s+1} = \frac{ \eta_{y, s} }{2}$, as in Line 5 and 6 of Algorithm \ref{alg:restart_primal_dual_algorithm_sa}.
Next, we can plug $\eta_{x, s}$ and $\eta_{y, s}$ into term $(a)$ and $(b)$.
By setting $T \geq \max \{ \frac{ 320 M^2 R_{x, s}^2 }{ \epsilon_{s-1}^2 } , \frac{ 320 B^2 L^2 G^{2v} R_{x, s}^{2v} }{ \epsilon_{s-1}^2 } \}$, we have
\begin{align*}
\frac{ R_{x, s}^2  }{ \eta_{x,s}  T } \leq \frac{ \epsilon_{s-1} }{ 8 } , \text{~and~}
\frac{ L^2 G^{2v} R_{x, s}^{2v}  }{ \eta_{y,s}  T } \leq \frac{ \epsilon_{s-1} }{ 8 }    .
\end{align*}
Then, for $(e)$, by setting $T \geq \frac{ 8192 \log(\frac{1}{\tildedelta}) M^2 R_{x, s}^2 }{ \epsilon_{s-1}^2 }$, we have $\frac{ 4 M R_{x, s} \sqrt{2 \log\frac{1}{\tildedelta}} }{ \sqrt{ T } } \leq \frac{\epsilon_{s-1}}{16}$.
Last, for $(f)$, by setting $T \geq \frac{ 8192 \log(\frac{1}{\tildedelta}) B^2 L^2 G^{2v} R_{x, s}^{2v} }{ \epsilon_{s-1}^2 }$, we have $\frac{ 4 B L G^v R_{x, s}^v \sqrt{2 \log\frac{1}{\tildedelta}} }{ \sqrt{ T } } \leq \frac{\epsilon_{s-1}}{16}$.

Therefore, we have
\begin{align*}
P(\xbar_{0}^{s+1}) - P(x_{0, \epsilon}^{(s), \dagger}) = P(\xbar_{s}) - P(x_{0, \epsilon}^{(s), \dagger}) \leq \frac{\epsilon_{s-1}}{2} = \epsilon_s,
 \end{align*}
i.e.,
 \begin{align*}    
P(\xbar_{0}^{s+1}) - P^* \leq \epsilon_s + \epsilon,
  \end{align*}
By induction, after running $S = \lceil \log(\frac{\epsilon_{0}}{\epsilon}) \rceil$ stages, with probability $(1 - \tildedelta)^{S} \geq 1 - S \tildedelta$, we have
$$
P(\xbar_{S}) - P^* \leq \epsilon_{S} + \epsilon \leq 2 \epsilon,
$$
where we set $\tildedelta = \delta/S$.
Considering the requirements from (\ref{eq:convergence_induction}), for $T$, we have
\begin{align*}
T
\geq & 
\max \bigg\{ 
\frac{ 320 M^2 R_{x, s}^2 }{ \epsilon_{s-1}^2 } , 
\frac{ 320 B^2 L^2 G^{2v} R_{x, s}^{2v} }{ \epsilon_{s-1}^2 } , 
\frac{ 8192 \log(\frac{1}{\tildedelta}) M^2 R_{x, s}^2 }{ \epsilon_{s-1}^2 } ,
\frac{ 8192 \log(\frac{1}{\tildedelta}) B^2 L^2 G^{2v} R_{x, s}^{2v} }{ \epsilon_{s-1}^2 } \bigg\}   .
% \\
% = & 
% \max \bigg\{ M^2 R_{x, s}^2 , B^2 L^2 G^{2v} R_{x, s}^{2v} \bigg\} \cdot \bigg\{ 320, 8192 \log(\frac{1}{\delta}) \bigg\} \cdot \frac{ 1 }{ \epsilon_{s-1}^2 }
% \\
% = &
% \max \bigg\{ 
% \frac{320 B^{2} L^{2} G^{2v} R_{1}^{2v}}{\epsilon_{0}^{2v} \epsilon^{2-2v}}, 
% \frac{2048 ( M + B )^{2}  R_{1}^{2} \log ({1}/{\tildedelta})}{\epsilon_{0}^{2}}
% \bigg\} 
% \\
% \geq & 
% \max \bigg\{  %\frac{320 M^{2} c^{2}}{\epsilon^{2(1-\theta)}} ,
% \frac{320 B^{2} L^{2} G^{2v} c^{2v}}{\epsilon^{2(1-v\theta)}}, 
% \frac{2048 ( M + B )^{2}  c^{2} \log ({1}/{\tildedelta})}{\epsilon^{2(1-\theta)}}
% \bigg\}
\end{align*}
Recall that $v \in [0,1]$,
$R_{x, 1} \geq \frac{c \epsilon_{0}}{\epsilon^{1-\theta}}$, 
$R_{x, s} = \frac{ R_{x, 1} }{ 2^{s-1} }$ and 
$\epsilon_{ s-1 } = \frac{ \epsilon_0 }{ 2^{s-1} }$.
On one hand, we have 
$$
\frac{R_{x, s}^2}{ \epsilon_{s-1}^2 } 
= \frac{R_{x, 1}^2}{ \epsilon_0^2 }
% \geq \frac{c^2}{ \epsilon^{ 2(1-\theta) } }  
.
$$
On the other hand, for $s \leq \lfloor \log( \frac{ \epsilon_0 }{ \epsilon } ) \rfloor$, we have
$$
\frac{R_{x, s}^{2v} }{ \epsilon_{s-1}^2 } 
= 
\frac{ R_{x, 1}^{2v} }{ \epsilon_0^2 } \cdot ( 2^{s-1} )^{ 2(1-v) } 
\leq
\frac{ R_{x, 1}^{2v} }{ \epsilon_0^2 } \cdot ( 2^{ \log( \frac{ \epsilon_0 }{ \epsilon } ) } )^{ 2(1-v) } 
\\
=
\frac{ R_{x, 1}^{2v} }{ \epsilon_0^2 } \cdot ( \frac{ \epsilon_0 }{ \epsilon } )^{ 2(1-v) }
=
\frac{ R_{x, 1}^{2v} }{ \epsilon_0^{2v} \epsilon^{ 2 ( 1 - v ) } } 
.
$$
The above terms show that $T$ would not change as $s$ changes.
Provided $R_{x, 1} = O(\frac{c \epsilon_0}{ \epsilon^{ 1 - \theta } })$ and $R_{x, 1} \geq \frac{c \epsilon_0}{ \epsilon^{ 1 - \theta } }$, we have the total number of iterations is at most $
ST = 
O\bigg(\frac{\lceil \log(\frac{\epsilon_{0}}{\epsilon}) \rceil   \lceil \log ({S}/{\delta}) \rceil }{\epsilon^{2(1 - v\theta)}}\bigg) = \widetilde O \bigg(\frac{1}{\epsilon^{2(1 - v\theta)}}\bigg)$.

\end{proof}   % proof of thm 3

\subsection{Adaptive Variants of RSPD}

When setting the initial value of radius $R_{1}$ (as well as the value of $T$) in Algorithm~\ref{alg:restart_primal_dual_algorithm_sa}, one requires to know $c$,  $\theta$ and $\epsilon$ (setting $R_{1} \geq \frac{c \epsilon_{0}}{\epsilon^{1 - \theta}}$), which may not be feasible in practice. Below, we introduce an adaptive variant of Algorithm~\ref{alg:restart_primal_dual_algorithm_sa} to find the $\epsilon$-optimal solution without knowing either $c$ or $\theta$ and $\epsilon$ to initiate the algorithm under that $v=1$. The developments in this section are mostly direct extension of techniques introduced~\citep{ICMLASSG,yang2018rsg}. 

The idea of tackling unknown $c$ is similar to the grid search: starting from a guess of $c$ for setting $R_1, T$ to run RSPD and then restarting RSPD using a larger $c$ (increased by a constant factor) or equivalently a larger $R_1, T$. However,  in order to not waste the updates for using a smaller $c$ and also remove the dependence on $\epsilon$ for setting $R_1, T$, we equivalently increase $R_1$ and $T$ in a way that depends on $\theta$ such that a similar convergence rate can be still established. The details are presented in Algorithm~\ref{alg:restart_primal_dual_algorithm_sa_adaptive_c}. 
%The adaptive variant  is a restart of RSPD and we refer to it as ARSPD, whose detailed updates are presented in Algorithm~\ref{alg:restart_primal_dual_algorithm_sa_adaptive_c}. The key idea for tackling unknown $c$ or $\theta$ is similar to that proposed in \citep{yang2018rsg}. 
%First, let us assume that $\theta \in (0, 1)$.
%As we have seen, if $R_{1} \geq \frac{c \epsilon_{0}}{\epsilon^{1 - \theta}}$, Algorithm~\ref{alg:restart_primal_dual_algorithm_sa} guarantees a convergence rate of $\Otilde(\frac{1}{\epsilon^{2(1 - v \theta)}})$.
%The issue happens when $c$ is unknown and $R_{1}$ is under-estimated, i.e., $R_{1} < \frac{c \epsilon_{0}}{\epsilon^{1 - \theta}}$.
%To handle this challenge, similar to~\citep{ICMLASSG, yang2018rsg}, ARSPD starts with a relative small $T_1$ ($R_1$) with $K$ stages of RSPD. For each stage, it  geometrically increases $T_1$ ($R_1$) by a factor $2^{2(1-v\theta)}$ ($2^{1-\theta}$). 
The following theorem gives convergence result of Algorithm~\ref{alg:restart_primal_dual_algorithm_sa_adaptive_c}.
Its proof is in Appendix \ref{app:sec:proof:thm:convergence_rspd_adaptive}.

%
%
%%\subsubsection{Adaptive to Parameter $c$ }
%
%Frist, let us assume that $\theta \in (0, 1)$.
%As we have seen, if $R_{1} \geq \frac{c \epsilon_{0}}{\epsilon^{1 - \theta}}$, Algorithm~\ref{alg:restart_primal_dual_algorithm_sa} guarantees a convergence rate of $\Otilde(\frac{1}{\epsilon^{2(1 - v \theta)}})$.
%The issue happens when $c$ is unknown and $R_{1}$ is under-estimated, i.e., $R_{1} < \frac{c \epsilon_{0}}{\epsilon^{1 - \theta}}$.
%To handle this challenge, we introduce a variant of Algorithm~\ref{alg:restart_primal_dual_algorithm_sa}
%and summarize it in Algorithm~\ref{alg:restart_primal_dual_algorithm_sa_adaptive_c}.
%The main difference lies on the update of the radius $R_{s}$ and the number of iterations in each stage (Line~\ref{alg:line:update_radius_and_t}).
%The following Theorem gives convergence analysis of Algorithm~\ref{alg:restart_primal_dual_algorithm_sa_adaptive_c}.
%

\begin{thm}\label{theorem:convergence_rspd_adaptive_c}
Suppose that Assumption~\ref{assumption:general} holds with $v=1$, and there exists $\hatepsilon_1\in(\epsilon, \epsilon_0/2]$ such that the initial value $R_1^{(1)}$ satisfies $R_1^{(1)} = \frac{c\epsilon_{0}}{\hatepsilon_{1}^{1 - \theta}}$ and the error bound condition holds on $\mathcal S_{\hatepsilon_1}$ with $c>0, \theta\in(0,1)$.  For any $\delta \in (0, 1)$, $\epsilon\leq \epsilon_0/4$,
let $\hat\delta = \frac{\delta}{S(S+1)},  
S =  \lceil \log_2(\frac{\epsilon_0}{\epsilon})\rceil$, 
 $\kappa = 1$, and
$
T_1 
= 
\max \bigg\{ 320 M^2, 320 B^2 L^2 G^2, 8192 \log(\frac{1}{\tildedelta}) M^2, 8192 \log(\frac{1}{\tildedelta}) B^2 L^2 G^2 \bigg\}
\cdot \frac{(R_1^{(1)})^2}{\epsilon_0^2}   .
$
% $T_1= \max \bigg\{ \frac{320 B^{2} L^{2} G^{2} (R_{1}^{(1)})^{2}}{\epsilon_{0}^{2} }, 
%  \frac{2048 ( M + B )^{2}  (R_{1}^{(1)})^{2} \log ({1}/{\hat\delta})}{\epsilon_{0}^{2}}
% \bigg\}$.
After at most  $K = \lceil \log(\frac{\hat{\epsilon}_{1}}{\epsilon}) \rceil+1$ calls of RSPD,  Algorithm~\ref{alg:restart_primal_dual_algorithm_sa_adaptive_c} guarantees that 
%$\GP(\xbar_{S}) = P(\xbar_{S}) - P(\xstar) \leq 2 \epsilon$ 
$P(x^{(K)}) - P(\xstar) \leq 2 \epsilon$ 
with probability $1 - \delta$ with an iteration complextiy of  
$\widetilde O(\log(1/\delta)/\epsilon^{2(1-\theta)})$  .
\end{thm}
{\bf Remark:} The requirement on the local error bound condition of the above theorem seems slightly stronger than that holds on $\S_\epsilon$. However, for a convex function it has been shown that a local error bound condition implies an error bound condition on any compact set with the same $\theta$ but possibly different $c$~\citep{arxiv:1510.08234}. The above theorem and   Algorithm~\ref{alg:restart_primal_dual_algorithm_sa_adaptive_c} do not cover the case $\theta=1$. But this can be easily resolved by setting $R_1 = \hat c_1\epsilon_0$ according to an initial guess of $c$, and then increasing $\hat c_1$ or $R_1$ by two times and rerun RSPD. It is easy to see that after $\log(c/\hat c_1)$ times the estimated value of $c$ will become larger than the true $c$ and the convergence theory in previous subsection will apply. As a result the total iteration complexity is only amplified by a factor of $\log(c/\hat c_1)$.

\begin{algorithm}[t]
\caption{Adaptive RSPD (ARSPD)}
\label{alg:restart_primal_dual_algorithm_sa_adaptive_c}
\begin{algorithmic}[1]
\STATE Initialization: $x^{(0)} \in X$, $S$, $T_1$, $\epsilon_{0}$, $R_{1}^{(1)}$ and $\kappa\in(0,1]$.
\STATE Set: $\epsilon^{(1)}_0 = \epsilon_0$, $\eta_{x,1} = \frac{\epsilon_{0}}{40  M^{2}  }$, $\eta_{y,1} = \frac{\epsilon_{0}}{40  B^{2} }$
\FOR{$k = 1, 2, ..., K$}
\STATE $x^{(k)} =$ RSPD($x^{(k-1)}, S, T_k, R^{(k)}_{1}, \epsilon^{(k)}_0$)
~\label{alg3:line:pd_sa_update_x}

\STATE $R_1^{(k+1)} = R_1^{(k)} \cdot 2^{1- \theta}$, $T_{k+1} = T_{k} \cdot 2^{2(1 - \theta)}$ and $\epsilon_0^{(k+1)} = \epsilon_0^{(k)} \cdot \kappa$. \label{alg:line:update_radius_and_t}
\ENDFOR
\STATE Return $x^{(K)}$
\end{algorithmic}
\end{algorithm}
Finally, we can show that even if $\theta$ is unknown, by setting $\theta=0$ in  Algorithm~\ref{alg:restart_primal_dual_algorithm_sa_adaptive_c}, we can still prove an improved convergence. Let $B_{\epsilon} = \max_{v \in \calL_{\epsilon}} \min_{z \in X^*} || v - z ||$  be the maximum distance between the points in the $\epsilon$-level set $\calL_{\epsilon}$ and the optimal set $X^*$.  
Proof of the following theorem is similar to the one of Theorem \ref{theorem:convergence_rspd_adaptive_c} (in Appendix \ref{app:sec:proof:thm:convergence_rspd_adaptive}) with slight modification.

%When $\theta$ is unknown, we simply treat $\theta = 0$ and $c = B_{\epsilon} = \max_{v \in \calL_{\epsilon}} \min_{z \in X^*} || v - z ||$, which represents the maximum distance between the points in the $\epsilon$-level set $\calL_{\epsilon}$ and the optimal set $X^*$. In this way, we initialize a value for the radius by $R_{1} = \frac{\epsilon_{0} B_{\hatepsilon_{1}} }{ \hatepsilon_{1} } $ for $\hatepsilon_{1} \in [ \epsilon, \frac{\epsilon_{0}}{2} ] $, which does not require $\theta$ to determine $R_{1}$ explicitly.
%Then as Algorithm~\ref{alg:restart_primal_dual_algorithm_sa_adaptive_c}, one can modify the update of the radius at the $s$-th stage by $R_{s} = R_{s-1} \cdot 2$ and the number of iterations by $T_{s} = T_{s-1}\cdot 4$. The following theorem provides the convergence analysis for the adaptive variant of Algorithm~\ref{alg:restart_primal_dual_algorithm_sa_adaptive_c} when $\theta$ is unknown.

\begin{thm}\label{theorem:convergence_rspd_adaptive_theta}
Suppose that Assumption~\ref{assumption:general} (1$\sim$4) holds with $v=1$, and $R^{(1)}_{1}$  is sufficiently large such  that there exists $\hatepsilon_{1} \in [ \epsilon, \frac{\epsilon_{0}}{2}]$ and  $R^{(1)}_{1}= \frac{ B_{\hatepsilon_{1}} \epsilon_{0}}{\hatepsilon_{1} }$. 
Given $\delta \in (0, 1)$, let $\theta=0$, $\hat\delta = \frac{\delta}{S(S+1)}$, 
$S =  \lceil \log_2(\frac{\epsilon_0}{\epsilon})\rceil$, %$K = \lceil \log(\frac{\hat{\epsilon}_{1}}{\epsilon}) \rceil+1$, 
$
T_1 
= 
\max \bigg\{ 320 M^2, 320 B^2 L^2 G^2, 8192 \log(\frac{1}{\tildedelta}) M^2, 8192 \log(\frac{1}{\tildedelta}) B^2 L^2 G^2 \bigg\}
\cdot \frac{(R_1^{(1)})^2}{\epsilon_0^2}   ,
$
% $T_1= \max \bigg\{ \frac{320 B^{2} L^{2} G^{2} (R_{1}^{(1)})^{2}}{\epsilon_{0}^{2} }, 
%  \frac{2048 ( M + B )^{2}  (R_{1}^{(1)})^{2} \log ({1}/{\hat\delta})}{\epsilon_{0}^{2}}
% \bigg\}$. 
and
$\kappa = 1$.
After at most  $K = \lceil \log(\frac{\hat{\epsilon}_{1}}{\epsilon}) \rceil+1$ calls of RSPD, 
Algorithm~\ref{alg:restart_primal_dual_algorithm_sa_adaptive_c} guarantees that 
%$\GP(\xbar_{S}) = P(\xbar_{S}) - P(\xstar) \leq 2 \epsilon$ 
$P(x^{(K)}) - P(\xstar) \leq 2 \epsilon$ 
with probability $1 - \delta$ with an iteration complexity of 
$\widetilde O(  \log(\frac{1}{\delta}) B_{\hatepsilon_{1}}^{2} /\epsilon^{2}  )$.
\end{thm}
{\bf Remark:} This iteration complexity is still an improved one compared with that in~\citep{Nemirovski:2009:RSA:1654243.1654247}, reducing the dependence on the size of $X$ and $Y$ to the $B_{\hatepsilon_1}$.

\section{Applications and Experiments}\label{sec:app}

\begin{table}[t]
\centering
\caption{Data statistics.}
\label{tab:data_stats}
\begin{tabular}{c|c|c}
\hline
Datasets     &    \#Examples    &    \#Features   \\\hline
w8a          &     49,749       &        300      \\
rcv1         &     20,242       &      47,236      \\
a9a          &     32,561       &        123      \\
real-sim     &     72,309       &    20,958       \\
covtype      &     581,012      &        54       \\
URL          &    2,396,130     &    3,231,961    \\
\hline
\end{tabular}
\end{table}

In this section, we investigate the effectiveness of our algorithms on two applications, i.e., distributionally robust optimization (DRO) and area under receiver operating characteristic curve (AUC) maximization.
We perform DRO experiments on four benchmark datasets, a9a, real-sim, rcv1 and w8a.
AUC experiments are performed on a9a, real-sim, covtype and URL.
Table~\ref{tab:data_stats} shows the statistics of the used six datasets.

\begin{figure*}[h]
\centering
\hspace{-0.25in}
{\includegraphics[scale=.215]{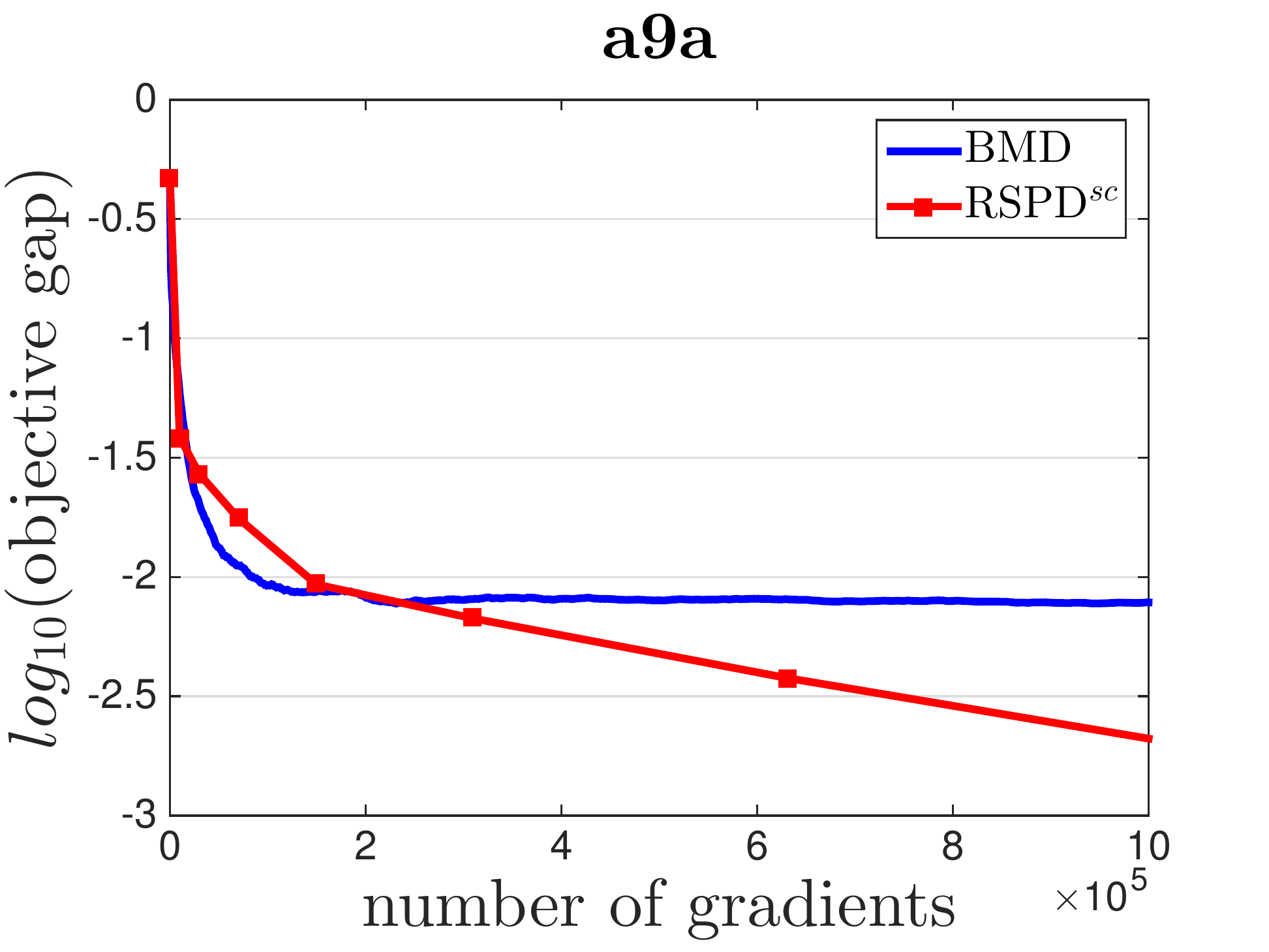}}
\hspace{-0.25in}
{\includegraphics[scale=.215]{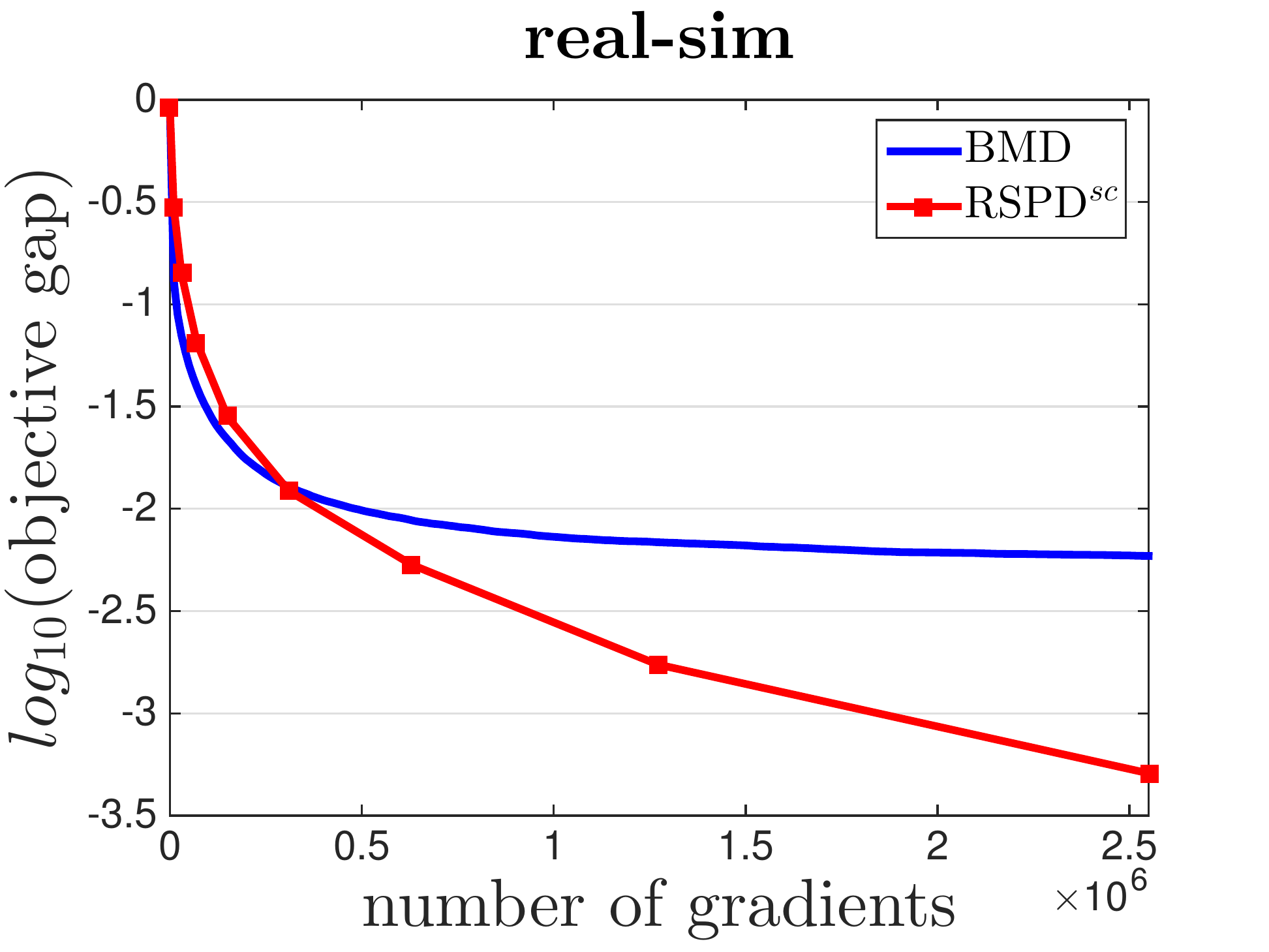}}
\hspace{-0.25in}
{\includegraphics[scale=.215]{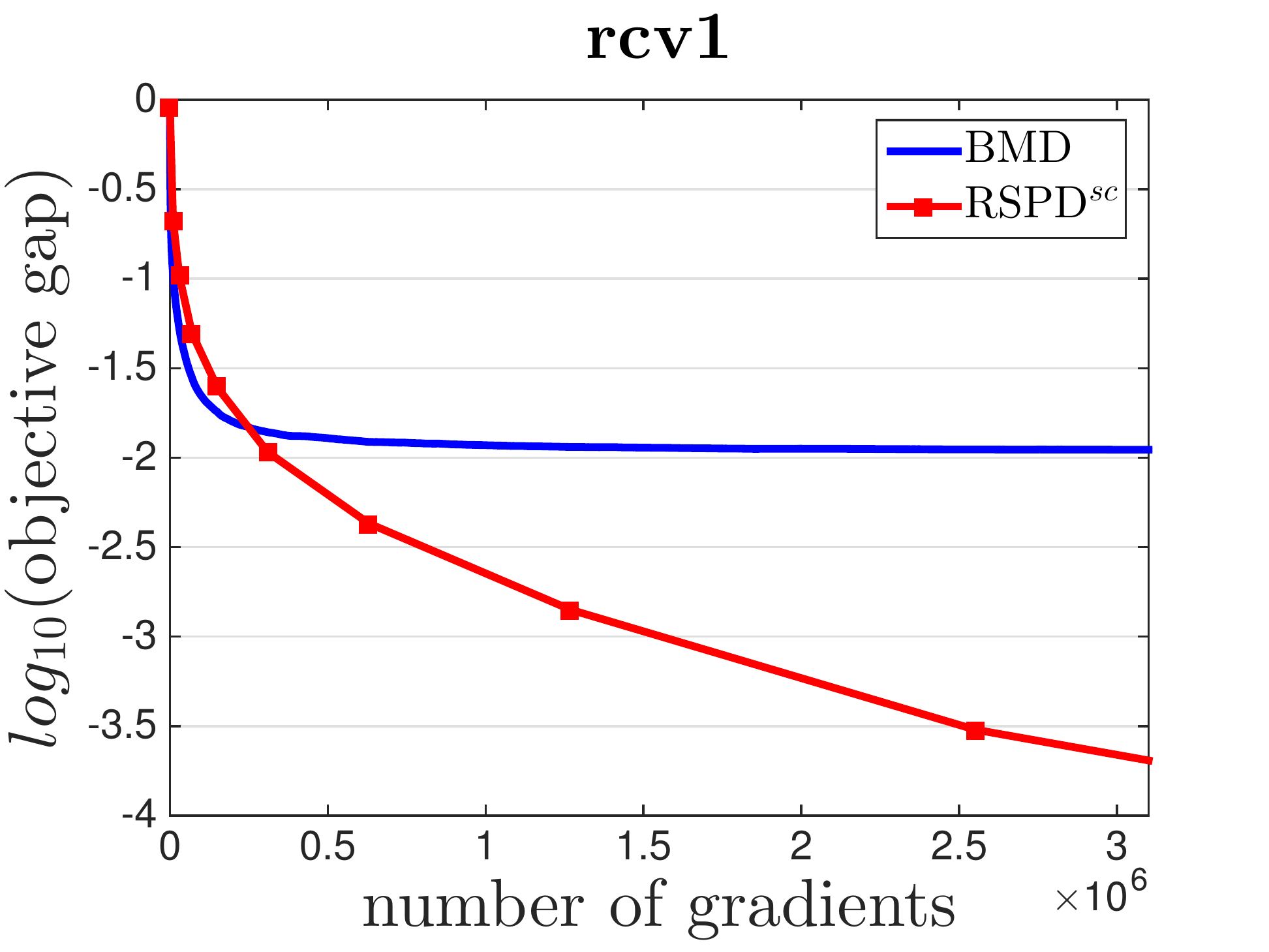}}
\hspace{-0.25in}
{\includegraphics[scale=.215]{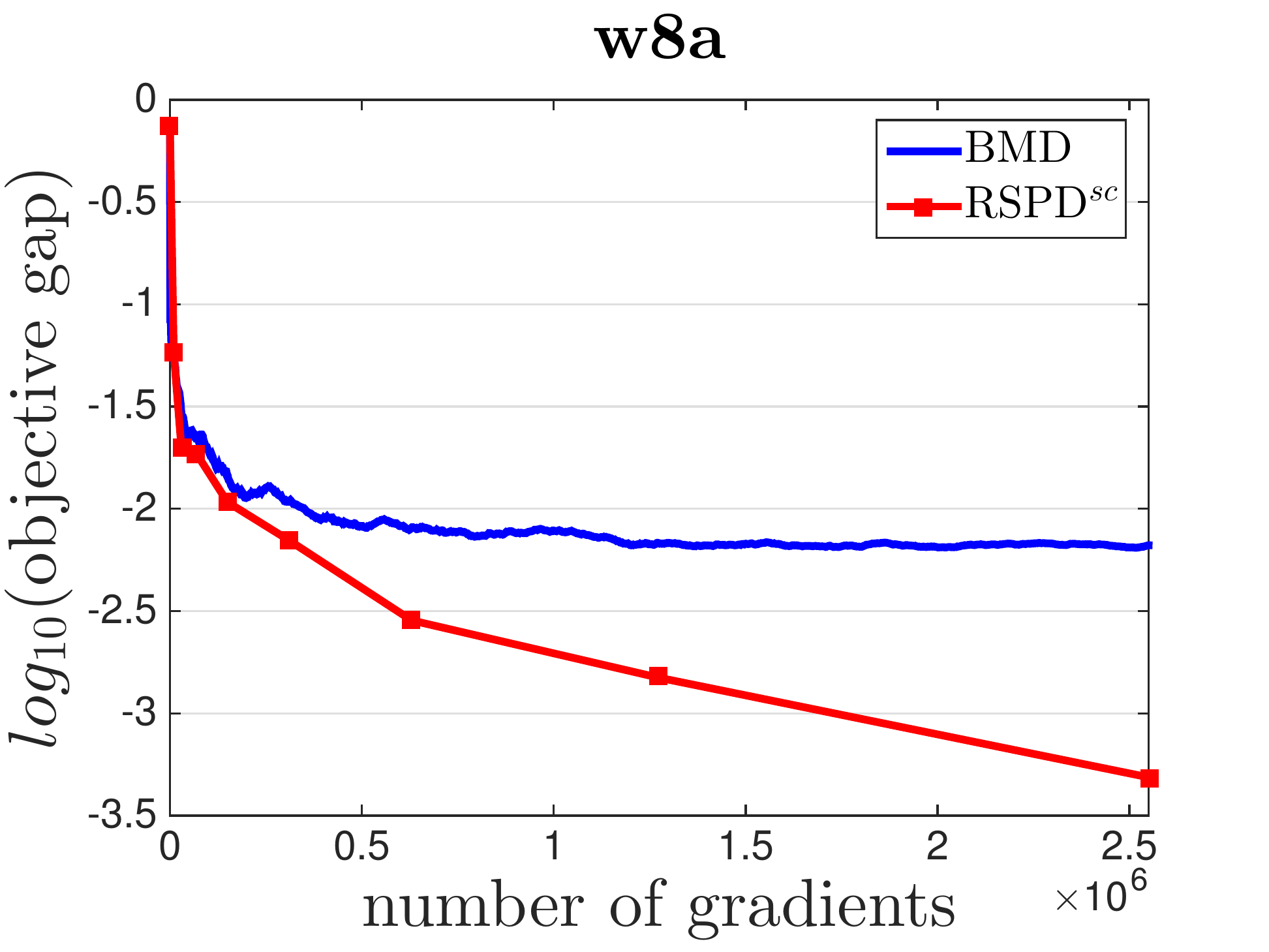}}
\hspace{-0.25in}
\caption{Results for Distributionally Robust Optimization}
\label{fig:dro_n_grad}
\end{figure*}

% Put cpu time result to supplement
\begin{figure*}[h]
\centering
\hspace{-0.25in}
{\includegraphics[scale=.215]{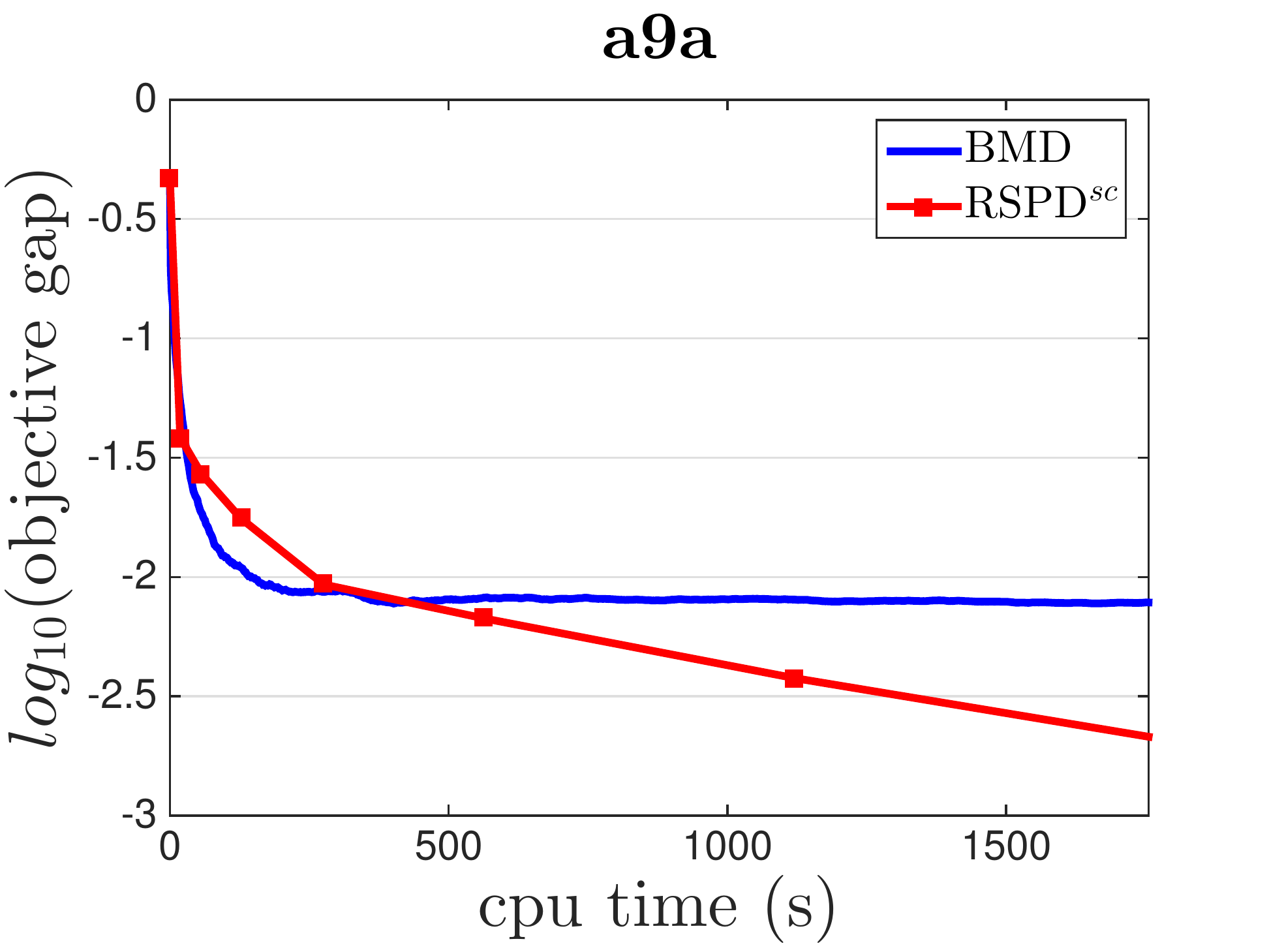}}
\hspace{-0.25in}
{\includegraphics[scale=.215]{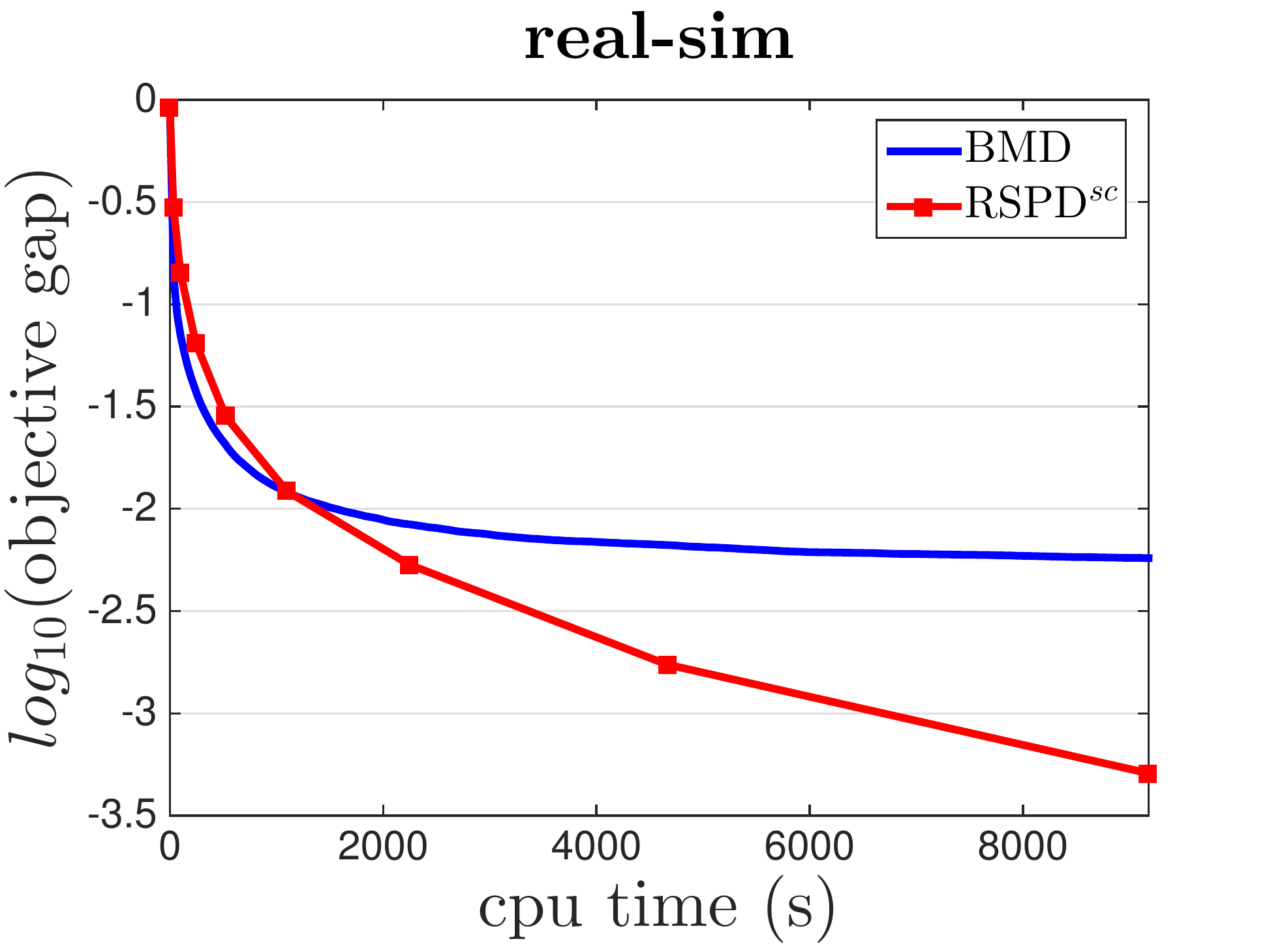}}
\hspace{-0.25in}
{\includegraphics[scale=.215]{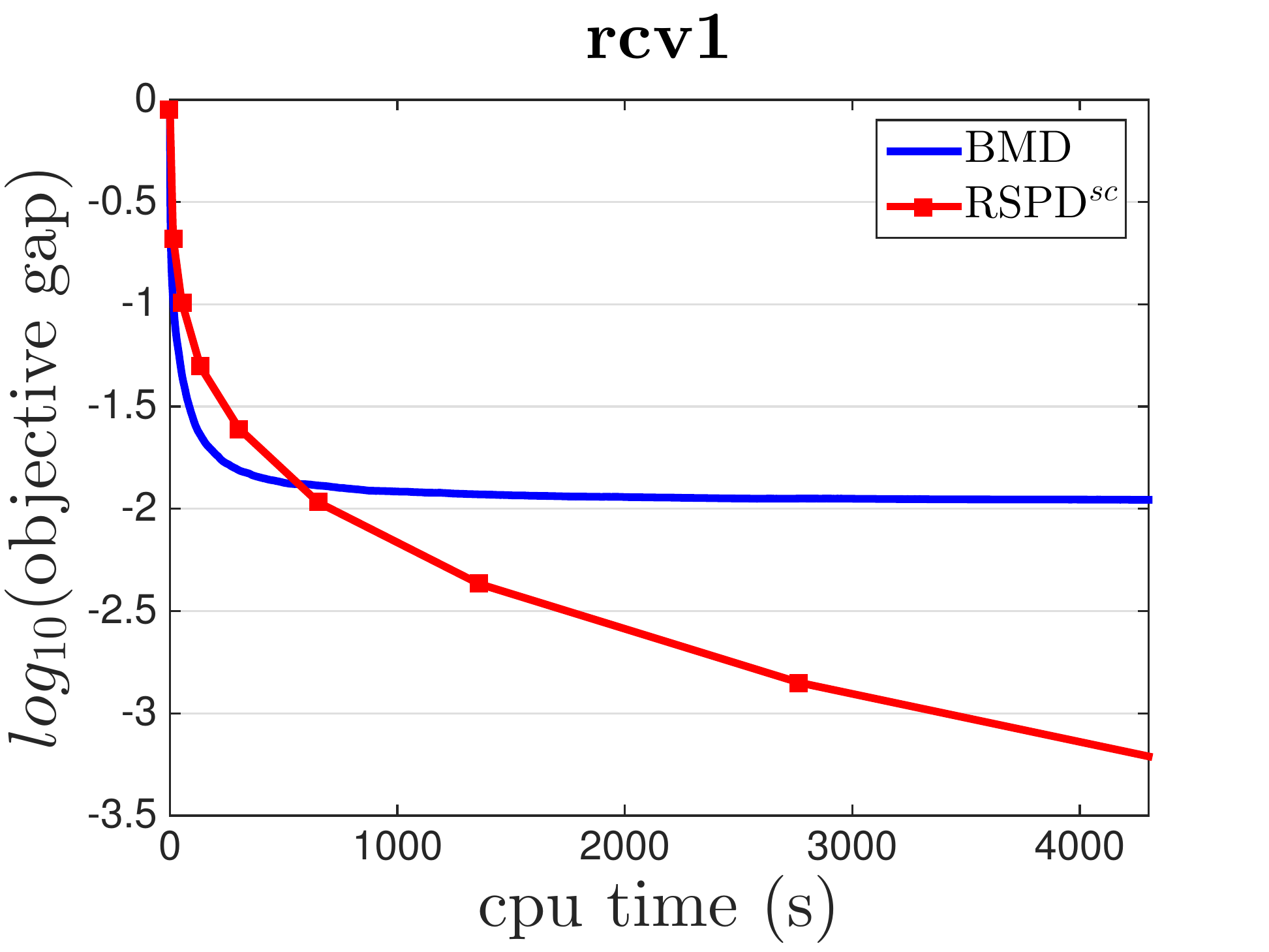}}
\hspace{-0.25in}
{\includegraphics[scale=.215]{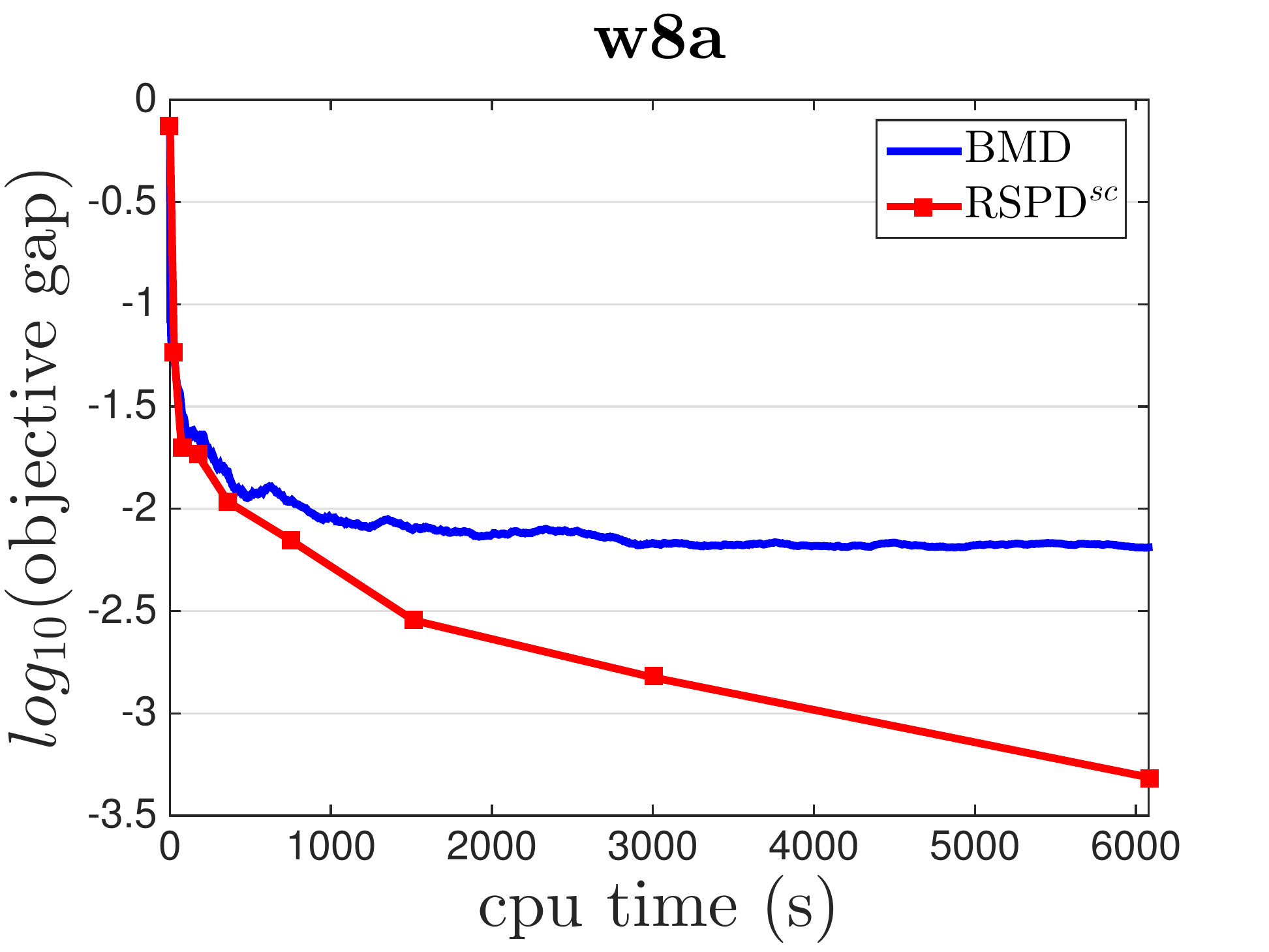}}
\hspace{-0.25in}
\caption{Results for Distributionally Robust Optimization by CPU time}
\label{fig:dro_cpu_time}
\end{figure*}

{\bf DRO.}  First, we consider solving the DRO~(\ref{eqn:dro}) for binary classification as mentioned in the Introduction. We use the square distance for $V$ that was studied in~\citep{DBLP:conf/nips/NamkoongD17}, i.e., $V(y, \mathbf 1/n) = \frac{\lambda_1}{2}  \| ny - \mathbf{1}\|_2^2$. 
%which can be written as 
%$\min_{\w\in\R^d} \max_{\p\in\Delta_n} \sum_{i=1}^{n}p_i \ell_i(\w) - \frac{\lambda_1}{2}  \| n\p - \mathbf{1}\|_2^2 + \frac{\lambda_2}{2}\|\w\|_2^2$, where $\Delta_n = \{ \p | \sum_{i=1}^{n}p_i = 1, p_i \geq 0, i=1,\dots, n\}$ is a probability simplex and
For the loss function,  we consider the non-smooth hinge loss $\ell_i(x)=\max\{0,1 -b_i x^\top a_i\}$, where $a_i\in\R^d$ denotes the feature vector and $b_i\in\{1, -1\}$ denotes the label. We also include a regularizer $g(x)$ on the model parameter $x$. Using different regularizers will give different properties for the primal objective function. For example, if $g(x) = \frac{\lambda_2}{2}\|x\|_2^2$, then the primal objective function $P(x)$ is obviously a strongly convex function. If $g(x) = \lambda_2\|x\|_1$, then we can prove that the primal objective function $P(x)$ is a piecewise quadratic convex function, which satisfies the LEB condition with $\theta=1/2$. 
The proof is given in Appendix~\ref{app:sec:proof:piecewise_quadratic}. 
%We report the result of RSPD$^{\text{sc}}$ for solving the problem with $g(x) = \frac{\lambda_2}{2}\|x\|_2^2$ here and include the results for the $\ell_1$ norm in the supplement. 
We report the result of RSPD$^{\text{sc}}$ for solving the problem with $g(x) = \frac{\lambda_2}{2}\|x\|_2^2$ here.

We compare with the baseline called Bandit Mirror Descent (BMD) algorithm considered in~\citep{DBLP:conf/nips/NamkoongD16}, which has a convergence rate of $O(1/\sqrt{T})$. The stochastic gradients are computed in the same way as in~\citep{DBLP:conf/nips/NamkoongD16}. Computing the restarted dual solution $y^{(s+1)}_0= \mathcal A(x^{(s+1)}_0)$ takes $O(nd)$ time complexity, and each update for the primal variable and the dual variable takes $O(d)$ and $O(n)$, respectively. Therefore, the total time complexity of RSPD for finding an $\epsilon$-optimal solution is $O(nd\log(1/\epsilon) +\frac{n+d}{\epsilon})$. In contrast, the time complexity of BMD is $O((n+d)/\epsilon^2)$. 

We conduct experiments on four datasets from libsvm website using $\ell_2$ regularization for $g(x)$. The regularizer parameters are set to be $\lambda_1 = \lambda_2 = \frac{1}{n}$ for all datasets.   The initial step sizes of all algorithms are tuned in the range of $\{10^{-5:1:3}\}$.
All algorithms start with the same initial solutions with $y_0 = \frac{\textbf{1}}{n}$ and $x_0 = \mathbf{0}$. 
In implementing RSPD$^{\text{sc}}$, we start with an initial $T=10^4$ increased by a factor of $2$ at each epoch.  
The results of objective gap against the number of gradients and against CPU time are shown in Figure~\ref{fig:dro_n_grad} and Figure~\ref{fig:dro_cpu_time}, respectively.
%More results in terms of cpu time can be found in supplement. %Section~\ref{section:supp:more_result} of supplement.
It is clear that the proposed algorithm converge much faster than the baseline algorithm BMD. 

\yancomment{  % This part may not be very important?
Let $f(x) = \max_{p} \sum_{i=1}^{n} p^{T} \ell( x ) - \sigma(p)$.
By Proposition 2.3 in~\citep{rockafellar1987linear} and the assumption that $\ell(x)$ is piecewise linear, the above function is piecewise linear-quadratic in $x$, and thus obeys the LEB condition.
}

\begin{figure*}[t]
\centering
\hspace{-0.1in}
{\includegraphics[scale=.2125]{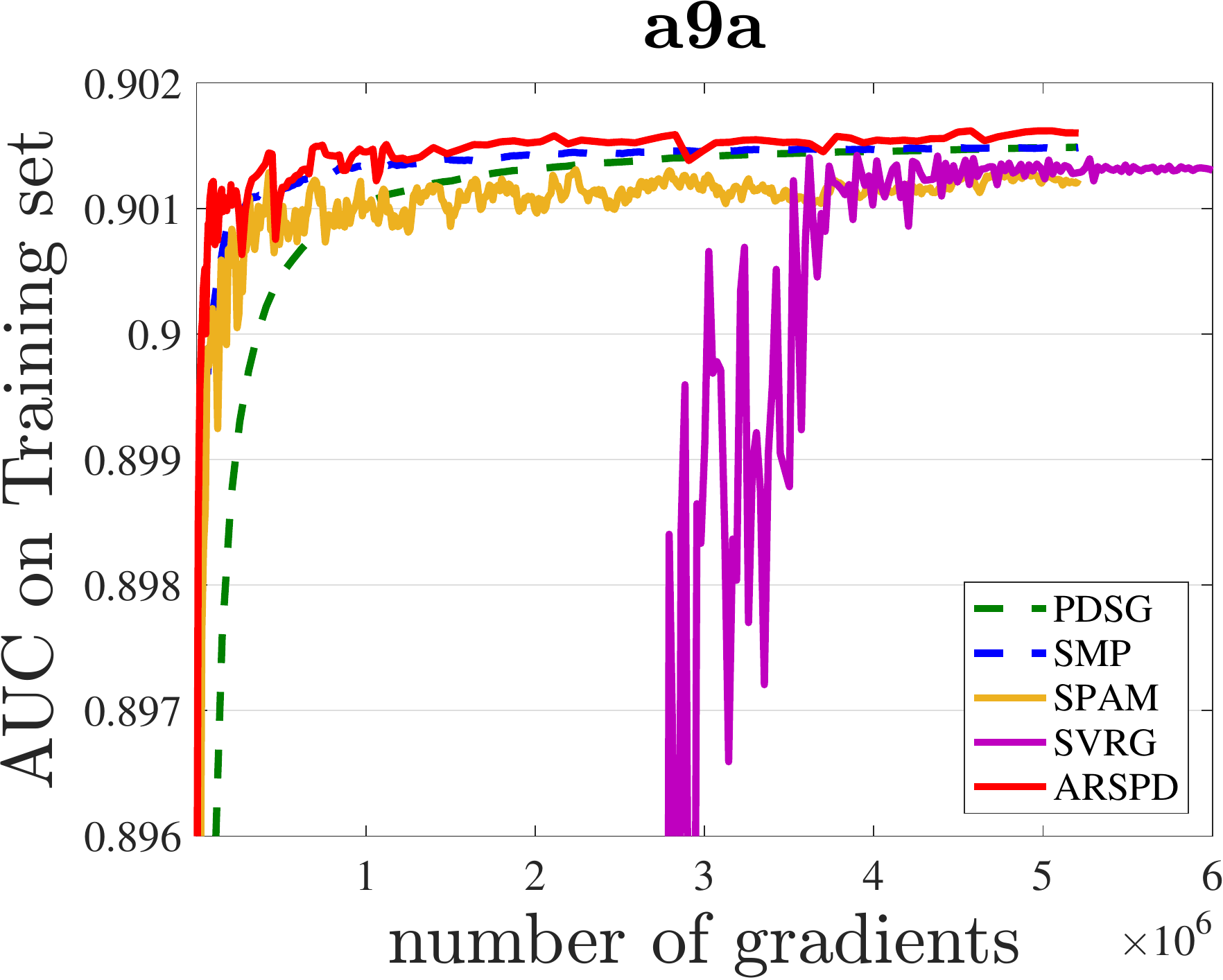}}
\hspace{-0.1in}
{\includegraphics[scale=.2125]{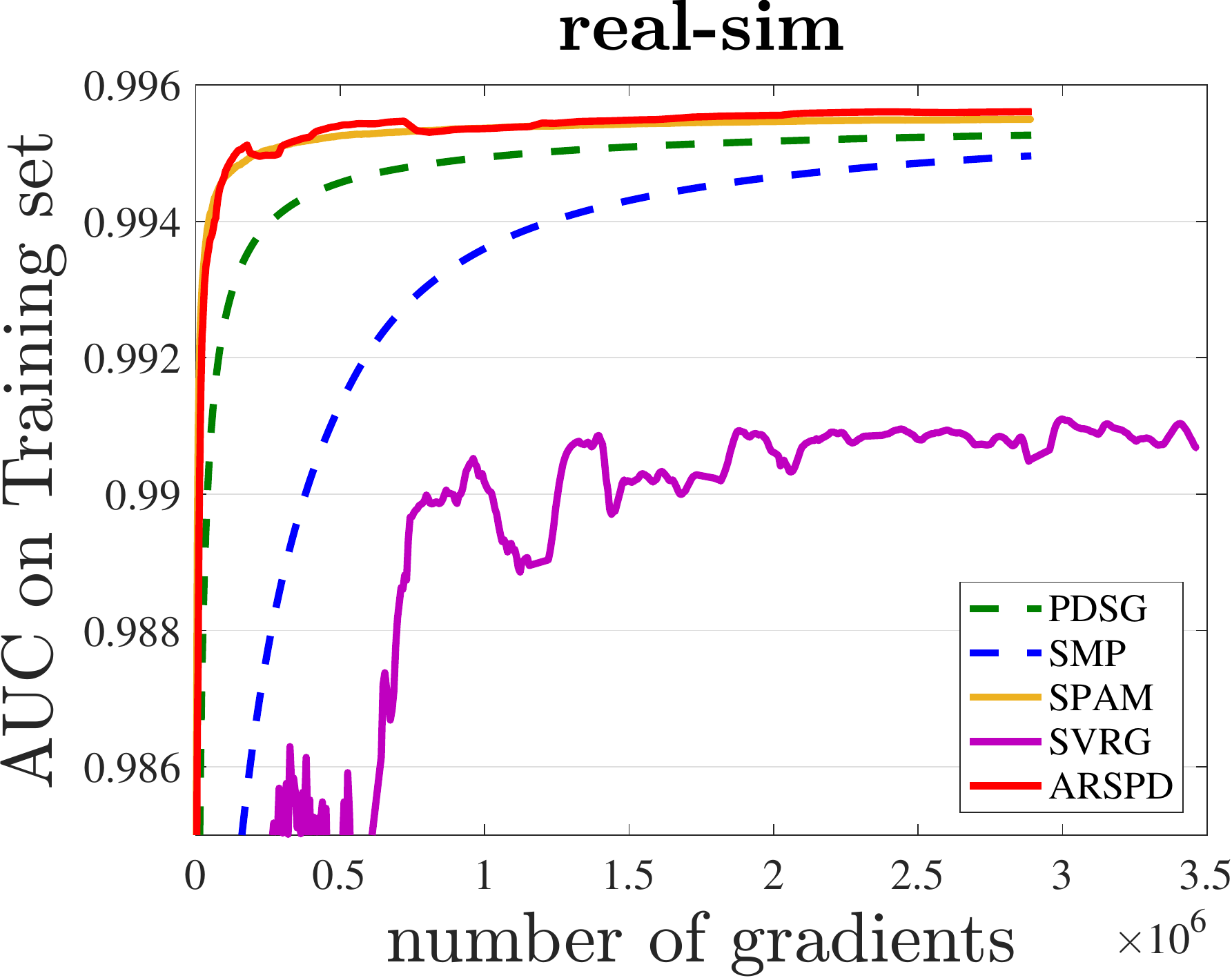}}
\hspace{-0.1in}
{\includegraphics[scale=.2125]{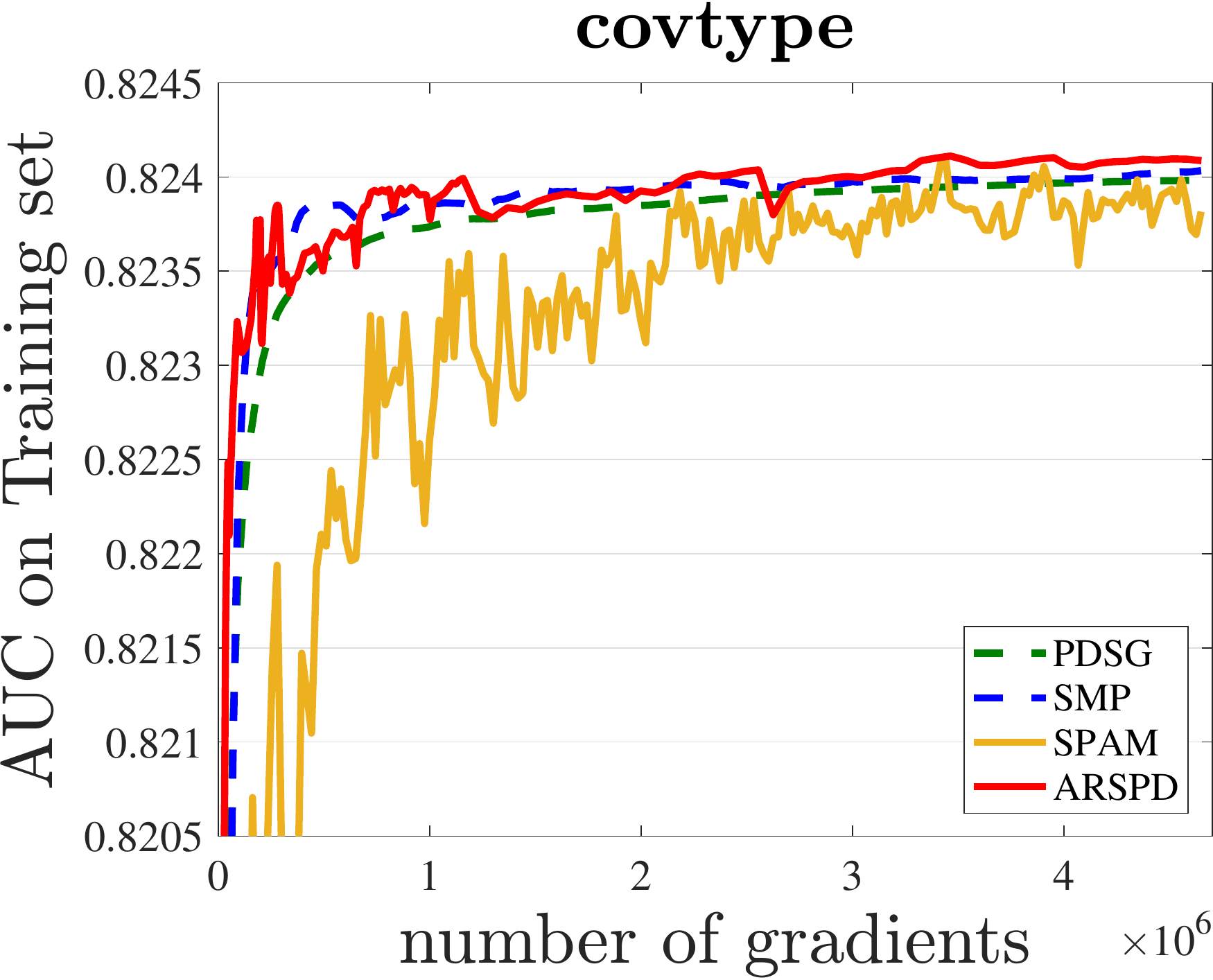}}
\hspace{-0.1in}
{\includegraphics[scale=.2125]{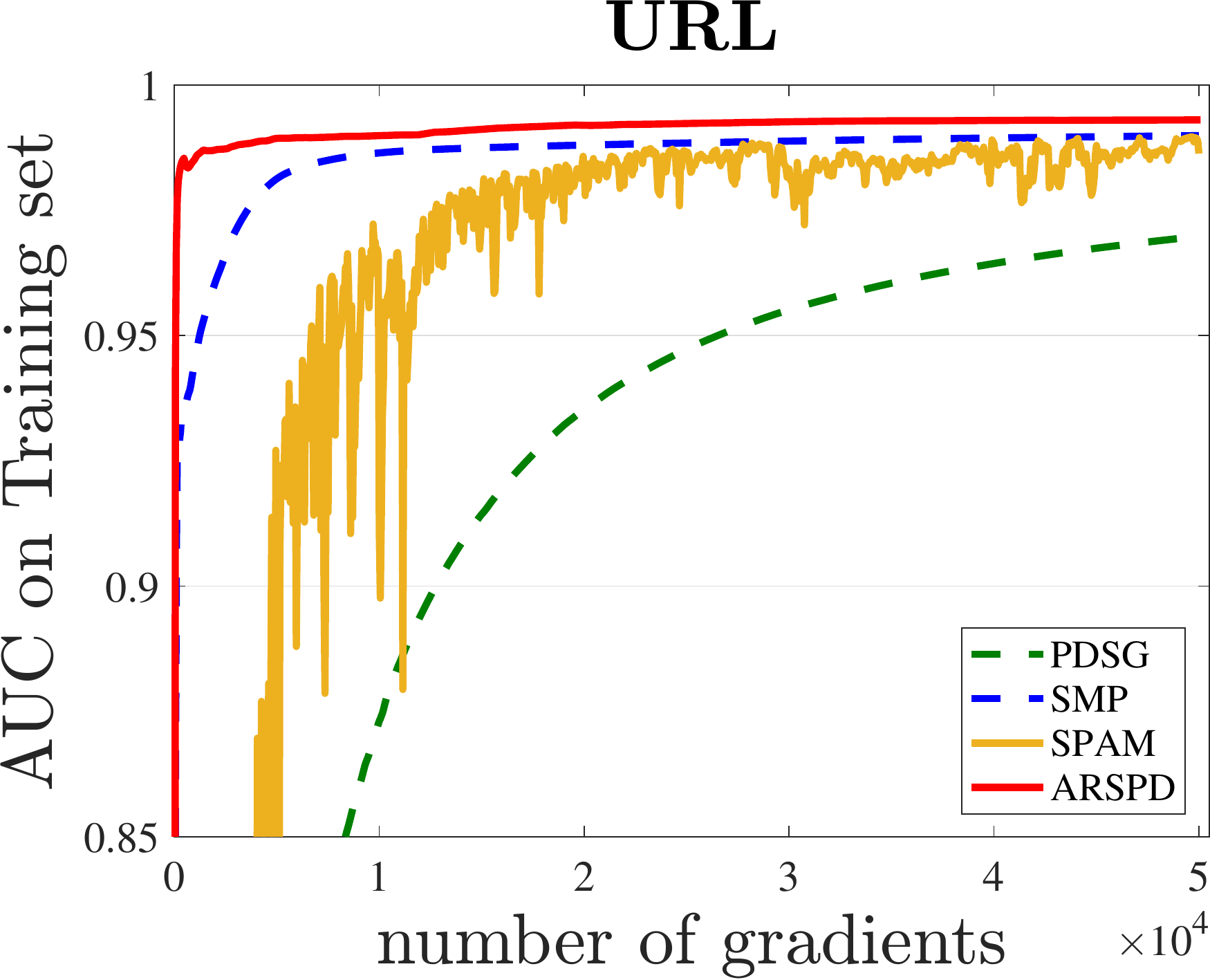}}
\hspace{-0.1in}
\caption{Results for AUC Maximization (with L2 ball constraint)}
\label{fig:auc_n_grad}
\end{figure*}

\begin{figure*}[t]
\centering
\hspace{-0.1in}
{\includegraphics[scale=.2125]{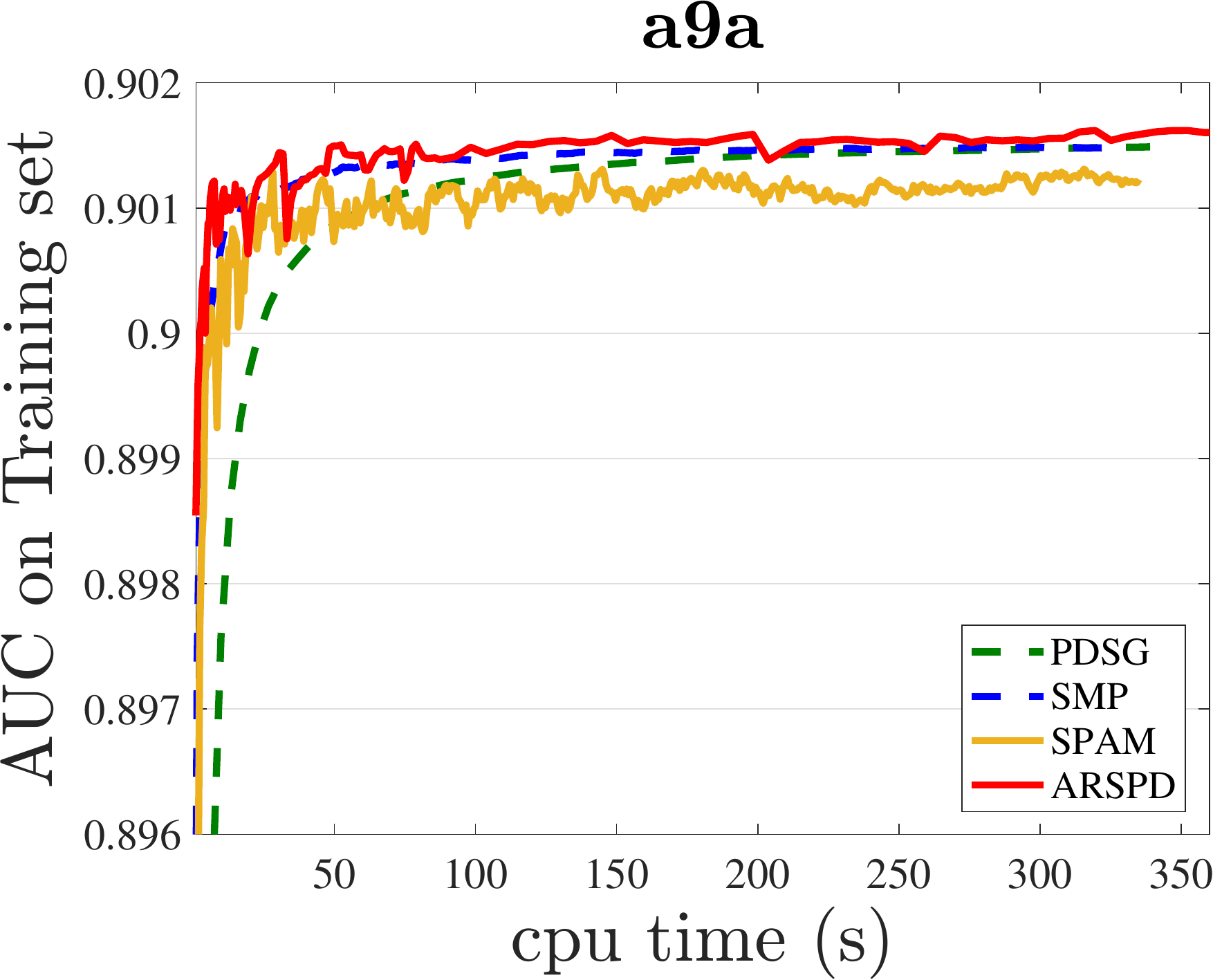}}
\hspace{-0.1in}
{\includegraphics[scale=.2125]{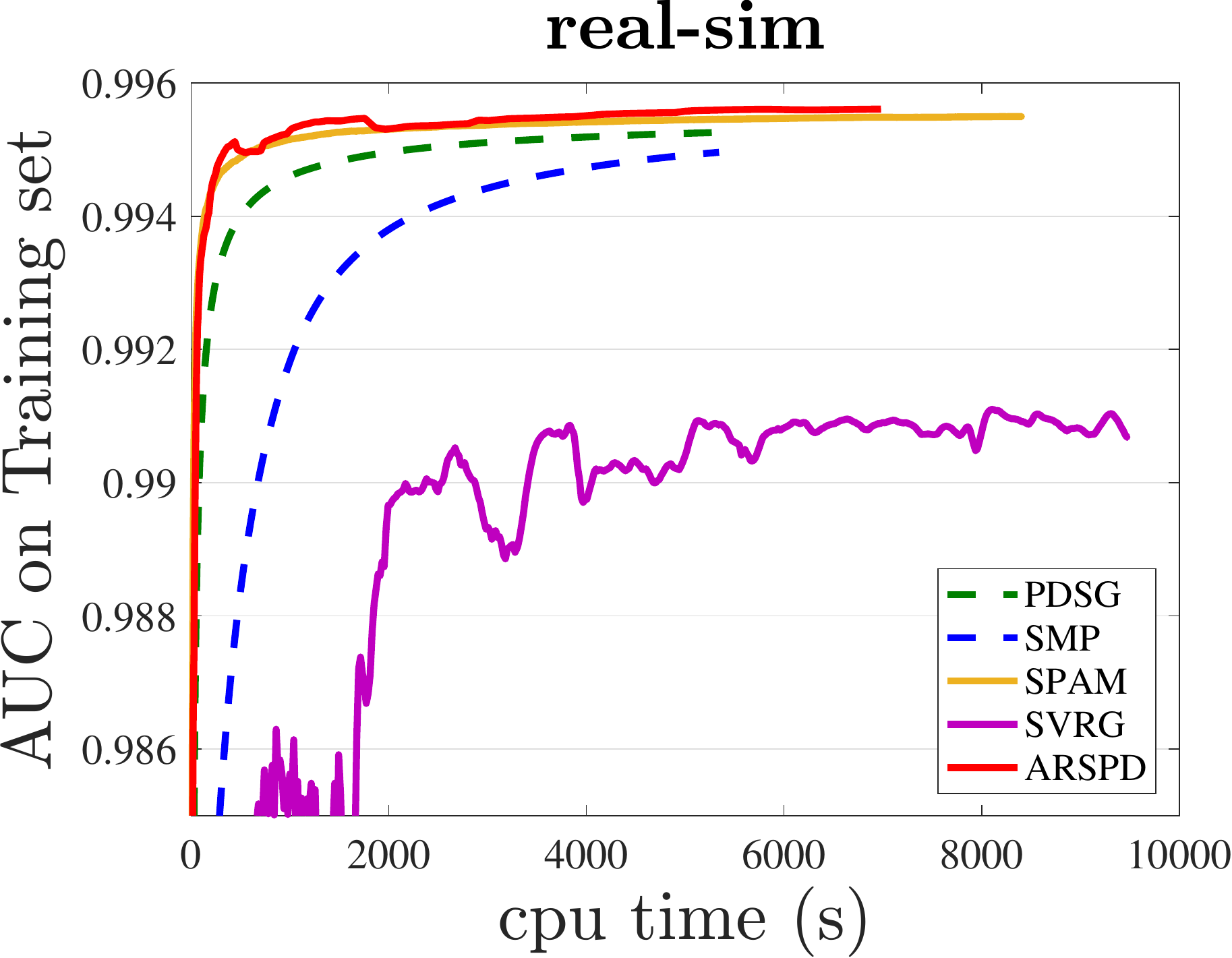}}
\hspace{-0.15in}
{\includegraphics[scale=.2125]{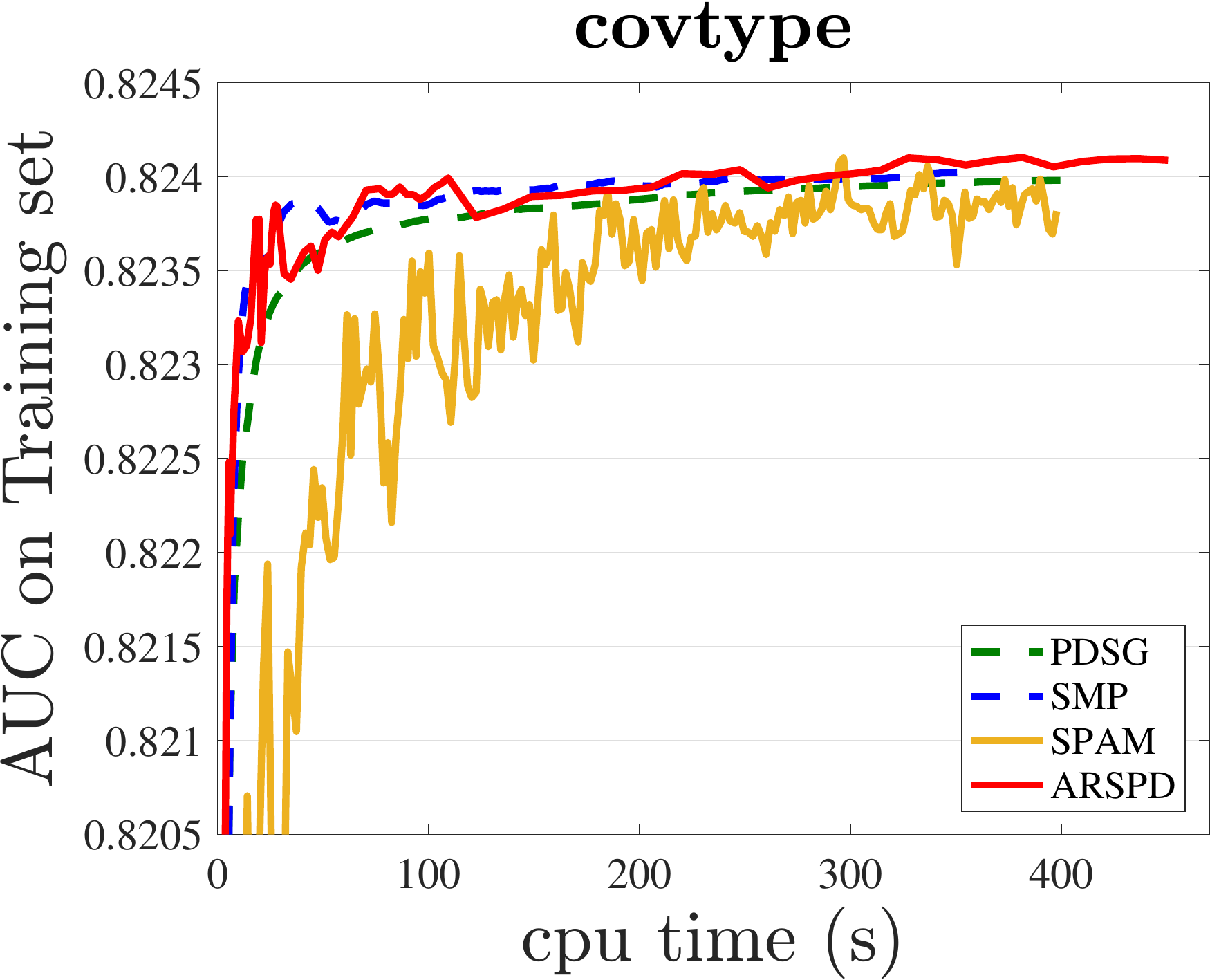}}
\hspace{-0.09in}
{\includegraphics[scale=.2125]{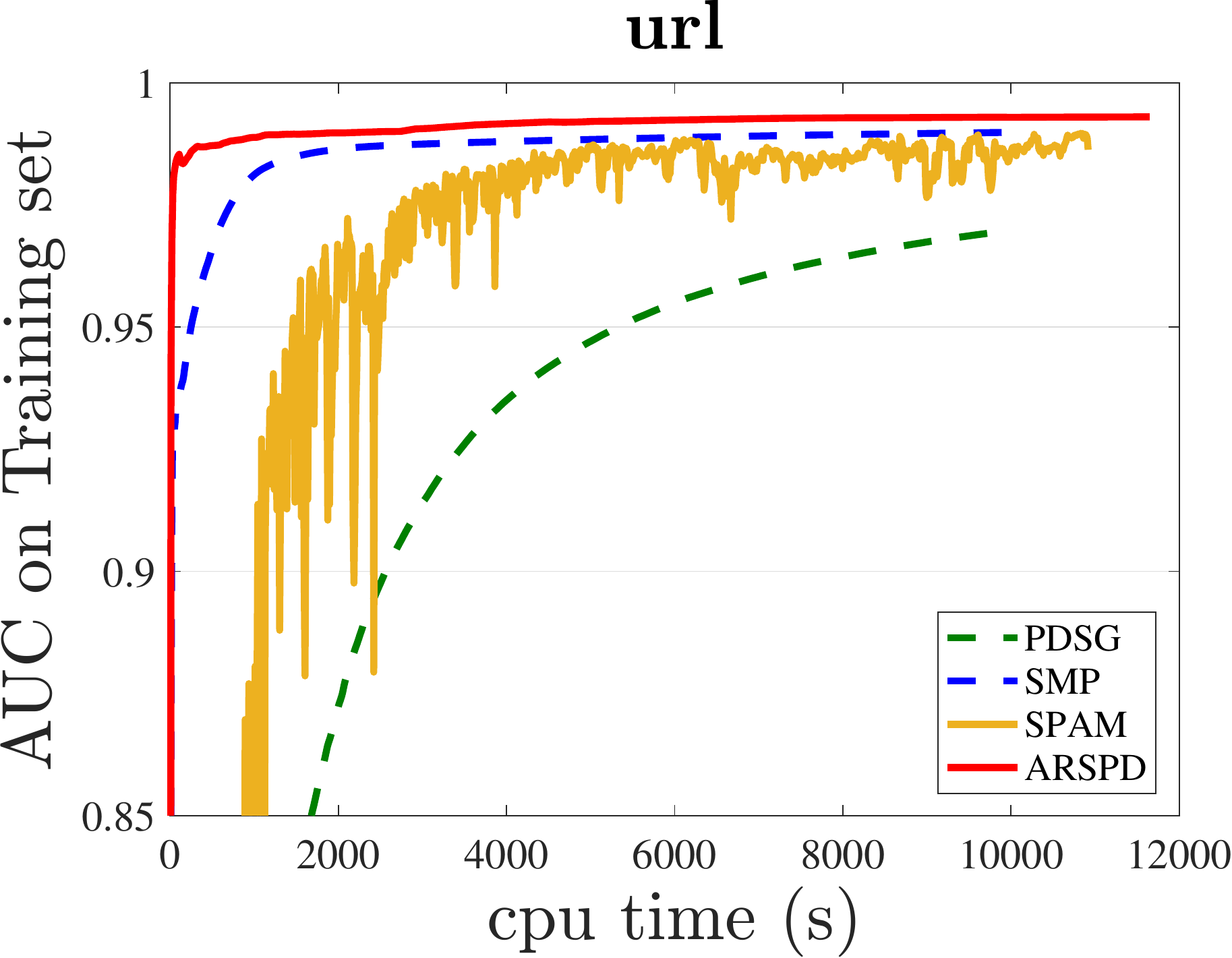}}
\hspace{-0.1in}
\caption{Results for AUC Maximization by CPU time (with L2 ball constraint)}
\label{figure:auc_cputime}
\end{figure*}

{\bf AUC Maximization.} Next, we consider empirical AUC maximization by solving  the min-max saddle-point formulation proposed by~\citep{DBLP:conf/nips/YingWL16}: 
%A surrogate convex loss to maximize 
%AUC 
%is 
%$
%\min_{\w} \sum_{i \in {\bf P}, j \in {\bf N}} (1 - \w (\x_i - \x_j) )^2
%$
%where ${\bf P} = \{ i | y_i = +1, i=1,...,n \}$ and ${\bf N} = \{ j | y_j = -1, i=1,...,n \}$ include the indices for the positive and negative instances.
%\citep{DBLP:conf/nips/YingWL16} proposes the following equivalent problem
\[
\min_{\w , a, b } \max_{\alpha} \frac{1}{n}\sum_{i=1}^{n} F(\w, a, b, \alpha; (\x_i, z_i) ), 
\]
where $\x_i\in\R^d, z_i\in\{1,-1\}$ denote the feature-label pairs of a training example, 
$
F(\w, a, b, \alpha; (\x, z) ) = (1 - p) (\w^{\top} \x - a)^{2} I_{[z=1]} + p(\w^{\top} \x - b)^{2} I_{[z=-1]} - p(1-p)\alpha^2 + 2(1+\alpha) (p \w^{\top}\x I_{[z=-1]} - (1-p) \w^{\top}\x I_{[z=1]})$, $p$ is the percentage of positive example, and
$I_{[\cdot]}$ is the indicator function.
%By treating the primal variable 
Let ${\bf{v}} = [\w^{\top}, a, b]^{\top} \in \R^{d+2}$.  In order to achieve good AUC performance, we add a ball constraint on $\w$. Bounds on $(a, b)$ can be derived similarly to~\citep{DBLP:conf/nips/YingWL16}. %Using different norm for defining the ball constraint may lead to different LEB condition for the primal objective $P(\v)$ in terms of $\v$. For example, 
If we use $\ell_1$ ball $\|\v\|_1\leq B$, it was shown in~\citep{fastAUC18} that the primal objective function satisfies the LEB with $\theta=1/2$. If we use $\ell_2$ ball constraint $\|\v\|_2\leq B$, under a mild condition that $\min_{\v\in\R^{d+2}}P(\v)<\min_{\|\v\|_2\leq B}P(\v)$ it was shown that a LEB with $\theta=1/2$ is satisfied~\citep{DBLP:conf/nips/LiuZZRY18}. Then the iteration complexity of RSPD is given by $\widetilde O(1/\epsilon)$. Since the dual variable is one-dimensional, computing the restarted dual solution $y^{(s+1)}_0$ takes $O(d)$ complexity given the averaged feature vectors for the positive and negative examples are precomputed. Hence, when LEB with $\theta=1/2$ is satisfied, the total time complexity of RSPD or ARSPD is $\widetilde O(d\log(1/\epsilon) + d/\epsilon)$. 
%If this condition does not hold, then a LEB condition holds for some $\theta\in(0, 1]$ because the objective function is sub-analytic~\citep{}. 
%We also note that SPDC~\citep{DBLP:conf/icml/ZhangL15} is applicable in the AUC task and its total time complexity is $O(nd^2 +d^3+d^2/\sqrt{\epsilon})$ since every iteration needs to solve a linear system (i.e., the proximal mapping of the quadratic part of the primal variable), where $n$ is sample size.
We also note that SPDC~\citep{DBLP:conf/icml/ZhangL15} is applicable in the AUC task,
but it does not give a linear rate for the considered AUC problem, because there is no strong convexity for primal variable as required for achieving a linear rate.
Adding a small strongly convex regularizer on the primal variable,
its total time complexity is $O(nd^2 + d^2/\sqrt{\epsilon})$ since every iteration needs to solve a linear system (i.e., the proximal mapping of the quadratic part of the primal variable), where $n$ is sample size.
Here, we report the results of the proposed adaptive algorithm for the problem with an $\ell_2$ ball constraint and an $\ell_{1}$ ball, respectively.  
%Here, we report the results of the proposed adaptive algorithm for the problem with an $\ell_2$ ball constraint.  
%The results  for using $\ell_1$  ball constraint  are included in the supplement. 

\begin{figure*}[t]
\centering
\hspace{-0.1in}
{\includegraphics[scale=.2125]{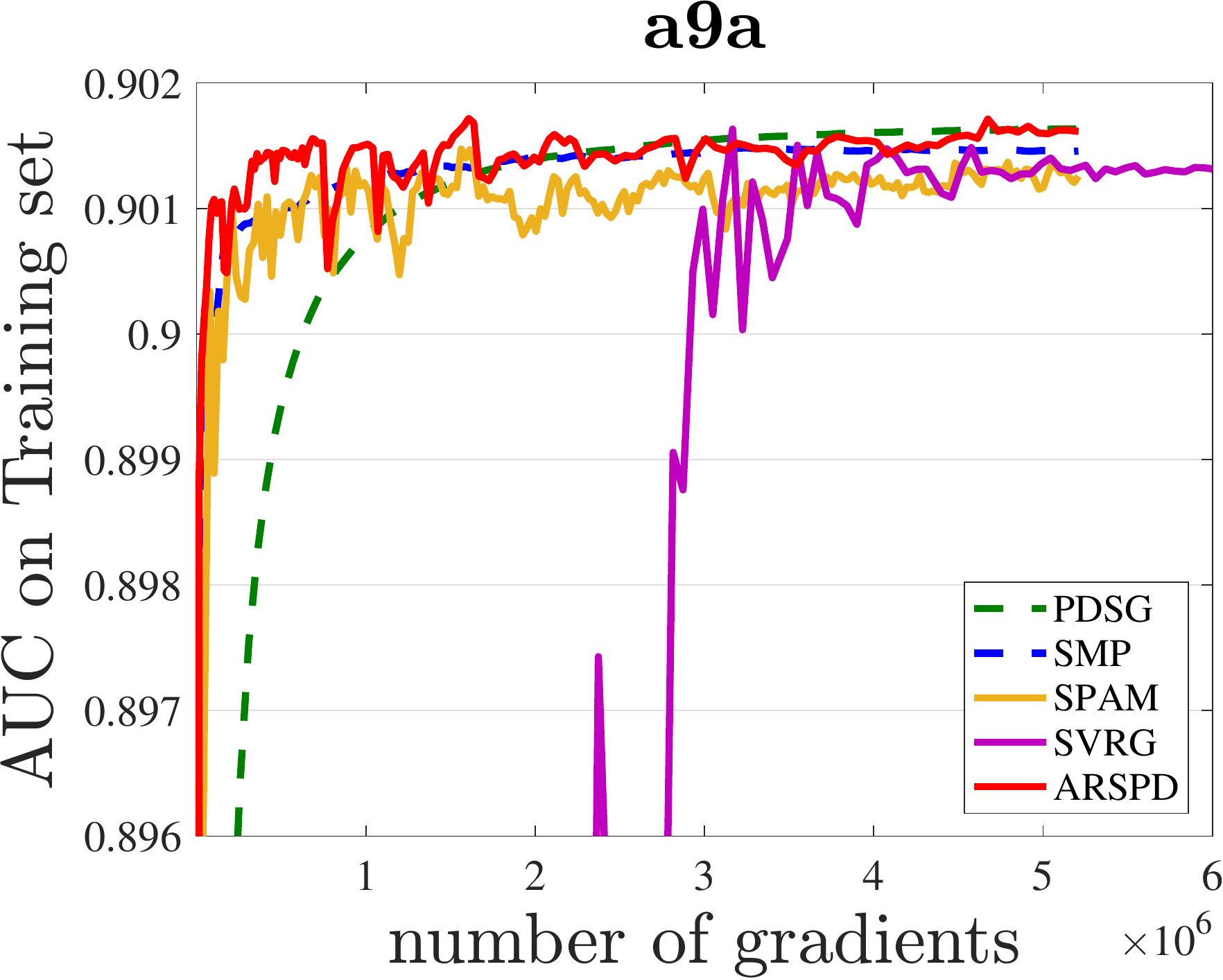}}
\hspace{-0.1in}
{\includegraphics[scale=.2125]{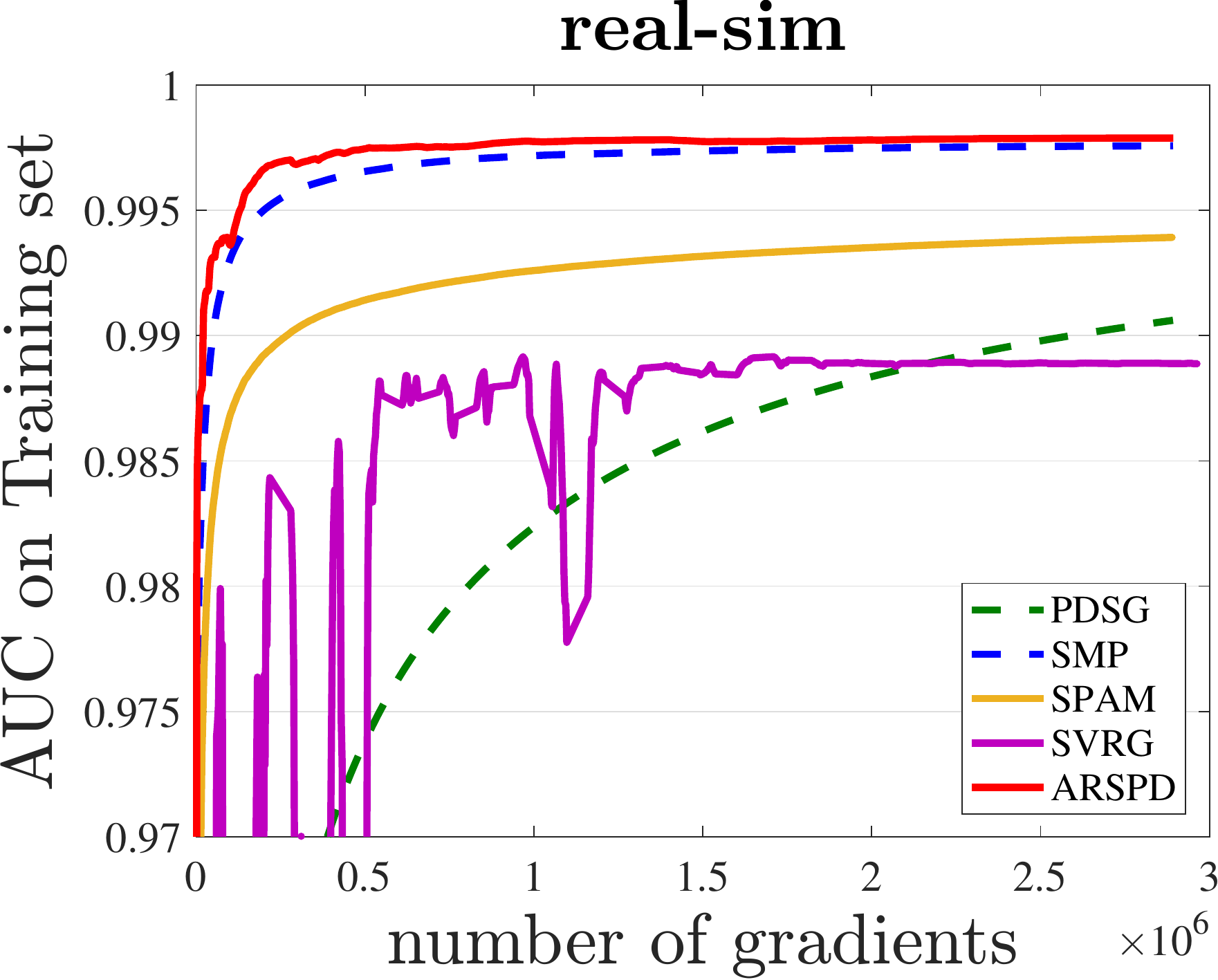}}
\hspace{-0.1in}
{\includegraphics[scale=.2125]{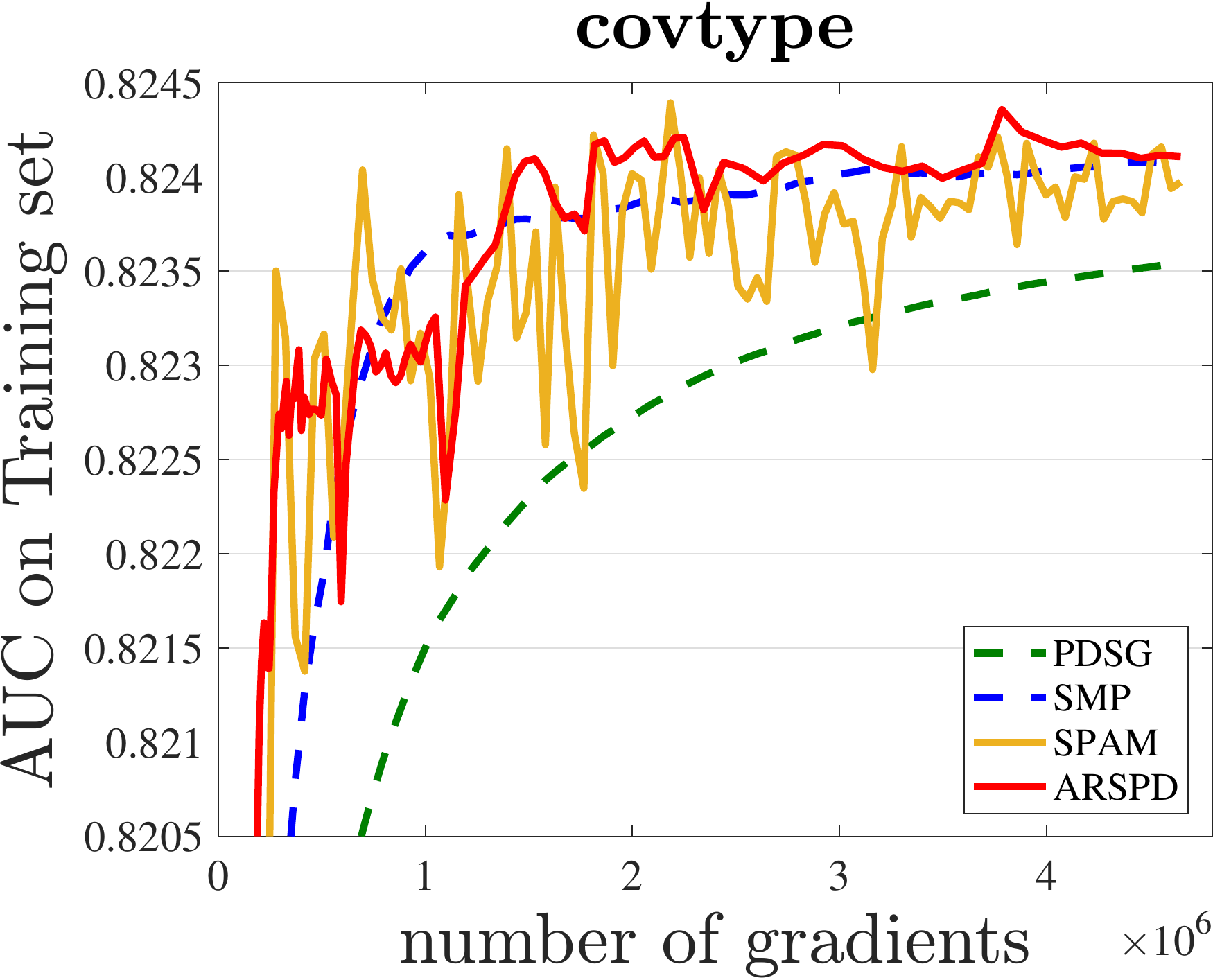}}
\hspace{-0.1in}
{\includegraphics[scale=.2125]{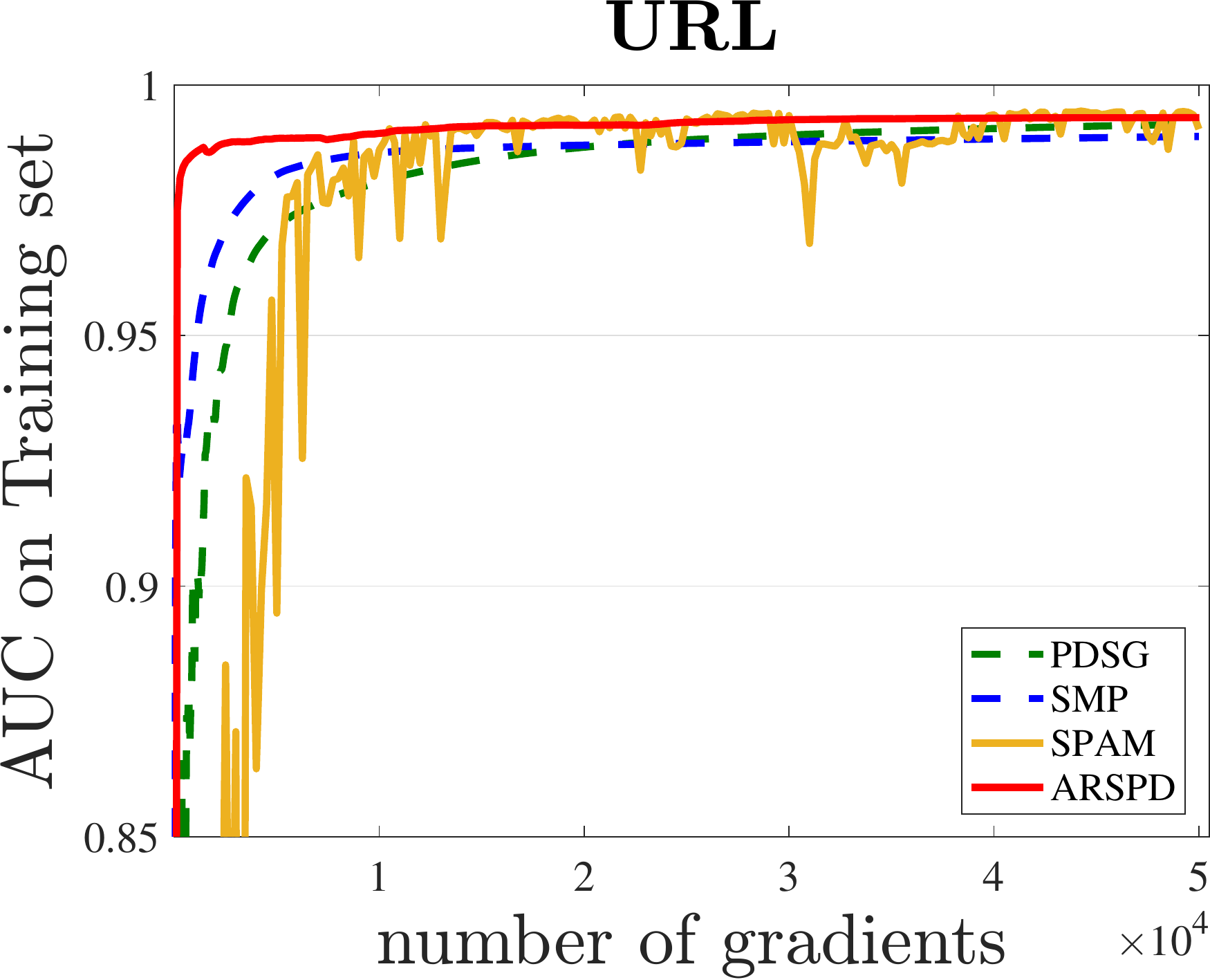}}
\caption{Results for AUC Maximization (with L1 ball constraint)}
\label{figure:auc_n_grad_l1}
\end{figure*}
% AUC-cputime with L1 constraint
\begin{figure*}[t]
\centering
\hspace{-0.1in}
{\includegraphics[scale=.2125]{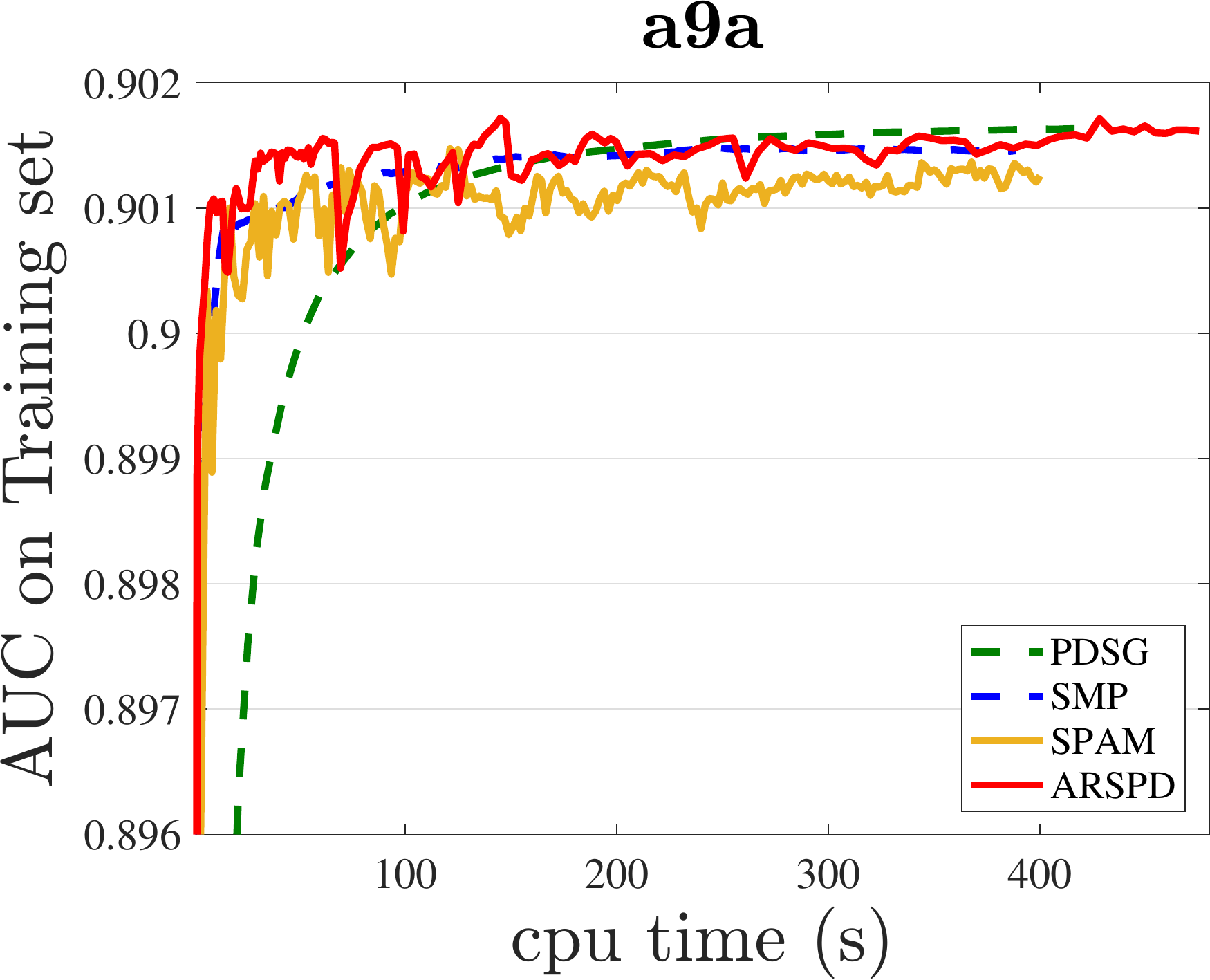}}
\hspace{-0.1in}
{\includegraphics[scale=.2125]{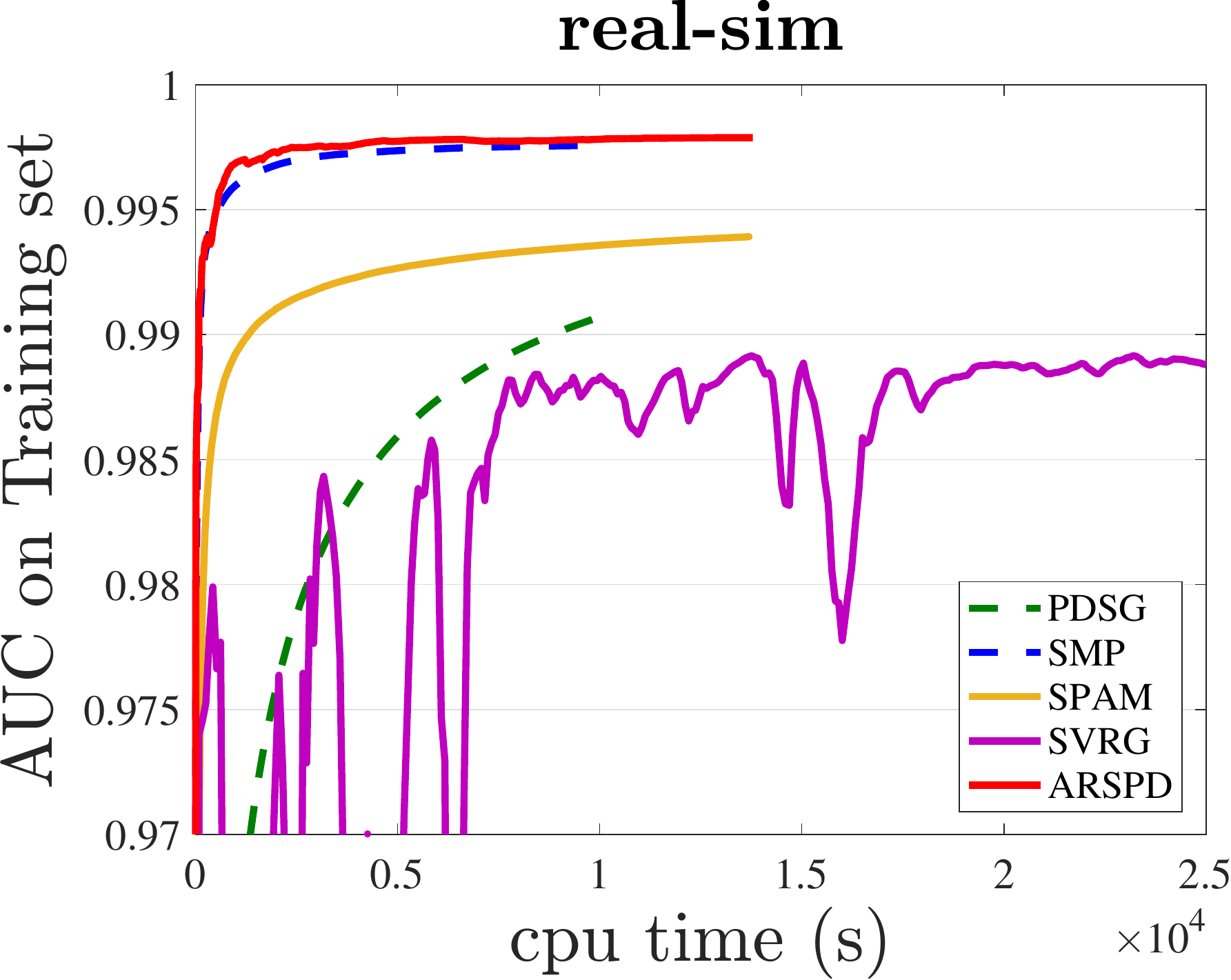}}
\hspace{-0.1in}
{\includegraphics[scale=.2125]{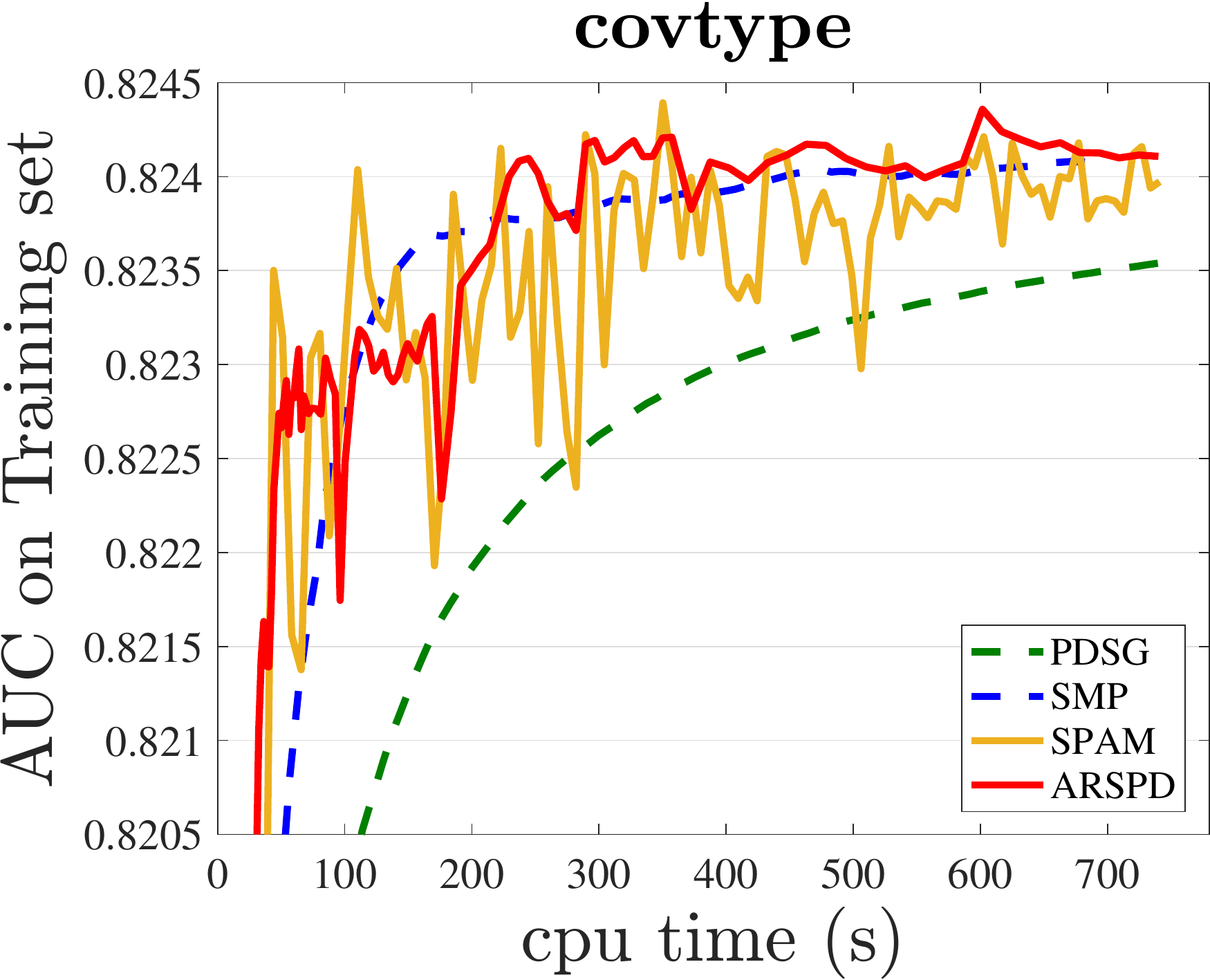}}
\hspace{-0.1in}
{\includegraphics[scale=.2125]{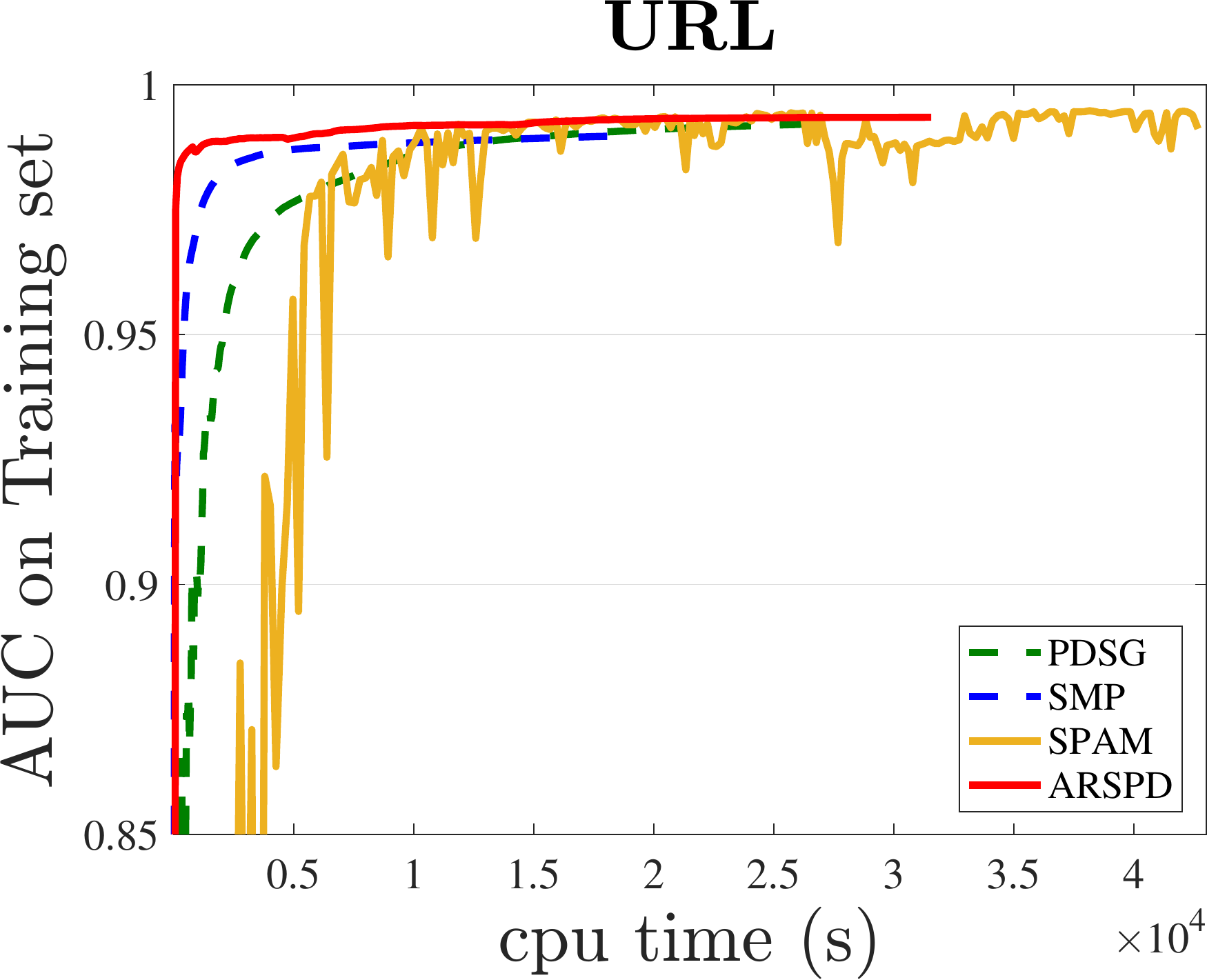}}
\caption{Results for AUC Maximization by CPU time (with L1 ball constraint)}
\label{figure:auc_cputime_l1}
\end{figure*}

Since the function $F$ is smooth in terms of $\v$ and $\alpha$, we include more applicable baselines for comparison. In particular, we compare with four algorithms, i.e., PDSG~\citep{Nemirovski:2009:RSA:1654243.1654247}, SPAM~\citep{DBLP:conf/icml/NatoleYL18}, SMP~\citep{juditsky2011} and primal-dual SVRG~\citep{DBLP:conf/nips/PalaniappanB16}. 
For primal-dua SVRG, we directly use the formulation of AUC proposed in the paper and conduct the experiment using the code provided by the authors
\footnote{Code derived at \url{https://sites.google.com/site/pbalamuru/home/sagsaddle-code}}. 
SPAM is an algorithm proposed particularly for the stochastic AUC maximization. SMP and SVRG utilize the smoothness of the objective function. The complexity of PDSG and SMP for finding an $\epsilon$-stationary solution is given by $O(d/\epsilon^2)$. Note that both SPAM and SVRG require a strong convexity of the objective function on the primal variable. To this end, we add an $\ell_{2}$ regularizer, i.e., $\frac{\lambda}{2}||\w||_{2}^{2}$ with a small value of $\lambda = \Theta(\epsilon)$. These two algorithms have a total time complexity for finding a solution $\v$ such that $\|\v- \v_*\|^2\leq \epsilon$  given by $\widetilde O(d/\epsilon^2)$ and $\widetilde O(nd + nd/\epsilon)$, respectively. We can see that all baseline algorithms have worse time complexity than RSPD, especially the primal dual SVRG algorithm. 

%To investigate the effectiveness of ARSPD, we employ an $\ell_{2}$ ball constraint, i.e., $||\w||_{2} \leq R_{\text{AUC}}$ in the AUC maximization problem. 
%This constraint may make $\theta$ unknown.
%The value of $R_{\text{AUC}}$ is fixed to $10$ in all the experiments.
%Additionally, to make the problem strongly convex, we add an $\ell_{2}$ regularizer, i.e., $\frac{\lambda}{2}||\w||_{2}^{2}$, to SPAM and SVRG for saddle point, both of which require the strong convexity in their analysis.
%The value of $\lambda$ is very small (close to $\epsilon$).

In the $\ell_{2}$ ball setting, we fix $B=10$ and $\lambda = 10^{-4}$ on all datasets.
In the $\ell_{1}$ ball setting, we set $B = 100$ on a9a, covtype and URL, and $B = 1000$ on real-sim.
The initial step sizes of all algorithms are tuned in the range of $\{ 10^{-5:1:3} \}$. For ARSPD, we set $S=5$ and simply set $\theta = 0$ pretending that we do not know the value of true $\theta$ and tune $\kappa = \{ 0.25, 0.5, 0.75, 1 \}$. The initial solution of all algorithms are set to $\mathbf{0}$.
% The four used datasets are listed in Table~\ref{tab:data_stats}.
%On each dataset, we randomly sample $80\%$ as the training set and the rest as the testing set. We report the average results over five different samplings.
For the $\ell_{2}$ ball setting, the convergence curves of AUC on four data sets against the number of gradients and CPU time are shown in Figure~\ref{fig:auc_n_grad} and Figure~\ref{figure:auc_cputime}, including two large-scale datasets covtype and URL, on which SVRG is too slow to be plotted. 
For the $\ell_{1}$ ball setting, the convergence curves of AUC against the number of gradients and CPU time are shown in Figure~\ref{figure:auc_n_grad_l1} and Figure{\ref{figure:auc_cputime_l1}}.
We can see that the overall performance of ARSPD is the best among all algorithms.

\section{Conclusion}\label{sec:conc}
In this paper, we have proposed novel stochastic primal-dual algorithms for solving convex-concave problems with no bilinear structure assumed, which  employ a mixture of stochastic gradient updates and deterministic dual updates. %, which occur for a logarithmic number.
%The designed algorithm deals with a family of convex-concave problems that do not include a bilinear structure, which, however, is a requirement of many existing studies on stochastic primal-dual algorithms.
A fast convergence rate of $O(1/T)$ was achieved  under strong convexity on the primal and dual variables.
In addition, we design variants for more general problems without strong convexity achiving adaptive rates. 
%We provide theoretical analysis and investigate the performance of our algorithms in the experiments of distributionally robust optimization and AUC maximization.
Empirical results verify the effectiveness of our algorithms.

% In the unusual situation where you want a paper to appear in the
% references without citing it in the main text, use \nocite
%\nocite{langley00}

\bibliography{all}
%\bibliographystyle{plain}

%%%%%%%%%%%%%%%%%%%%%%%%%%%%%%%%%%%%%%%%%%%%%%%%%%%%%%%%%%%%%%%%%%%%%%%%%%%%%%%
%%%%%%%%%%%%%%%%%%%%%%%%%%%%%%%%%%%%%%%%%%%%%%%%%%%%%%%%%%%%%%%%%%%%%%%%%%%%%%%
% DELETE THIS PART. DO NOT PLACE CONTENT AFTER THE REFERENCES!
%%%%%%%%%%%%%%%%%%%%%%%%%%%%%%%%%%%%%%%%%%%%%%%%%%%%%%%%%%%%%%%%%%%%%%%%%%%%%%%
%%%%%%%%%%%%%%%%%%%%%%%%%%%%%%%%%%%%%%%%%%%%%%%%%%%%%%%%%%%%%%%%%%%%%%%%%%%%%%%
\newpage
\onecolumn
%{\begin{center} \Large Supplement \end{center}}
\appendix

\section{A Lemma Regarding $\mathcal A(x)$}
\label{app:sec:regarding_A}

\begin{lem}\label{lemma:supp:1}
Let $\mathcal A(x) = \arg\max_{y\in\text{dom}(\phi^*)}y^\top\ell(x) - \phi^*(y)$, where $\phi^*$ is the convex conjugate of a differentiable function $\phi$, then
\begin{align*}
\mathcal A(x) = \nabla \phi(\ell(x)).
\end{align*}
\end{lem}
\begin{proof}
%Recall that $\mathcal A(x) = \arg\max_{y\in\text{dom}(\phi^*)}y^\top\ell(x) - \phi^*(y)$.
Let $\hat y =\mathcal A(x)$, then we know
\begin{align*}
\phi(\ell(x)) = \hat y^\top\ell(x) - \phi^*(\hat y).
\end{align*}
Since $\phi$ is differentinable, and then by using Lemma 11.4 in~\citep{bianchi} we have
\begin{align*}
\hat y = \nabla \phi(\ell(x)).
\end{align*}
That is 
\begin{align*}
\mathcal A(x) = \nabla \phi(\ell(x)).
\end{align*}
\end{proof}

\section{Proof of Lemma~\ref{lemma:convergence_RSPDsc_per_stage}}
\label{app:sec:proof:lem:convergence_RSPDsc_per_stage}

For simplicity of presentation, we use the notations $\Delta_{x}^{t} = \nabla_{x} f(x_{t}, y_{t}; \xi_{t})$, $\Delta_{y}^{t} = \nabla_{y} f(x_{t}, y_{t}, ; \xi_{t})$, $\partial_{x}^{t} = \nabla_{x} f(x_{t}, y_{t})$ and $\partial_{y}^{t} = \nabla_{y} f(x_{t}, y_{t})$.
To prove Lemma~\ref{lemma:convergence_RSPDsc_per_stage}, we would leverage the following two update approaches:
\begin{align}
\label{eq1:two_update_sequences}
  & 
  \left \{ 
    \begin{array}{cc}
    x_{t+1} = \arg \min_{x \in X } &  x^{\top} \Delta_{x}^{t} + \frac{1}{2 \eta_x } || x - x_{t} ||^{2}   \\
    y_{t+1} = \arg \min_{y \in Y } & -y^{\top} \Delta_{y}^{t} + \frac{1}{2 \eta_y } || y - y_{t} ||^{2}   
    \end{array}
  \right.     \nonumber   \\
  & 
  \left \{ 
    \begin{array}{cc}
    \xtilde_{t+1} = \arg \min_{x \in X} &  x^{\top} ( \partial_{x}^{t} - \Delta_{x}^{t} ) + \frac{1}{2 \eta_x } || x - \xtilde_{t} ||^{2}   \\
    \ytilde_{t+1} = \arg \min_{y \in Y} & -y^{\top} ( \partial_{y}^{t} - \Delta_{y}^{t} ) + \frac{1}{2 \eta_y } || y - \ytilde_{t} ||^{2}    ,
    \end{array}
  \right. 
\end{align}
where $x_{0} = \xtilde_{0}$ and $y_{0} = \ytilde_{0}$.
The first two updates are identical to Line~\ref{alg1:line:pd_sa_update_x} and Line~\ref{alg1:line:pd_sa_update_y} in Algorithm~\ref{alg:RSPDsc}.
This can be verified easily.
Take the first one as example:
\begin{align*}
      x_{t+1} 
=    \Pi_{X } ( x_{t} - \eta_x \Delta_{x}^{t} )  
=     \arg\min_{x \in X } || x - (x_{t} - \eta_x \Delta_{x}^{t}) ||^{2} 
=    \arg\min_{x \in X } \frac{1}{2 \eta_x} || x - x_{t} ||^{2} + x^{\top} \Delta_{x}^{t}  .
\end{align*}
Let $\psi(x) = x^{\top} u + \frac{1}{2\gamma} || x - v ||^{2}$ with $x' = \arg \min_{x \in X} \psi(x)$, which includes the four update approaches in~(\ref{eq1:two_update_sequences}) as special cases.
By using the strong convexity of $\psi(x)$ and the first order optimality condition ($\nabla \psi(x')^{\top} ( x - x' ) \geq 0$), for any $x$, we have
\begin{align*}
       \psi(x) - \psi(x') \geq & \nabla \psi(x')^T ( x - x') + \frac{1}{2\gamma} || x - x' ||^{2} \geq \frac{1}{2\gamma} || x - x' ||^{2},
\end{align*}
which implies
\begin{align*}
       0 \leq & ( x - x' )^{\top} u + \frac{1}{2\gamma}  || x - v ||^{2} - \frac{1}{2\gamma}  || x' - v ||^{2} - \frac{1}{2\gamma} || x - x' ||^{2}  \\
=    & ( v- x' )^{\top}  u- (v - x )^{\top} u + \frac{1}{2\gamma}  || x - v ||^{2} - \frac{1}{2\gamma} || x' - v ||^{2} - \frac{1}{2\gamma}  || x - x' ||^{2} \\
=    & - \frac{1}{2\gamma}  || x' - v ||^{2} + ( v - x' )^{\top} u  + \frac{1}{2\gamma} || x - v ||^{2}  - \frac{1}{2\gamma}  || x - x' ||^{2} - (v - x )^{\top} u  \\
\leq &  \frac{\gamma}{2} || u ||^{2} + \frac{1}{2\gamma}  || x - v ||^{2}  - \frac{1}{2\gamma} || x - x' ||^{2} - (v - x )^{\top} u. 
\end{align*}
Then
\begin{align}
\label{eq1:proof_expectation1}
(v - x )^{\top} u  \leq \frac{\gamma}{2} || u ||^{2} + \frac{1}{2\gamma}  || x - v ||^{2}  - \frac{1}{2\gamma}  || x - x' ||^{2} .
\end{align}
Applying the above result to the updates in~(\ref{eq1:two_update_sequences}), we have
\begin{align}
\label{eq1:seperate_all}
& ( x_{t} - x )^{\top} \Delta_{x}^{t} \leq \frac{1}{2 \eta_x} || x - x_{t} ||^{2} - \frac{1}{2 \eta_x} || x - x_{t+1} ||^{2} + \frac{\eta_x}{2} || \Delta_{x}^{t} ||^{2}   \nonumber\\
& ( y - y_{t} )^{\top} \Delta_{y}^{t} \leq \frac{1}{2 \eta_y} || y - y_{t} ||^{2} - \frac{1}{2 \eta_y} || y - y_{t+1} ||^{2} + \frac{\eta_y}{2} || \Delta_{y}^{t} ||^{2}   \nonumber\\
& ( \xtilde_{t} - x )^{\top} ( \partial_{x}^{t} - \Delta_{x}^{t} ) \leq \frac{1}{2 \eta_x} || x - \xtilde_{t} ||^{2} - \frac{1}{2 \eta_x} || x - \xtilde_{t+1} ||^{2} + \frac{\eta_x}{2} || \partial_{x}^{t} - \Delta_{x}^{t} ||^{2}   \nonumber\\
& ( y - \ytilde_{t} )^{\top} ( \partial_{y}^{t} - \Delta_{y}^{t} ) \leq \frac{1}{2 \eta_y} || y - \ytilde_{t} ||^{2} - \frac{1}{2 \eta_y} || y - \ytilde_{t+1} ||^{2} + \frac{\eta_y}{2} || \partial_{y}^{t} - \Delta_{y}^{t} ||^{2} .
\end{align}
Adding the above four inequalities, we have
\begin{align}
\label{eq1:combine_all}
\LHS 
=    & 
       ( x_{t} - x )^{\top} \Delta_{x}^{t} + ( y - y_{t} )^{\top} \Delta_{y}^{t} + ( \xtilde_{t} - x )^{\top} ( \partial_{x}^{t} - \Delta_{x}^{t} ) + ( y - \ytilde_{t} )^{\top} ( \partial_{y}^{t} - \Delta_{y}^{t} )      \nonumber\\
=    & 
       ( x_{t} - x )^{\top} \partial_{x}^{t} + ( x_{t} - x )^{\top} ( \Delta_{x}^{t} - \partial_{x}^{t} )
       + ( y - y_{t} )^{\top} \partial_{y}^{t} + ( y - y_{t} )^{\top} ( \Delta_{y}^{t} - \partial_{y}^{t} )    \nonumber\\
     & + ( \xtilde_{t} - x )^{\top} ( \partial_{x}^{t} - \Delta_{x}^{t} ) + ( y - \ytilde_{t} )^{\top} ( \partial_{y}^{t} - \Delta_{y}^{t} )   \nonumber\\
=    & 
       - ( x - x_{t} )^{\top} \partial_{x}^{t} + (y - y_{t})^{\top} \partial_{y}^{t}  - ( x_{t} - \xtilde_{t} )^{\top} ( \partial_{x}^{t} - \Delta_{x}^{t} ) - ( \ytilde_{t} - y_{t} )^{\top} ( \partial_{y}^{t} - \Delta_{y}^{t} ) \nonumber\\
\geq & 
       - ( f(x, y_{t}) - f(x_{t}, y_{t}) ) + ( f(x_{t}, y) - f(x_{t}, y_{t}) ) - ( x_{t} - \xtilde_{t} )^{\top} ( \partial_{x}^{t} - \Delta_{x}^{t} ) - ( \ytilde_{t} - y_{t} )^{\top} ( \partial_{y}^{t} - \Delta_{y}^{t} )   \nonumber\\
\RHS
=    & \frac{1}{2\eta_x} \Big\{ || x - x_{t} ||^{2} - || x - x_{t+1} ||^{2} + || x - \xtilde_{t} ||^{2} - || x - \xtilde_{t+1} ||^{2} \Big\} + \frac{\eta_x}{2} \Big\{ || \Delta_{x}^{t} ||^{2} + || \partial_{x}^{t} - \Delta_{x}^{t} ||^{2} \Big\}   \nonumber\\
    &+ \frac{1}{2\eta_y} \Big\{  || y - y_{t} ||^{2} - || y - y_{t+1} ||^{2}  + || y - \ytilde_{t} ||^{2} - || y - \ytilde_{t+1} ||^{2} \Big\} + \frac{\eta_y}{2} \Big\{ || \Delta_{y}^{t} ||^{2} + || \partial_{y}^{t} - \Delta_{y}^{t} ||^{2} \Big\}  \nonumber\\
\leq & \frac{1}{2\eta_x} \Big\{ || x - x_{t} ||^{2} - || x - x_{t+1} ||^{2} + || x - \xtilde_{t} ||^{2} - || x - \xtilde_{t+1} ||^{2} \Big\} + \frac{5\eta_x M^2}{2}   \nonumber\\
    &+ \frac{1}{2\eta_y} \Big\{  || y - y_{t} ||^{2} - || y - y_{t+1} ||^{2}  + || y - \ytilde_{t} ||^{2} - || y - \ytilde_{t+1} ||^{2} \Big\} + \frac{5 \eta_y B^2}{2},
\end{align}
where the last inequality uses the facts that $|| \Delta_{x}^{t} || \leq M$, $|| \partial_{x}^{t} || \leq M$, $|| \Delta_{y}^{t} || \leq B$ and $|| \partial_{y}^{t} || \leq B$.
Then we combine the LHS and RHS by summing up $t = 0, ..., T-1$:
\begin{align}
\label{eq1:combine_all2}
       \sum_{t=0}^{T-1} ( f(x_{t}, y) - f(x, y_{t}) )
\leq & \frac{1}{2\eta_x} \Big\{ || x - x_{0} ||^{2} - || x - x_{T} ||^{2} + || x - \xtilde_{0} ||^{2} - || x - \xtilde_{T} ||^{2}  \Big\} + \frac{5\eta_x T M^2}{2}  \nonumber\\
     &\frac{1}{2\eta_y} \Big\{ || y - y_{0} ||^{2} - || y - y_{T} ||^{2} + || y - \ytilde_{0} ||^{2} - || y - \ytilde_{T} ||^{2} \Big\} + \frac{5\eta_y T  B^2}{2} \nonumber\\
     & + \sum_{t=0}^{T-1} \Big( ( x_{t} - \xtilde_{t} )^{\top} ( \partial_{x}^{t} - \Delta_{x}^{t} ) + ( y_{t} - \ytilde_{t} )^{\top} ( \partial_{y}^{t} - \Delta_{y}^{t} ) \Big)  .
\end{align}
By Jensen's inequality, we have
\begin{align}
\label{eq1:combine_all3}
       f(\xbar_T, y) - f(x, \ybar_T)
\leq & 
       \frac{|| x - x_{0} ||^{2} }{\eta_x T  }  +  \frac{ || y - y_{0} ||^{2} }{\eta_yT  }+ \frac{5\eta_x M^2}{2}    + \frac{5\eta_y B^2}{2}  \nonumber\\
     & + \frac{1}{T} \sum_{t=0}^{T-1} \Big( (  x_{t} - \xtilde_{t} )^{\top} ( \partial_{x}^{t} - \Delta_{x}^{t} ) + (  y_{t} - \ytilde_{t} )^{\top} ( \partial_{y}^{t} - \Delta_{y}^{t} ) \Big),
\end{align}
where $\bar x_T = \sum_{t=0}^{T-1} x_{t}/T$,  $\bar y_T = \sum_{t=0}^{T-1} y_{t}/T$.
Let $\hat y_T = \arg\max_{y\in Y}f(\bar x_T, y)$ and $x_*\in X^*$, we get
\begin{align*}
     \max_{y\in Y} f(\bar x_T, y) - f(x_*, \bar y_T)\leq&   \frac{|| x_* - x_{0} ||^{2} }{\eta_x   T  } +  \frac{ || \hat y_T - y_{0} ||^{2} }{\eta_y T  }+ \frac{5\eta_x M^2}{2}    + \frac{5\eta_y B^2}{2}\\
     & + \frac{1}{T} \sum_{t=0}^{T-1} \Big( (  x_{t} - \xtilde_{t} )^{\top} ( \partial_{x}^{t} - \Delta_{x}^{t} ) + (  y_{t} - \ytilde_{t} )^{\top} ( \partial_{y}^{t} - \Delta_{y}^{t} ) \Big), %+ \frac{ 8 ( M + B ) R \sqrt{2 \log\frac{1}{\tildedelta}} }{ \sqrt{T} } ,
\end{align*}

Then we complete the proof by taking the expectation on both sides of above inequality and using the the facts that $\E[(  x_{t} - \xtilde_{t} )^{\top} ( \partial_{x}^{t} - \Delta_{x}^{t} ) + (  y_{t} - \ytilde_{t} )^{\top} ( \partial_{y}^{t} - \Delta_{y}^{t} )]=0$.
%Then we employ Azuma's inequality (Lemma~\ref{lemma:azuma}) to upper bound the last term with a high probability.
%Let $V_{t} = ( \xtilde_{t} - x_{t} )^{T} ( \partial_{x}^{t} - \Delta_{x}^{t} ) + ( \ytilde_{t} - y_{t} )^{T} ( \partial_{y}^{t} - \Delta_{y}^{t} )$ be martingale difference sequence.
%\begin{align*}
%       | V_{t} | 
%=    &
%       | ( \xtilde_{t} - x_{t} )^{T} ( \partial_{x}^{t} - \Delta_{x}^{t} ) + ( \ytilde_{t} - y_{t} )^{T} ( \partial_{y}^{t} - \Delta_{y}^{t} ) |   \nonumber\\
%\leq & 
%       | ( \xtilde_{t} - x_{t} )^{T} ( \partial_{x}^{t} - \Delta_{x}^{t} ) |  + | ( \ytilde_{t} - y_{t} )^{T} ( \partial_{y}^{t} - \Delta_{y}^{t} ) |   \nonumber\\
%\leq &
%       || \xtilde_{t} - x_{t} || ( || \partial_{x}^{t} || + || \Delta_{x}^{t} || ) + || \ytilde_{t} - y_{t} || ( || \partial_{y}^{t} || + || \Delta_{y}^{t} || )  \nonumber\\
%\leq & 
%       2 M (|| \xtilde_{t} - x_{0} ||+|| x_{0} - x_{t} || )+ 2 B (|| \ytilde_{t} - y_{0}|| +|| y_{0} - y_{t} ||) \nonumber\\
%\leq & 
%       %4 M \Rx + 4 B\Ry   .
%       8 ( M + B ) R  .
%\end{align*}
%Therefore, with probability at least $1 - \tildedelta$, we have
%$$
%f(\xbar, y) - f(x, \ybar)
%\leq
%       \frac{|| x - x_{0} ||^{2} }{\eta_x T  }  +  \frac{ || y - y_{0} ||^{2} }{\eta_y T }+ \frac{5\eta_x M^2}{2}    + \frac{5\eta_y B^2}{2} + \frac{ 8 ( M + B ) R \sqrt{2 \log\frac{1}{\tildedelta}} }{ \sqrt{T} }  .
%$$
%\end{proof} % proof of Lemma 3 (on the convergence of a single stage)

\section{Proof of Lemma~\ref{lemma:convergence_rspd_per_stage}}
\label{app:sec:proof:lem:convergence_rspd_per_stage}
% Below is the proof of Lemma (on the convergence of a single stage), which applies two-update-sequence approach to construct martingale difference sequence
For simplicity of presentation, we use the notations $\Delta_{x}^{t} = \nabla_{x} f(x_{t}, y_{t}; \xi_{t})$, $\Delta_{y}^{t} = \nabla_{y} f(x_{t}, y_{t}, ; \xi_{t})$, $\partial_{x}^{t} = \nabla_{x} f(x_{t}, y_{t})$ and $\partial_{y}^{t} = \nabla_{y} f(x_{t}, y_{t})$.

To prove Lemma~\ref{lemma:convergence_rspd_per_stage}, we would leverage the following two update approaches:
\begin{align}
\label{eq:two_update_sequences}
  & 
  \left \{ 
    \begin{array}{cc}
    x_{t+1} = \arg \min_{x \in X \cap \calB(x_{0}, R_x)} &  x^{\top} \Delta_{x}^{t} + \frac{1}{2 \eta_x } || x - x_{t} ||^{2}   \\
    y_{t+1} = \arg \min_{y \in Y \cap \calB(y_{0}, R_y)} & -y^{\top} \Delta_{y}^{t} + \frac{1}{2 \eta_y } || y - y_{t} ||^{2}   
    \end{array}
  \right.     \nonumber   \\
  & 
  \left \{ 
    \begin{array}{cc}
    \xtilde_{t+1} = \arg \min_{x \in X \cap \calB(x_{0}, R_x)} &  x^{\top} ( \partial_{x}^{t} - \Delta_{x}^{t} ) + \frac{1}{2 \eta_x } || x - \xtilde_{t} ||^{2}   \\
    \ytilde_{t+1} = \arg \min_{y \in Y \cap \calB(y_{0}, R_y)} & -y^{\top} ( \partial_{y}^{t} - \Delta_{y}^{t} ) + \frac{1}{2 \eta_y } || y - \ytilde_{t} ||^{2}    ,
    \end{array}
  \right. 
\end{align}
where $x_{0} = \xtilde_{0}$ and $y_{0} = \ytilde_{0}$.
% The first two lines are identical to Line~\ref{alg2:line:pd_sa_update_x} and Line~\ref{alg2:line:pd_sa_update_y} in Algorithm~\ref{alg:restart_primal_dual_algorithm_sa}.
The first two lines are identical to Line 5 and 6 in Algorithm~\ref{alg:restart_primal_dual_algorithm_sa}.
This can be verified easily.
Take the first one as example:
\begin{align*}
x_{t+1} 
= &
\Pi_{X \cap \calB(x_{0}, R_x)} ( x_{t} - \eta_x \Delta_{x}^{t} )  
\\
= &
\arg\min_{x \in X \cap \calB(x_{0}, R_x)} \| x - (x_{t} - \eta_x \Delta_{x}^{t}) \|^{2} 
\\
= &
\arg\min_{x \in X \cap \calB(x_{0}, R_x)} \frac{1}{2 \eta_x} \| x - x_{t} \|^{2} + x^{\top} \Delta_{x}^{t}  .
\end{align*}
Let us define $\psi(x) = x^{\top} u + \frac{1}{2\gamma} || x - v ||^{2}$ with $x' = \arg \min_{x \in X} \psi(x)$, which includes the four update approaches in~(\ref{eq:two_update_sequences}) as special cases.
By using the strong convexity of $\psi(x)$ and the first order optimality condition ($\nabla \psi(x')^{\top} ( x - x' ) \geq 0$), for any $x$, we have
\begin{align*}
       \psi(x) - \psi(x') \geq & \nabla \psi(x')^T ( x - x') + \frac{1}{2\gamma} || x - x' ||^{2} \geq \frac{1}{2\gamma} || x - x' ||^{2},
\end{align*}
which implies
\begin{align*}
       0 \leq & ( x - x' )^{\top} u + \frac{1}{2\gamma}  || x - v ||^{2} - \frac{1}{2\gamma}  || x' - v ||^{2} - \frac{1}{2\gamma} || x - x' ||^{2}  \\
=    & ( v- x' )^{\top}  u- (v - x )^{\top} u + \frac{1}{2\gamma}  || x - v ||^{2} - \frac{1}{2\gamma} || x' - v ||^{2} - \frac{1}{2\gamma}  || x - x' ||^{2} \\
=    & - \frac{1}{2\gamma}  || x' - v ||^{2} + ( v - x' )^{\top} u  + \frac{1}{2\gamma} || x - v ||^{2}  - \frac{1}{2\gamma}  || x - x' ||^{2} - (v - x )^{\top} u  \\
\leq &  \frac{\gamma}{2} || u ||^{2} + \frac{1}{2\gamma}  || x - v ||^{2}  - \frac{1}{2\gamma} || x - x' ||^{2} - (v - x )^{\top} u. 
\end{align*}
Then
\begin{align}
\label{eq:proof_expectation1}
(v - x )^{\top} u  \leq \frac{\gamma}{2} || u ||^{2} + \frac{1}{2\gamma}  || x - v ||^{2}  - \frac{1}{2\gamma}  || x - x' ||^{2} .
\end{align}
Applying the above result to the updates in~(\ref{eq:two_update_sequences}) (treating $u$ above as $\Delta_x^t$, $\Delta_y^t$, $\partial_x^t - \Delta_x^t$, $\partial_y^t - \Delta_y^t$, respectively), we have
\begin{align}
\label{eq:seperate_all}
& ( x_{t} - x )^{\top} \Delta_{x}^{t} \leq \frac{1}{2 \eta_x} || x - x_{t} ||^{2} - \frac{1}{2 \eta_x} || x - x_{t+1} ||^{2} + \frac{\eta_x}{2} || \Delta_{x}^{t} ||^{2}   \nonumber\\
& ( y - y_{t} )^{\top} \Delta_{y}^{t} \leq \frac{1}{2 \eta_y} || y - y_{t} ||^{2} - \frac{1}{2 \eta_y} || y - y_{t+1} ||^{2} + \frac{\eta_y}{2} || \Delta_{y}^{t} ||^{2}   \nonumber\\
& ( \xtilde_{t} - x )^{\top} ( \partial_{x}^{t} - \Delta_{x}^{t} ) \leq \frac{1}{2 \eta_x} || x - \xtilde_{t} ||^{2} - \frac{1}{2 \eta_x} || x - \xtilde_{t+1} ||^{2} + \frac{\eta_x}{2} || \partial_{x}^{t} - \Delta_{x}^{t} ||^{2}   \nonumber\\
& ( y - \ytilde_{t} )^{\top} ( \partial_{y}^{t} - \Delta_{y}^{t} ) \leq \frac{1}{2 \eta_y} || y - \ytilde_{t} ||^{2} - \frac{1}{2 \eta_y} || y - \ytilde_{t+1} ||^{2} + \frac{\eta_y}{2} || \partial_{y}^{t} - \Delta_{y}^{t} ||^{2} .
\end{align}
Adding the above four inequalities, we have
\begin{align}
\label{eq:combine_all}
\LHS 
=    & 
       ( x_{t} - x )^{\top} \Delta_{x}^{t} + ( y - y_{t} )^{\top} \Delta_{y}^{t} + ( \xtilde_{t} - x )^{\top} ( \partial_{x}^{t} - \Delta_{x}^{t} ) + ( y - \ytilde_{t} )^{\top} ( \partial_{y}^{t} - \Delta_{y}^{t} )      \nonumber\\
=    & 
       ( x_{t} - x )^{\top} \partial_{x}^{t} + ( x_{t} - x )^{\top} ( \Delta_{x}^{t} - \partial_{x}^{t} )
       + ( y - y_{t} )^{\top} \partial_{y}^{t} + ( y - y_{t} )^{\top} ( \Delta_{y}^{t} - \partial_{y}^{t} )    \nonumber\\
     & + ( \xtilde_{t} - x )^{\top} ( \partial_{x}^{t} - \Delta_{x}^{t} ) + ( y - \ytilde_{t} )^{\top} ( \partial_{y}^{t} - \Delta_{y}^{t} )   \nonumber\\
=    & 
       - ( x - x_{t} )^{\top} \partial_{x}^{t} + (y - y_{t})^{\top} \partial_{y}^{t}  - ( x_{t} - \xtilde_{t} )^{\top} ( \partial_{x}^{t} - \Delta_{x}^{t} ) - ( \ytilde_{t} - y_{t} )^{\top} ( \partial_{y}^{t} - \Delta_{y}^{t} ) \nonumber\\
\geq & 
       - ( f(x, y_{t}) - f(x_{t}, y_{t}) ) + ( f(x_{t}, y) - f(x_{t}, y_{t}) ) - ( x_{t} - \xtilde_{t} )^{\top} ( \partial_{x}^{t} - \Delta_{x}^{t} ) - ( \ytilde_{t} - y_{t} )^{\top} ( \partial_{y}^{t} - \Delta_{y}^{t} )   \nonumber\\
\RHS
=    & \frac{1}{2\eta_x} \Big\{ || x - x_{t} ||^{2} - || x - x_{t+1} ||^{2} + || x - \xtilde_{t} ||^{2} - || x - \xtilde_{t+1} ||^{2} \Big\} + \frac{\eta_x}{2} \Big\{ || \Delta_{x}^{t} ||^{2} + || \partial_{x}^{t} - \Delta_{x}^{t} ||^{2} \Big\}   \nonumber\\
    &+ \frac{1}{2\eta_y} \Big\{  || y - y_{t} ||^{2} - || y - y_{t+1} ||^{2}  + || y - \ytilde_{t} ||^{2} - || y - \ytilde_{t+1} ||^{2} \Big\} + \frac{\eta_y}{2} \Big\{ || \Delta_{y}^{t} ||^{2} + || \partial_{y}^{t} - \Delta_{y}^{t} ||^{2} \Big\}  \nonumber\\
\leq & \frac{1}{2\eta_x} \Big\{ || x - x_{t} ||^{2} - || x - x_{t+1} ||^{2} + || x - \xtilde_{t} ||^{2} - || x - \xtilde_{t+1} ||^{2} \Big\} + \frac{5\eta_x M^2}{2}   \nonumber\\
    &+ \frac{1}{2\eta_y} \Big\{  || y - y_{t} ||^{2} - || y - y_{t+1} ||^{2}  + || y - \ytilde_{t} ||^{2} - || y - \ytilde_{t+1} ||^{2} \Big\} + \frac{5 \eta_y B^2}{2},
\end{align}
where the last inequality uses the facts that $|| \Delta_{x}^{t} || \leq M$, $|| \partial_{x}^{t} || \leq M$, $|| \Delta_{y}^{t} || \leq B$ and $|| \partial_{y}^{t} || \leq B$.
Then we combine the LHS and RHS by summing up $t = 0, ..., T-1$:
\begin{align}
\label{eq:combine_all2}
       \sum_{t=0}^{T-1} ( f(x_{t}, y) - f(x, y_{t}) )
\leq & \frac{1}{2\eta_x} \Big\{ || x - x_{0} ||^{2} - || x - x_{T} ||^{2} + || x - \xtilde_{0} ||^{2} - || x - \xtilde_{T} ||^{2}  \Big\} + \frac{5\eta_x T M^2}{2}  \nonumber\\
     &\frac{1}{2\eta_y} \Big\{ || y - y_{0} ||^{2} - || y - y_{T} ||^{2} + || y - \ytilde_{0} ||^{2} - || y - \ytilde_{T} ||^{2} \Big\} + \frac{5\eta_y T  B^2}{2} \nonumber\\
     & + \sum_{t=0}^{T-1} \Big( ( x_{t} - \xtilde_{t} )^{\top} ( \partial_{x}^{t} - \Delta_{x}^{t} ) + ( y_{t} - \ytilde_{t} )^{\top} ( \partial_{y}^{t} - \Delta_{y}^{t} ) \Big)  .
\end{align}
By Jensen's inequality, we have
\begin{align}
\label{eq:combine_all3}
       f(\xbar_T, y) - f(x, \ybar_T)
\leq & 
       \frac{|| x - x_{0} ||^{2} }{\eta_x T  }  +  \frac{ || y - y_{0} ||^{2} }{\eta_yT  }+ \frac{5\eta_x M^2}{2}    + \frac{5\eta_y B^2}{2}  \nonumber\\
     & + \frac{1}{T} \sum_{t=0}^{T-1} \Big( (  x_{t} - \xtilde_{t} )^{\top} ( \partial_{x}^{t} - \Delta_{x}^{t} ) + (  y_{t} - \ytilde_{t} )^{\top} ( \partial_{y}^{t} - \Delta_{y}^{t} ) \Big),
\end{align}
where $\bar x_T = \sum_{t=0}^{T-1} x_{t}/T$,  $\bar y_T = \sum_{t=0}^{T-1} y_{t}/T$.
Let $\hat y_T = \arg\max_{y \in Y \cap \calB(y_0, R_y) } f(\bar x_T, y)$ and any fixed $x \in X \cap \calB(x_0, R_x)$, we get
\begin{align*}
\max_{y \in Y \cap \calB(y_0, R_y)} f(\bar x_T, y) - f(x, \bar y_T)
\leq &
\frac{|| x - x_{0} ||^{2} }{\eta_x   T  }
+ \frac{ || \hat y_T - y_{0} ||^{2} }{\eta_y T  }
+ \frac{5\eta_x M^2}{2}    
+ \frac{5\eta_y B^2}{2} 
\\
& + \frac{1}{T} \sum_{t=0}^{T-1} \Big( (  x_{t} - \xtilde_{t} )^{\top} ( \partial_{x}^{t} - \Delta_{x}^{t} ) + (  y_{t} - \ytilde_{t} )^{\top} ( \partial_{y}^{t} - \Delta_{y}^{t} ) \Big), 
\end{align*}

Then we employ Azuma's inequality (Lemma~\ref{lemma:azuma}) to upper bound the last term with a high probability.
Let $V_{t} = ( \xtilde_{t} - x_{t} )^{T} ( \partial_{x}^{t} - \Delta_{x}^{t} ) + ( \ytilde_{t} - y_{t} )^{T} ( \partial_{y}^{t} - \Delta_{y}^{t} )$ be martingale difference sequence.
We have
\begin{align*}
       | V_{t} | 
=    &
       | ( \xtilde_{t} - x_{t} )^{T} ( \partial_{x}^{t} - \Delta_{x}^{t} ) + ( \ytilde_{t} - y_{t} )^{T} ( \partial_{y}^{t} - \Delta_{y}^{t} ) |   \nonumber\\
\leq & 
       | ( \xtilde_{t} - x_{t} )^{T} ( \partial_{x}^{t} - \Delta_{x}^{t} ) |  + | ( \ytilde_{t} - y_{t} )^{T} ( \partial_{y}^{t} - \Delta_{y}^{t} ) |   \nonumber\\
\leq &
       || \xtilde_{t} - x_{t} || ( || \partial_{x}^{t} || + || \Delta_{x}^{t} || ) + || \ytilde_{t} - y_{t} || ( || \partial_{y}^{t} || + || \Delta_{y}^{t} || )  \nonumber\\
\leq & 
       2 M (|| \xtilde_{t} - x_{0} ||+|| x_{0} - x_{t} || )+ 2 B (|| \ytilde_{t} - y_{0}|| +|| y_{0} - y_{t} ||) \nonumber\\
\leq & 
       4 M R_x + 4 B R_y   ,
\end{align*}
where the first inequality is due to the triangle inequality,
the second inequality is due to Cauchy–Schwarz inequality,
the third inequality is due to Assumption \ref{assumption:general} (2), and
the last inequality is due to $\tilde x_t, x_t \in X \cap \calB(x_0, R_x)$, $\tilde y_t, y_t \in Y \cap \calB(y_0, R_y)$.
Therefore, by Azuma's inequality with probability at least $1 - \tildedelta$, we have for any $x \in X \cap \calB(x_0, R_x)$
\begin{align*}
\max_{y \in Y \cap \calB(y_0, R_y)} f(\xbar_T, y) - f(x, \ybar_T)
\leq &
\frac{|| x - x_{0} ||^{2} }{\eta_x T  }  
+ \frac{ || \hat y_T - y_{0} ||^{2} }{\eta_y T }
+ \frac{5\eta_x M^2}{2}    
+ \frac{5\eta_y B^2}{2} 
\\
&
+ \frac{ 4 ( M R_x + B R_y ) \sqrt{2 \log\frac{1}{\tildedelta}} }{ \sqrt{T} }  .
\end{align*}

%After analyzing the per-stage convergence, we then proceed to the proof of the convergence analysis of Algorithm~\ref{alg:restart_primal_dual_algorithm_sa}.
% Below is the proof of Theorem 1, the main result, the convergence of spd

\section{Proof of Theorem~\ref{theorem:convergence_rspd_adaptive_c} (Theorem~\ref{theorem:convergence_rspd_adaptive_theta})}
\label{app:sec:proof:thm:convergence_rspd_adaptive}

The proof is similar to the proof of Theorem 3 in~\citep{ICMLASSG}. For completeness, we include it here. The proof of Theorem~\ref{theorem:convergence_rspd_adaptive_theta} can be also obtained by a slight change of the following proof. 
\begin{proof}
Based on the proof of Theorem~\ref{theorem:convergence_rspd}, since $v=1$ and by the settings of $S =  \lceil \log_2(\frac{\epsilon_0}{\epsilon})\rceil \geq  \lceil \log_2(\frac{\epsilon_0}{\hat\epsilon_1})\rceil$,  $R_1^{(1)} = \frac{c\epsilon_0}{\hat\epsilon_1^{1-\theta}}$, 
$
T_1 
= 
\max \bigg\{ 320 M^2, 320 B^2 L^2 G^2, 8192 \log(\frac{1}{\tildedelta}) M^2, 8192 \log(\frac{1}{\tildedelta}) B^2 L^2 G^2 \bigg\}
\cdot \frac{(R_1^{(1)})^2}{\epsilon_0^2}   ,
$
% $T_1= \max \bigg\{ \frac{320 B^{2} L^{2} G^{2} (R_{1}^{(1)})^{2}}{\epsilon_{0}^{2}}, 
%  \frac{2048 ( M + B )^{2}  (R_{1}^{(1)})^{2} \log ({1}/{\hat\delta})}{\epsilon_{0}^{2}}
% \bigg\}$, 
it can be shown that
\begin{align} \label{eqn:RRSPD}
P(x^{(1)}) - P^* \leq 2\hat\epsilon_1
\end{align}
with a probability $1-\frac{\delta}{S+1}$.
Next, by running RSPD with initial $x^{(1)}$ satisfying (\ref{eqn:RRSPD}) and the settings of $S =  \lceil \log_2(\frac{\epsilon_0}{\epsilon})\rceil \geq  \lceil \log_2(\frac{2\hat\epsilon_1}{\hat\epsilon_1/2})\rceil$, $R_1^{(2)} = \frac{c\epsilon_0}{(\hat\epsilon_1/2)^{1-\theta}} \geq \frac{c2\hat\epsilon_1}{(\hat\epsilon_1/2)^{1-\theta}}$, and 
$
T_2 = T_1 \cdot 2^{2(1-\theta)} 
= 
\max \bigg\{ 320 M^2, 320 B^2 L^2 G^2, 8192 \log(\frac{1}{\tildedelta}) M^2, 8192 \log(\frac{1}{\tildedelta}) B^2 L^2 G^2 \bigg\}
\cdot \frac{(R_1^{(2)})^2}{\epsilon_0^2}   ,
$
% $T_2= T_1\cdot 2^{2(1-\theta)} = \max \bigg\{ \frac{320 B^{2} L^{2} G^{2} (R_{1}^{(2)})^{2} }{\epsilon_{0}^{2}}, \frac{2048 ( M + B )^{2}  (R_{1}^{(2)})^{2}  \log ({1}/{\hat\delta})}{\epsilon_{0}^{2}} \bigg\} $, 
Theorem~\ref{theorem:convergence_rspd} ensures that with a probability at least $(1-\delta/(S+1))^2$,
\begin{align*} 
P(x^{(2)}) - P^* \leq \hat\epsilon_1.
\end{align*}
By continuing this process with $K =  \lceil \log_2(\hat\epsilon_1/\epsilon) \rceil + 1$, we can show that
\begin{align*} 
P(x^{(K)}) - P^* \leq 2\hat\epsilon_1/2^{K-1} \leq 2\epsilon
\end{align*}
with a probability at least $(1-\delta/(S+1))^K\geq 1-\delta\frac{K}{S+1}\geq 1- \delta$.
The total number of iterations for  $K$ calls of RSPD can be bounded by
\begin{align*} 
T_K &= S\sum_{k=1}^{K}T_k = S\sum_{k=1}^{K}T_12^{2(k-1)(1-\theta)} = ST_12^{2(K-1)(1-\theta)}\sum_{k=1}^{K} \left( 1/2^{2(1-\theta)}\right)^{K-k} \\ 
& \leq ST_12^{2(K-1)(1-\theta)} \frac{1}{1-1/2^{2(1-\theta)}} \leq  O\left( ST_1 \left( \frac{\hat\epsilon_1}{\epsilon}\right)^{2(1-\theta)} \right) \leq \widetilde O(\log(1/\delta)/\epsilon^{2(1-\theta)}).
\end{align*}
\end{proof}

\section{Piecewise Quadratic Function of Distributionally Robust Optimization}
\label{app:sec:proof:piecewise_quadratic}

We would like to prove the $\ell_{1}$ regularized DRO function is convex and piecewise quadratic, so it satifies the LEB condition with $\theta=1/2$. 
First we present the following proposition.

\begin{propo}
(Proposition 2.3~\citep{rockafellar1987linear})
Let $\rho_{V, Q} (s) = \sup_{v \in V} \{ s^{\top} v - \frac{1}{2} v^{\top} Q v\}$ where $Q$ is symmetric and positive semidefinite,
and $\rho_{V, Q} (s)$ is lower semicontinuous, convex and piecewise linear-quadratic.
Its effective domain $L = \{ s | \rho_{V, Q} < \infty \}$ is nonempty convex polyhedron that can be decomposed into finitely many polyhedral convex sets, on each of which $\rho_{V, Q}$ is quadratic or linear.
\end{propo}

We can rewrite DRO as $\max_{y \in \Delta_{n}} \sum_{i=1}^{n} y_{i} \ell_{i} (x) - \frac{\lambda_{1}}{2}|| ny - \mathbf{1}||^2 
= \max_{y \in \Delta_{n}} \sum_{i=1}^{n} y_{i} ( \ell_{i} (x) + n\lambda_{1}  ) - \frac{n^2\lambda_{1}}{2} y^{\top} \mathbf{I} y+\frac{n\lambda_{1}}{2}$,
which is piecewise linear-quadratic in $\Big( \ell(x) + n\lambda_{1} \mathbf{1} \Big)$ according to the above proposition.
If $\ell(x)$ is piecewise linear, the composition of the piecewise linear and piecewise quadratic functions is piecewise quadratic.
\end{document}